\documentclass[a4paper,USenglish,cleveref, autoref, thm-restate]{lipics-v2021}

\pdfoutput=1 
\hideLIPIcs  

\usepackage[utf8]{inputenc}
\usepackage[T1]{fontenc}
\usepackage{amsmath,amsfonts,amsthm,amssymb}
\usepackage[dvipsnames]{xcolor} 
\usepackage{xfrac}
\usepackage{booktabs}

\usepackage{epstopdf}
\usepackage{footnote}
\usepackage{tcolorbox}
\makesavenoteenv{tabular}
\makesavenoteenv{table}

\usepackage{setspace}
\usepackage{xspace}

\usepackage
[
    ruled,         
    vlined,        
    linesnumbered, 
]
{algorithm2e} 

\crefname{line}{line}{lines}
\creflabelformat{line}{#2{#1}#3}

\nolinenumbers

\let\oldsqrt\sqrt
\def\hksqrt{\mathpalette\DHLhksqrt}
\def\DHLhksqrt#1#2{\setbox0=\hbox{$#1\oldsqrt{#2\,}$}\dimen0=\ht0
   \advance\dimen0-0.2\ht0
   \setbox2=\hbox{\vrule height\ht0 depth -\dimen0}%
   {\box0\lower0.4pt\box2}}
\renewcommand\sqrt\hksqrt

\renewcommand{\le}{\leqslant}
\renewcommand{\ge}{\geqslant}

\newcommand*{\ce}{\textsc{Cluster Editing}\xspace}
\newcommand*{\cc}{\textsc{Correlation Clustering}\xspace}
\newcommand*{\cccost}{\cc cost\xspace}
\newcommand*{\clust}{clustering\xspace}

\newcommand*{\fcc}{\textsc{Fair Correlation Clustering}\xspace}
\newcommand*{\fc}{fair clustering\xspace}
\newcommand*{\fcs}{fair clusterings\xspace}
\newcommand*{\rfcc}{\textsc{Relaxed Fair Correlation Clustering}\xspace}
\newcommand*{\rfc}{relaxed fair clustering\xspace}
\newcommand*{\rfcs}{relaxed fair clusterings\xspace}
\newcommand*{\arfcc}{\(\alpha\)-\textsc{relaxed Fair Correlation Clustering}\xspace}
\newcommand*{\arfc}{\(\alpha\)-relaxed fair clustering\xspace}
\newcommand*{\arfcs}{\(\alpha\)-relaxed fair clusterings\xspace}
\newcommand*{\kbalpart}{\textsc{\(k\)-Balanced Partitioning}\xspace}
\newcommand*{\thrPart}{\textsc{3-Partition}\xspace}

\newcommand{\optPDef}[3]{%
      \noindent
      \begin{center}
      \begin{tcolorbox}
      \textsc{#1}\\[5pt]
      \begin{tabular}{p{0.13\columnwidth}p{0.8\columnwidth}}
      \textbf{Input:} & #2\\
      \textbf{Task:} & #3
      \end{tabular}
      \end{tcolorbox}
      \end{center}
      }

\renewcommand*{\P}{\textsf{P}\xspace}
\newcommand*{\NP}{\textsf{NP}\xspace}
\newcommand*{\APX}{\textsf{APX}\xspace}
\newcommand*{\PTAS}{\textsf{PTAS}\xspace}
\newcommand*{\FPT}{\textsf{FPT}\xspace}
\newcommand*{\XP}{\textsf{XP}\xspace}
\newcommand*{\poly}{\textsf{poly}}

\newcommand*{\bigO}{\mathrm{O}}
\newcommand*{\bigOmega}{\mathrm{\Omega}}

\newcommand{\N}{\mathbb{N}}
\newcommand{\R}{\mathbb{R}}
\newcommand{\Q}{\mathbb{Q}}
\newcommand{\Z}{\mathbb{Z}}
\newcommand{\mcP}{\mathcal{P}}
\newcommand{\mcS}{\mathcal{S}}
\newcommand{\cost}[1]{\text{cost}(#1)}

\newcommand{\comp}[1]{\(#1\)-component\xspace}
\newcommand{\rcomp}{\comp{r}}
\newcommand{\bcomp}{\comp{b}}
\newcommand{\rrcomp}{\comp{rr}}
\newcommand{\brcomp}{\comp{br}}
\newcommand{\brrcomp}{\comp{brr}}
\newcommand{\comps}[1]{\(#1\)-components\xspace}
\newcommand{\rcomps}{\comps{r}}
\newcommand{\bcomps}{\comps{b}}
\newcommand{\rrcomps}{\comps{rr}}
\newcommand{\brcomps}{\comps{br}}

\newcommand{\heads}{\{\emptyset, r, b, rr, br\}}

\newcommand{\join}{\textsc{Join}\xspace}
\newcommand{\joinroutine}{\join subroutine\xspace}

\newcommand{\svars}{\mathrm{setvars}\xspace}
\newcommand{\smax}{\mathrm{setmax}\xspace}

\newcommand*{\nwspace}{\hspace*{.1em}} 

\newcommand{\specialcell}[2][c]{%
	\def\arraystretch{1}%
	\begin{tabular}[#1]{@{}c@{}}#2\end{tabular}
}

\providecommand{\ignore}[1]{}

\title{Fair Correlation Clustering in Forests
}
\titlerunning{Fair Correlation Clustering}

\author{Katrin Casel}%
{Humboldt University of Berlin, Germany}%
{katrin.casel@hpi.de}%
{https://orcid.org/0000-0001-6146-8684}
{} 

\author{Tobias Friedrich}%
{Hasso Plattner Institute, University of Potsdam, Germany}%
{tobias.friedrich@hpi.de}%
{https://orcid.org/0000-0003-0076-6308}
{} 

\author{Martin Schirneck}%
{Faculty of Computer Science, University of Vienna, Austria}%
{martin.schirneck@univie.ac.at}%
{https://orcid.org/0000-0001-7086-5577}
{} 

\author{Simon Wietheger}%
{Hasso Plattner Institute, University of Potsdam, Germany}%
{simon.wietheger@student.hpi.de}%
{https://orcid.org/0000-0002-0734-0708}
{} 

\authorrunning{Casel, Friedrich, Schirneck \& Wietheger}
\Copyright{Katrin Casel, Tobias Friedrich, Martin Schirneck, and Simon Wietheger}



\category{} 

\relatedversion{} 

\acknowledgements{}

\EventEditors{Kunal Talwar}
\EventNoEds{1}
\EventLongTitle{4th Symposium on Foundations of Responsible Computing (FORC 2023)}
\EventShortTitle{FORC 2023}
\EventAcronym{FORC}
\EventYear{2023}
\EventDate{June 7--9, 2023}
\EventLocation{Stanford, CA, United States}
\EventLogo{}
\SeriesVolume{}
\ArticleNo{}

\begin{document}

\maketitle

\begin{abstract}	
The study of algorithmic fairness received growing attention recently. 
This stems from the awareness that bias in the input data for machine learning systems
may result in discriminatory outputs.
For clustering tasks, one of the most central notions of fairness is the formalization by Chierichetti, Kumar, Lattanzi, and Vassilvitskii~[NeurIPS 2017].
A clustering is said to be fair, if each cluster has the same distribution of manifestations of a sensitive attribute as the whole input set. 
This is motivated by various applications where the objects to be clustered have sensitive attributes that should not be over- or underrepresented. 
Most research on this version of fair clustering has focused on centriod-based objectives.

In contrast, we discuss the applicability of this fairness notion to \textsc{Correlation Clustering}. The existing literature on the resulting \textsc{Fair Correlation Clustering} problem either presents approximation algorithms with poor approximation guarantees or severely limits the possible distributions of the sensitive attribute (often only two manifestations with a 1:1 ratio are considered). 
Our goal is to understand if there is hope for better results in between these two extremes.  To this end, we consider restricted graph classes which allow us to characterize the distributions of sensitive attributes for which this form of fairness is tractable from a complexity point of view.

While existing work on \textsc{Fair Correlation Clustering} gives approximation algorithms, we focus on exact solutions and investigate whether there are efficiently solvable instances. 
The unfair version of \textsc{Correlation Clustering} is trivial on forests, but adding fairness creates a surprisingly rich picture of complexities. We give an overview of the distributions and types of forests where \textsc{Fair Correlation Clustering} turns from tractable to intractable.

As the most surprising insight, we consider the fact that
the cause of the hardness of \textsc{Fair Correlation Clustering} is not the strictness of the fairness condition. 
We lift most of our results to also hold for the relaxed version of the fairness condition. 
Instead, the source of hardness seems to be the distribution of the sensitive attribute. 
On the positive side, we identify some reasonable distributions that are indeed tractable. 
While this tractability is only shown for forests, it may open an avenue to design reasonable approximations for larger graph classes.
\end{abstract}

\section{Introduction}
\label{sec:intro}

In the last decade, the notion of fairness in machine learning has increasingly attracted interest, see for example the review by Pessach and Schmueli \cite{Pessach_Shmueli_2022}. Feldman, Friedler, Moeller, Scheidegger, and Venkatasubramanian  \cite{Feldman_Friedler_Moeller_Scheidegger_Venkatasubramanian_2015} formalize fairness based on a US Supreme Court decision on disparate impact from 1971. It requires that sensitive attributes like gender or skin color should neither be explicitly considered in decision processes like hiring but also should the manifestations of sensitive attributes be proportionally distributed in all outcomes of the decision process. Feldman et al.\ formalize this notion for classification tasks. Chierichetti, Kumar, Lattanzi, and Vassilvitskii \cite{Chierichetti_Kumar_Lattanzi_Vassilvitskii_2017} adapt this concept for clustering tasks. 

In this paper we employ the same disparate impact based understanding of fairness. Formally, the objects to be clustered have a color assigned to them that represents some sensitive attribute. Then, a clustering of these colored objects is called \emph{fair} if for each cluster and each color the ratio of objects of that color in the cluster  corresponds to the total ratio of vertices of that color. More precisely, a clustering is \emph{fair}, if it partitions the set of objects into \emph{fair subsets}.

\begin{definition}[Fair Subset]
    Let \(U\) be a finite set of objects colored by a function \(c\colon U\rightarrow [k]\)  for some \(k\in \N_{>0}\). Let \(U_i = \{u\in U\mid c(u)=i\}\) be the set of objects of color \(i\) for all \(i\in[k]\). Then, a set \(S\subseteq U\) is fair if and only if for all colors \(i\in [k]\) we have \(\frac{|S\cap U_i|}{|S|} = \frac{|U_i|}{|U|}\).
\end{definition}

To understand how this notion of fairness affects clustering decisions, consider the following example. Imagine that an airport security wants to find clusters among the travelers to assign to each group a level of potential risk  with corresponding anticipating measures. There are attributes like skin color that should not influence the assignment to a risk level. A bias in the data, however, may lead to some colors being over- or underrepresented in some clusters.  
Simply removing the skin color attribute from the data may not suffice as it may correlate with other attributes. Such problems are especially likely if one of the skin colors is far less represented in the data than others. 
A fair clustering finds the optimum clustering such that for each risk level the distribution of skin colors is fair, by requiring the distribution of each cluster to roughly match the distribution of skin colors among all travelers.

The seminal fair clustering paper by  Chierichetti et al.~\cite{Chierichetti_Kumar_Lattanzi_Vassilvitskii_2017} introduced this notion of fairness for clustering and studied it for the  objectives \(k\)-center and \(k\)-median. 
Their work was extended by Bera, Chakrabarty, Flores, and Negahbani \cite{Bera_Chakrabarty_Flores_Negahbani_2019}, who relax the fairness constraint in the sense of requiring upper and lower bounds on the representation of a color in each cluster. More precisely, they define the following generalization of fair sets.  
\begin{definition}[Relaxed Fair Set]\label{def:relaxes_fair_set}
For a finite set \(U\) and coloring \(c\colon U\rightarrow [k]\) for some \(k\in \N_{>0}\) 
let \(p_i,q_i\in \Q\) with \(0<p_i\le \frac{|U_i|}{|U|} \le q_i < 1\) for all \(i\in [k]\), where \(U_i = \{u\in U\mid c(u)=i\}\). A set \(S\subseteq U\) is relaxed fair with respect to \(q_i\) and \(p_i\) if and only if     \(p_i \le \frac{|S\cap U_i|}{|S|} \le q_i\)  for all \(i\in [k]\).
\end{definition}
Following these results, this notion of (relaxed) fairness was extensively studied for centroid-based clustering objectives with many positive results.

For example, Bercea et al.~\cite{Bercea_2018} give bicreteira constant-factor approximations for facility location type problems like $k$-center and $k$-median. 
Bandyapadhyay, Fomin and Simonov~\cite{Bandyapadhyay_Fomin_Simonov_2021} use the technique of fair coresets introduced by Schmidt, Schwiegelshohn, and Sohler~\cite{Schmidt_Schwiegelshohn_Sohler_2020} to give constant factor approximations for many centroid-based clustering objectives; among many other results, they give a PTAS for fair $k$-means and $k$-median in Euclidean space. Fairness for centroid-based objectives seems to be so well understood, that most research already considers more generalized settings, like streaming~\cite{Schmidt_Schwiegelshohn_Sohler_2020}, or imperfect knowledge of group membership~\cite{Esmaeili_Brubach_Tsepenekas_Dickerson_2020}.

In comparison, there are few (positive) results for this fairness notion applied to graph clustering objectives. The most studied with respect to fairness among those is \cc, arguably the most studied graph clustering objective. For \cc we are given a pairwise similarity measure for a set of objects and the aim is to find a clustering that minimizes the number of similar objects placed in separate clusters and the number of dissimilar objects placed in the same cluster.
Formally, the input to \cc is  a graph \(G=(V,E)\), and the goal is to find a partition~\(\mcP\) of~\(V\) that minimizes the \cccost defined as 
\begin{gather}\label{eq:cccost}
    \cost{G,\mcP} = |\nwspace \{\{u,v\}\in \tbinom{V}{2}\setminus E \mid \mcP[u] = \mcP[v] \} \nwspace |
    	 + |\nwspace  \{\{u,v\}\in E \mid \mcP[u] \neq \mcP[v]\} \nwspace |.
\end{gather}

\fcc then is the task to find a partition into \emph{fair} sets that minimizes the \cccost. We emphasize that this is the complete, unweighted, min-disagree form of \cc. (It is often called \emph{complete} because every pair of objects is either similar or dissimilar but none is indifferent regarding the clustering.
It is unweighted as the (dis)similarity between two vertices is binary. 
A pair of similar objects that are placed in separate clusters as well as a pair of dissimilar objects in the same cluster is called a \emph{disagreement}, hence the naming of the min-disagree form.)

There are two papers that appear to have started studying \fcc independently\footnote{Confusingly, they both carry the title \emph{Fair Correlation Clustering}.}.
Ahmadian, Epasto, Kumar, and Mahdian~\cite{Ahmadian_Epasto_Kumar_Mahdian_2020} analyze settings where the fairness constraint is given by some \(\alpha\) and require that the ratio of each color in each cluster is at most \(\alpha\). For \(\alpha=\frac{1}{2}\), which corresponds to our fairness definition if there are two colors in a ratio of \(1:1\), they obtain a 256-approximation. For \(\alpha=\frac{1}{k}\), where \(k\) is the number of colors in the graph, they give a \(16.48k^2\)-approximation. We note that all their variants are only equivalent to our fairness notion if there are \(\alpha^{-1}\) colors that all occur equally often.
Ahmadi, Galhotra, Saha, and Schwartz~\cite{Ahmadi_Galhotra_Saha_Schwartz_2020}  give an $\bigO{(c^2})$-approximation algorithm for instances with two colors in a ratio of \(1:c\). In the special case of a color ratio of \(1:1\),  they obtain a \(3\beta+4\)-approximation, given any \(\beta\)-approximation to unfair \cc.  With a more general color distribution, their approach also worsens drastically. For instances with \(k\) colors in a ratio of \(1:c_2:c_3:\ldots:c_k\) for positive integers \(c_i\), they give an $\bigO(k^2\cdot \max_{2\le i\le k} c_i)$-approximation for the strict, and an $\bigO(k^2\cdot \max_{2\le i\le k} q_i)$-approximation for the relaxed setting\footnote{Their theorem states they achieve an $\bigO(\max_{2\le i\le k} q_i)$-approximation but when looking at the proof it seems they have accidentally forgotten the \(k^2\) factor.}.

Following these two papers, Friggstad and Mousavi~\cite{Friggstad_Mousavi_2021} provide an approximation to the \(1:1\) color ratio case with a factor of \(6.18\). To the best of our knowledge, the most recent publication on \fcc is by Ahmadian and Negahbani~\cite{Ahmadian_Negahbani} who give approximations for \fcc with a slightly different way of relaxing fairness. They give an approximation with ratio $\mathcal O(\varepsilon^{-1} k\max_{2\le i\le k} c_i)$ for color distribution \(1:c_2:c_3:\ldots:c_k\), where $\varepsilon$ relates to the amount of relaxation (roughly $q_i=(1+\epsilon)c_i$ for our definition of relaxed fairness).  

All these results for \fcc seem to converge towards considering the very restricted setting of two colors in a ratio of \(1:1\) in order to give some decent approximation ratio.
In this paper, we want to understand if this is unavoidable, or if there is hope to find better results for other (possibly more realistic) color distributions. In order to isolate the role of fairness, we consider ``easy'' instances for \cc, and study the increase in complexity when adding fairness constraints. \cc without the fairness constraint is easily solved on forests. We find that \fcc restricted to forests turns \NP-hard very quickly, even when additionally assuming constant degree or diameter. Most surprisingly, this hardness essentially also holds for relaxed fairness,  showing that the hardness of the problem is not due to the strictness of the fairness definition.

On the positive side, we identify color distributions that allow for efficient algorithms. Not surprisingly, this includes ratio \(1:1\), and extends to a constant number of $k$ colors with distribution \(c_1:c_2:c_3:\ldots:c_k\) for  constants $c_1,\dots, c_k$. Such distributions can be used to model sensitive attributes with a limited number of manifestation that are almost evenly distributed. Less expected, we also find tractability for, in a sense, the other extreme. We show that \fcc on forests can be solved in polynomial time for two colors with ratio \(1:c\) with $c$ being very large (linear in the number of overall vertices). Such a distribution can be used to model a scenario where a minority is drastically underrepresented and thus in dire need of fairness constraints. Although our results only hold for forests, we believe that they can offer a starting point for more general graph classes. We especially hope that our work sparks interest in the so far neglected distribution of ratio \(1:c\) with $c$ being very large.

\subsection{Related Work}
\label{subsec:intro_related_work}
The study of clustering objectives similar or identical to \cc dates back to the 1960s \cite{BenDor_Shamir_Yakhini_1999,Regnier_1983,Zahn_1964}.
Bansal, Blum, and Chawla~\cite{Bansal_Blum_Chawla_2004} were the first to coin the term \cc as a clustering objective. 
We note that it is also studied under the name \ce. The most general formulation of \cc regarding weights considers 
two positive real values for each pair of vertices, the first to be added to the cost if the objects are placed in the same cluster and the second to be added if the objects are placed in separate clusters \cite{Ailon_Charikar_Newman_2008}. The recent book by Bonchi, García-Soriano, and Gullo \cite{Bonchi_GarciaSoriano_Gullo_2022} gives a broad overview of the current research on \cc. 

We focus on the particular variant that considers a complete graph with $\{-1,1\}$ edge-weights, and the min disagreement objective function. This version is \APX-hard \cite{Charikar_Guruswami_Wirth_2005}, implying in particular that there is no algorithm giving an arbitrarily good approximation unless $\P = \NP$. The best known approximation for \cc is the very recent breakthrough  by Cohen{-}Addad, Lee and  Newman~\cite{bestccapx} who give a ratio of $(1.994+\epsilon)$.

We show that in forests, all clusters of an optimal \cc solution have a fixed size. 
In such a case, \cc is related to   \kbalpart. There, the task is to partition the graph into \(k\) clusters of equal size while minimizing the number of edges that are cut by the partition. Feldmann and Foschini \cite{Feldmann_Foschini_2015} study this problem on trees and their results have interesting parallels with ours.

Aside from the results on \fcc already discussed  above, we are only aware of three papers that consider a fairness notion close to the one of Chierichetti et al.~\cite{Chierichetti_Kumar_Lattanzi_Vassilvitskii_2017} for a graph clustering objective.
Schwartz and Zats~\cite{Schwartz_Zats_2022} consider incomplete \fcc with the max-agree objective function. 
Dinitz, Srinivasan, Tsepenekas, and Vullikanti \cite{Dinitz_Srinivasan_Tsepenekas_Vullikanti_2022} study \textsc{Fair Disaster Containment}, a graph cut problem involving fairness. Their problem is not directly a fair clustering problem since they only require  one part of their partition (the saved part) to be fair.
Ziko, Yuan, Granger, and Ayed \cite{Ziko_Yuan_Granger_Ayed_2021}  give a heuristic 
 approach for fair clustering in general that however does not allow for theoretical guarantees on the quality of the solution.

\section{Contribution}
\label{sec:contribution}
We now outline our findings on \fcc.
We start by giving several structural results that underpin our further investigations.
Afterwards, we present our algorithms and hardness results for certain graph classes and color ratios.
We further show that the hardness of fair clustering does \emph{not} stem from
the requirement of the clusters exactly reproducing the color distribution of the whole graph.
This section is concluded by a discussion of possible directions for further research.

\subsection{Structural Insights}
\label{subsec:contrib_structural}
\begin{figure}
    \centering
    \includegraphics[height=.15\textwidth]{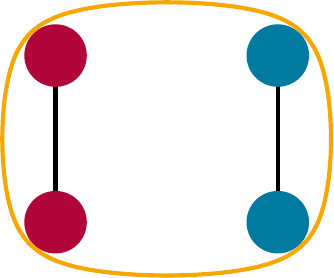}
    \hspace{.1\textwidth}
    \includegraphics[height=.15\textwidth]{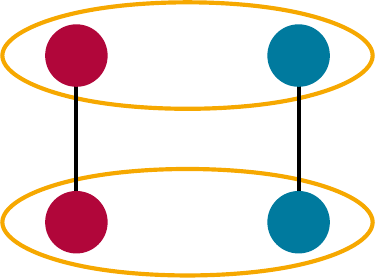}
    \caption{Example forest where a cluster of size 4 and two clusters of size 2 incur the same cost. With one cluster of size 4 (left), the inter-cluster cost is 0 and the intra-cluster cost is 4. With two clusters of size 2 (right), both the inter-cluster and intra-cluster cost are 2.}
\label{fig:1to1_4is2}
\end{figure}

We outline here the structural insights that form the foundation of all our results.
We first give a close connection between the cost of a clustering,
the number of edges ``cut'' by a clustering,
and the total number of edges in the graph.
We refer to this number of ``cut'' edges as the \emph{inter-cluster} cost as opposed to the number of non-edges inside clusters, which we call the \emph{intra-cluster} cost. 
Formally, the intra- and inter-cluster cost are the first and second summand of the \cccost in \Cref{eq:cccost}, respectively.
The following lemma shows that minimizing the inter-cluster cost suffices to minimize the total cost if all clusters are of the same size.
This significantly simplifies the algorithm development for \cc.

\begin{restatable}{lemma}{costbycuts}
\label{lem:costByCuts}
    Let \(\mcP\) be a partition of the vertices of an \(m\)-edge graph \(G\). Let \(\chi\) denote the inter-cluster cost incurred by \(\mcP\) on \(G\).
    If all sets in the partition are of size \(d\), then \(\cost{\mcP}=\frac{(d-1)}{2} \nwspace n -m +2\chi\).
    In particular, if \(G\) is a tree, \(\cost{\mcP} = \frac{(d-3)}{2} \nwspace n + 2\chi +1\).
\end{restatable}

The condition that all clusters need to be of the same size seems rather restrictive at first.
However, we prove in the following that in bipartite graphs and, in particular, in forests and trees
there is always a minimum-cost \fc such that indeed all clusters are equally large.
This property stems from how the fairness constraint acts on the distribution of colors
and is therefore specific to \fcc.
It allows us to fully utilize \autoref{lem:costByCuts}
both for building reductions in \NP-hardness proofs as well as for algorithmic approaches 
as we can restrict our attention to partitions with equal cluster sizes. 

Consider two colors of ratio \(1:2\), then any fair cluster must contain at least
1 vertex of the first color and 2 vertices of the second color to fulfil the fairness requirement.
We show that a minimum-cost clustering of a forest, due to the small number of edges,
consists entirely of such minimal clusters.
Every clustering with larger clusters incurs a higher cost.

\begin{restatable}{lemma}{smallclustersforest}
\label{lem:smallClustersForest}
    Let \(F\) be a forest with \(k\ge 2\) colors in a ratio of \(c_1:c_2:\ldots:c_k\) with \(c_i\in \N_{>0}\) for all \(i\in[k]\), \(\gcd(c_1,c_2,\ldots,c_k)=1\), and \(\sum_{i=1}^k c_i \ge 3\). Then, all clusters of every minimum-cost \fc are of size \(d = \sum_{i=1}^k c_i\). 
\end{restatable}

\autoref{lem:smallClustersForest} does not extend to two colors in a ratio of \(1:1\)
as illustrated in \autoref{fig:1to1_4is2}. 
In fact, this color distribution is the only case for forests where a partition with larger clusters 
can have the same (but no smaller) cost. 
We prove a slightly weaker statement than \autoref{lem:smallClustersForest},
namely, that \emph{there is} always a minimum-cost \fc whose cluster sizes are given by the color ratio.
We find that this property, in turn, holds not only for forests but for every bipartite graph.
Note that in general bipartite graphs there are more color ratios than only $1:1$ that allow for
these ambiguities.

\begin{restatable}{lemma}{smallclustersbipartite}
\label{lem:smallClustersBipartite}
    Let \(G=(A\cup B, E)\) be a bipartite graph with \(k\ge 2\) colors in a ratio of \(c_1:c_2:\ldots:c_k\) with \(c_i\in \N_{>0}\) for all \(i\in[k]\) and \(\gcd(c_1,c_2,\ldots,c_k)=1\). Then, there is a minimum-cost \fc such that all its clusters are of size \(d=\sum_{i=1}^k c_i\). Further, each minimum-cost \fc with larger clusters can be transformed into a minimum-cost \fc such that all clusters contain no more than \(d\) vertices in linear time.
\end{restatable}

In summary, the results above show that the ratio of the color classes is the key parameter
determining the cluster size.
If the input is a bipartite graph
whose vertices are colored with~$k$ colors in a ratio of $c_1 : c_2 : \dots : c_k$, our results imply that without loosing optimality, solutions can be restricted to contain only clusters of size $d = \sum_{i=1}^k c_i$, each with exactly $c_i$ vertices of color $i$.
Starting from these observations, we show in this work that the color ratio is also
the key parameter determining the complexity of \fcc.
On the one hand, the simple structure of optimal solutions restricts the search space
and enables polynomial-time algorithms, at least for some instances.
On the other hand, these insights allow us to show hardness already for very restricted input classes. 
The technical part of most of the proofs consists of exploiting the connection between
the clustering cost, total number of edges,
and the number of edges cut by a clustering.

\subsection{Tractable Instances}
\label{subsec:contrib_feasible}

\setlength{\belowrulesep}{.5em}

\begin{table}
\caption{%
	Running times of our algorithms for \fcc on forests depending on the color ratio.
 	Value $p$ is any rational such that $\sfrac{n}{p}-1$ is integral;
 	$c_1, c_2, \dots, c_k$ are coprime positive integers, possibly depending on $n$.
    Functions \(f\) and \(g\) are 
    given in \Cref{thm:forestByColorsAlgo,thm:forestLarge1_CAlgo}.
}

\centering
\begin{tabular}{lccccc}
	\toprule
	Color Ratio		& \(1:1\) & \(1:2\) & 
		\(1:(\sfrac{n}{p} \,{-}\, 1)\) & \(c_1:c_2:\ldots:c_k\)\\[.5em]
	Running Time	& $\bigO(n)$ & $\bigO(n^6)$ & 
		$\bigO\!\left(n^{f(p)} \right)$ & $\bigO\!\left(n^{g(c_1,\dots,c_k)}\right)$\\[.5em]
	\bottomrule
\end{tabular}
\label{table:resultsNoRestriction}
\end{table}

We start by discussing the algorithmic results.
The simplest case is that of two colors, each one occurring equally often.
We prove that for bipartite graphs with a color ratio $1:1$
\fcc is equivalent to the maximum bipartite matching problem,
namely, between the vertices of different color.
Via the standard reduction to computing maximum flows,
this allows us to benefit from the recent breakthrough by
Chen, Kyng, Liu, Peng, Probst Gutenberg, and Sachdeva~\cite{Chen22MaxFlowAlmostLinear}.
It gives an algorithm running in time $m^{1+o(1)}$.

The remaining results focus on forests as the input, see \autoref{table:resultsNoRestriction}.
It should not come as a surprise that our main algorithmic paradigm is dynamic programming.
A textbook version finds a maximum matching in linear time in a forests, solving the $1:1$ case.
For general color ratios, we devise much more intricate dynamic programs.
We use the color ratio $1:2$ as an introductory example.
The algorithm has two phases.
In the first, we compute a list of candidate \emph{splittings}
that partition the forest into connected parts
containing at most 1 blue and 2 red vertices each.
In the second phase, we assemble the parts of each of the splittings 
to fair clusters and return the cheapest resulting clustering.
The difficulty lies in the two phases not being independent from each other.
It is not enough to minimize the ``cut'' edges in the two phases separately.
We prove that the costs incurred by the merging additionally depends on the number of 
of parts of a certain type generated in the splittings.
Tracking this along with the number of cuts results in a $\bigO(n^6)$-time algorithm.
Note that we did not optimize the running time as long as it is polynomial.

We generalize this to $k$ colors in a ratio $c_1 : c_2 : \dots : c_k$.\footnote{%
	The $c_i$ are coprime, but they are not necessarily constants with respect to $n$.
}
We now have to consider \emph{all} possible colorings 
of a partition of the vertices such that in each part the $i$-th color occurs at most $c_i$ times.
While assembling the parts, we have to take care that the merged colorings remain compatible.
The resulting running time is $\bigO(n^{g(c_1,\dots,c_k)})$ for some (explicit) polynomial $g$.
Recall that, by \autoref{lem:smallClustersForest}, the minimum cluster size is $d = \sum_{i=1}^k c_i$.
If this is a constant, then the dynamic program runs in polynomial time.
If, however, the number of colors $k$ or some color's proportion grows with $n$,
it becomes intractable.
Equivalently, the running time gets worse if there are very large but sublinearly many clusters.

To mitigate this effect,
we give a complementary algorithm at least for forests with two colors.
Namely, consider the color ratio $1: \frac{n}{p}-1$.
Then, an optimal solution has $p$ clusters each of size $d = \sfrac{n}{p}$.
The key observation is that the forest contains $p$ vertices of the color with fewer occurrences,
say, blue, and any fair clustering isolates the blue vertices.
This can be done by cutting at most $p-1$ edges
and results in a collection of (sub-)trees where each one has at most one blue vertex.
To obtain the clustering, we split the trees with red excess vertices
and distribute those among the remaining parts.
We track the costs of all the $\bigO(n^{\poly(p)})$ many cut-sets and rearrangements
to compute the one of minimum cost.
In total, the algorithm runs in time $\bigO(n^{f(p)})$ for some polynomial in $p$.
In summary, we find that if the number of clusters $p$ is constant, then the running time is polynomial.
Considering in particular an integral color ratio \(1:c\),\footnote{
	In a color ratio $1:c$, $c$ is not necessarily a constant,
	but ratios like $2:5$ are not covered.
}, we find tractability for forests if $c = \bigO(1)$ or \(c = \bigOmega(n)\). 
We will show next that \fcc with this kind of a color ratio is \NP-hard already on trees,
hence the hardness must emerge somewhere for intermediate $c$.

\subsection{A Dichotomy for Bounded Diameter}
\label{subsec:contrib_diameter}

\begin{table}[t]
\caption{Complexity of \fcc on trees and general graphs depending on the diameter.
	The value $c$ is a positive integer, possibly depending on $n$.}
    
\centering
\begin{tabular}{cccc}
	\toprule
	Diameter & Color Ratio & Trees & General Graphs\\[.25em]
	\hline\\[-.75em]
	\(2,3\) & any & $\bigO(n)$ & \NP-hard \\[.25em]
	\(\ge 4\) & \(1:c\) & \NP-hard & \NP-hard\\[.25em]
	\bottomrule
    \end{tabular}
\label{table:resultsDiameter}
\end{table}

\autoref{table:resultsDiameter} shows the complexity of \fcc on graphs with bounded diameter.
We obtain a dichotomy for trees with two colors with ratio $1:c$.
If the diameter is at most $3$, an optimal clustering is computable in $\bigO(n)$ time,
but for diameter at least $4$, the problem becomes \NP-hard.
In fact, the linear-time algorithm extends to trees with an arbitrary number of colors in any ratio.

The main result in that direction is the hardness of \fcc already on trees with diameter at least 4 and two colors of ratio $1:c$.
This is proven by a reduction from the strongly \NP-hard \thrPart problem.
There, we are given positive integers $a_1, \dots, a_\ell$ where $\ell$ is a multiple of $3$
and there exists some $B$ with $\sum_{i = 1}^\ell a_i = B \cdot \frac{\ell}{3}$.
The task is to partition the numbers $a_i$ into triples such that each one of those sums to $B$.
The problem remains \NP-hard if all the $a_i$ are strictly between $\sfrac{B}{4}$ and $\sfrac{B}{2}$,
ensuring that, if some subset of the numbers sums to $B$, it contains exactly three elements.

We model this problem as an instance of \fcc as illustrated in \autoref{fig:treeOctopus_main}.
We build $\ell$ stars, where the $i$-th one consists of $a_i$ red vertices,
and a single star of $\sfrac{\ell}{3}$ blue vertices.
The centers of the blue star and all the red stars are connected.
The color ratio in the resulting instance is $1:B$.
\autoref{lem:smallClustersForest} then implies that
there is a minimum-costs clustering with $\sfrac{\ell}{3}$ clusters,
each with a single blue vertex and $B$ red ones.
We then apply \autoref{lem:costByCuts} to show that 
this cost is below a certain threshold if and only if
each cluster consist of exactly three red stars (and an arbitrary blue vertex),
solving \thrPart.

\begin{figure}
        \centering
        \includegraphics[height=.26\textwidth]{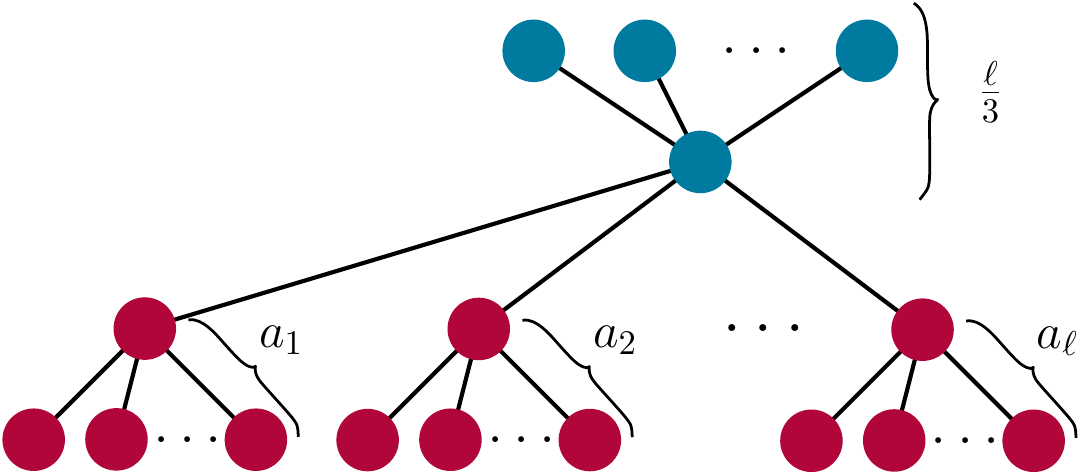}
        \caption{The tree with diameter 4 in the reduction from \thrPart to \fcc.}
    \label{fig:treeOctopus_main}
\end{figure}

\subsection{Maximum Degree}
\label{subsec:contrub_degree}

The reduction above results in a tree with a low diameter but arbitrarily high maximum degree.
We have to adapt our reductions to show hardness also for bounded degrees.
The results are summarized in \autoref{table:resultsDegree}.
If the \fcc instance is not required to be connected, we can represent \thrPart 
with a forest of trees with maximum degree $2$, that is, a forest of paths.
The input numbers are modeled by paths with $a_i$ vertices.
The forest also contains $\sfrac{\ell}{3}$ isolated blue vertices, which again implies that
an optimal fair clustering must have $\sfrac{\ell}{3}$ clusters each with $B$ red vertices.
By defining a sufficiently small cost threshold, 
we ensure that the fair clustering has cost below it if and only if none of the path-edges are ``cut''
by the clustering, corresponding to a partition of the $a_i$.

There is nothing special about paths, we can arbitrarily restrict the shape of the trees,
as long it is always possible to form such a tree with a given number of vertices.
However, the argument crucially relies on the absence of edges between the $a_i$-paths/trees
and does not transfer to connected graphs.
This is due to the close relation between inter-cluster costs
and the total number of edges stated in \autoref{lem:costByCuts}.
The complexity of \fcc on a single path with a color ratio $1:c$ therefore remains open. 
Notwithstanding, we show hardness for trees in two closely related settings:
keeping the color ratio at $1:c$ but raising the maximum degree to $5$,
or having a single path but a total of $\sfrac{n}{2}$ colors
and each color shared by exactly $2$ vertices.

For the case of maximum degree $5$ and two colors with ratio $1:c$,
we can again build on the \thrPart machinery.
The construction is inspired by how Feldmann and Foschini~\cite{Feldmann_Foschini_2015}
used the problem to show hardness of computing so-called $k$-balanced partitions.
We adapt it to our setting in which the vertices are colored
and the clusters need to be fair.

For the single path with $\sfrac{n}{2}$ colors, we reduce from
(the 1-regular 2-colored variant of) the 
\textsc{Paint Shop Problem for Words}~\cite{Epping_Hochstattler_Oertel_2004}.
There, a word is given in which every symbol appears exactly twice.
The task is to assign the values $0$ and $1$ to the letters of the word\footnote{%
	The original formulation~\cite{Epping_Hochstattler_Oertel_2004} assigns colors, 
	aligning better with the paint shop analogy.
	We change the exposition here in order to avoid confusion with
	the colors in the fairness sense.
}
such that that, for each symbol, exactly one of the two occurrences receives a $1$,
but the number of blocks of consecutive $0$s and $1$s
over the whole word is minimized.
In the translation to \fcc,
we represent the word as a path and the symbols as colors.
To remain fair, there must be two clusters containing exactly one vertex of each color,
translating back to a $0/1$-assignment to the word.

\begin{table}[t]
\caption{Hardness of \fcc on trees and forests depending on the maximum degree.
	The value $c$ is a positive integer, possibly depending on $n$.
	The complexity for paths (trees with maximum degree $2$) 
	with color ratio $1:c$ is open.}

\centering
\begin{tabular}{cccc}
	\toprule
	Max.\ Degree & Color Ratio & Trees & Forests\\[.25em]
    \hline\\[-.75em]
	\(2\) & \(1:c\) &  & \NP-hard \\[.25em]
	\(\ge 2\) & \specialcell{\(n/2\) colors,\\2 vertices each} & \NP-hard & \NP-hard \\[.75em]
	\(\ge 5\) & \(1:c\) & \NP-hard & \NP-hard \\[.25em]
    \bottomrule
\end{tabular}
\label{table:resultsDegree}
\end{table}

\subsection{Relaxed Fairness}
\label{subsec:contrib_relaxed}

One could think that the hardness of \fcc already for classes of trees and forests 
has its origin in the strict fairness condition.
After all, the color ratio in each cluster must precisely mirror that of the whole graph.
This impression is deceptive.
Instead, we lift most of our hardness results to \rfcc considering the \emph{relaxed fairness}
of Bera~et~al.~\cite{Bera_Chakrabarty_Flores_Negahbani_2019}.
Recall \autoref{def:relaxes_fair_set}.
It prescribes two rationals $p_i$ and $q_i$ for each color $i$
and allows, the proportion of $i$-colored elements  in any cluster to be in the interval $[p_i,q_i]$,
instead of being precisely $\sfrac{c_i}{d}$, where $d = \sum_{j=1}^k c_j$.  

The main conceptual idea is to show that, in some settings but not all,
the \emph{minimum-cost} solution under a relaxed fairness constraint is in fact \emph{exactly} fair.
This holds for the settings described above where we reduce from \thrPart.
In particular, \rfcc with a color ratio of $1:c$ is \NP-hard on trees with diameter $4$
and forests of paths, respectively.
Furthermore, the transferal of hardness is immediate for the case of a single path with $\sfrac{n}{2}$ colors
and exactly $2$ vertices of each color.
Any relaxation of fairness still requires one vertex of each color in every cluster,
maintaining the equivalence to the \textsc{Paint Shop Problem for Words}.

In contrast, algorithmic results are more difficult to extend
if there are relaxedly fair solutions that have lower cost than any exactly fair one.
We then no longer know the cardinality of the clusters in an optimal solution.
As a proof of concept, we show that a slight adaption of our dynamic program for two colors in a ratio of $1:1$ still works for what we call $\alpha$-\emph{relaxed fairness}.\footnote{%
	This should not be confused with the notion of $\alpha$\emph{-fairness}
	in resource allocation~\cite{JangYang22AlphaFairness,KumariSingh18AlphaFairness}.
}
There, the lower fairness ratio is $p_i = \alpha \cdot \frac{c_i}{d}$
and the upper one is $q_i = \frac{1}{\alpha} \cdot \frac{c_i}{d}$ for some parameter $\alpha \in (0,1)$.
We give an upper bound on the necessary cluster size depending on $\alpha$,
which is enough to find a good splitting of the forest.
Naturally, the running time now also depends on $\alpha$, but is of the form $O(n^{h(1/\alpha)})$
for some polynomial $h$.
In particular, we get an polynomial-time algorithm for constant $\alpha$. 
The proof of correctness consists of an exhaustive case distinction
already for the simple case of $1:1$.
We are confident that this can be extended to more general color ratios,
but did not attempt it in this work.

\subsection{Summary and Outlook}
\label{subsec:contrib_summary}

We show that \fcc on trees, and thereby forests, is \NP-hard. 
It remains so on trees of constant degree or diameter, 
and--for certain color distributions--it is also \NP-hard on paths. 
On the other hand,
we give a polynomial-time algorithm if the minimum size \(d\) of a fair cluster is constant. We also provide an efficient algorithm for the color ratio \(1:c\) if the total number of clusters is constant, corresponding to \(c\in\Theta(n)\).
For our main algorithms and hardness results, we prove that they still hold when the fairness constraint is relaxed, so the hardness is not due to the strict fairness definition.
%
%
Ultimately, we hope that the insights gained from these proofs as well as our proposed algorithms prove helpful to the future development of algorithms to solve \fcc on more general graphs.
In particular, fairness with color ratio \(1:c\) with $c$ being very large seems to be an interesting and potentially tractable type of distribution for future study.

As first steps to generalize our results,
we give a polynomial-time approximation scheme (PTAS) for \fcc on forests. 
Another avenue for future research could be that \autoref{lem:smallClustersBipartite},
bounding the cluster size of optimal solutions, extends also to bipartite graphs.
This may prove helpful in developing exact algorithms for bipartite graphs with other 
color ratios than \(1:1\).

Parameterized algorithms are yet another approach to solving more general instances. 
When looking at the decision version of \fcc,
our results can be cast as an \XP-algorithm when the problem is parameterized by the cluster size $d$,
for it can be solved in time $\bigO(n^{g(d)})$ for some function $g$.
Similarly, we get an \XP-algorithm for the number of clusters as parameter.
We wonder whether \fcc can be placed in the class \FPT of fixed-parameter tractable problems
for any interesting structural parameters.
This would require a running time of, e.g., $g(d) \cdot \poly(n)$.
There are \FPT-algorithms for \ce parameterized by the cost of the solution~\cite{Boecker_Baumbach_2013}. 
Possibly, future research might provide similar results for the fair variant as well.
A natural extension of our dynamic programming approach could potentially lead to 
an algorithm parameterizing by the treewidth of the input graph.
Such a solution would be surprising, however, since to the best of our knowledge 
even for normal, unfair \cc\footnote{%
	In more detail, no algorithm for complete \cc has been proposed. 
	Xin \cite{Xin_2011} gives a treewidth algorithm for \emph{incomplete} \cc for the treewidth of the graph of all positively and negatively labeled edges.
}
and for the related \textsc{Max Dense Graph Partition}~\cite{Darlay_Brauner_Moncel_2012} no treewidth approaches are known.

Finally, it is interesting how \fcc behaves on paths. 
While we obtain \NP-hardness for a particular color distribution from the \textsc{Paint Shop Problem For Words}, the question of whether \fcc on paths with for example two colors in a ratio of \(1:c\) is efficiently solvable or not is left open. 
However, we believe that this question is rather answered by the study of the related (discrete) \textsc{Necklace Splitting} problem, see the work of Alon and West~\cite{Alon_West_1986}.
There, the desired cardinality of every color class is explicitly given,
and it is non-constructively shown that there always exists a split of the necklace
with the number of cuts meeting the obvious lower bound.
A constructive splitting procedure may yield some insights for \fcc on paths.

\vspace*{1em}
\noindent

\section{Preliminaries}
\label{sec:prelims}
We fix here the notation we are using for the technical part and give the formal definition of \fcc.

\subsection{Notation}
\label{subsec:notation}

We refer to the set of natural numbers \(\{0,1,2,\ldots\}\) by \(\N\).
For \(k\in \N\), let \([k]=\{1,2,\ldots, k\}\) and \(\N_{>k}= \N\setminus\left(\{0\}\cup [k]\right)\). We write \(2^{[k]}\) for the power set of \([k]\).
By \(\gcd(a_1,a_2,\ldots,a_k)\) we denote the \emph{greatest common divisor} of \(a_1,a_2\ldots,a_k\in\N\). 

An \emph{undirected graph} \(G=(V,E)\) is defined by a set of vertices \(V\) and a set of edges \(E\subseteq \binom{V}{2} = \{\{u,v\} \mid u,v\in V, u\neq v\}\). If not stated otherwise, by the \emph{size of \(G\)} we refer to \(n+m\), where \(n=|V|\) and \(m=|E|\). A graph is called \emph{complete} if \(m = \frac{n(n-1)}{2}\). We call a graph \(G=(A\cup B, E)\) \emph{bipartite} if there are no edges in \(A\) nor \(B\), i.e., \(E \cap \binom{A}{2} = E \cap \binom{B}{2} = \emptyset\).
For every \(S\subseteq V\), we let \(G[S] = \left(S, E\cap\binom{S}{2}\right)\) denote the \emph{subgraph induced by \(S\)}.
The \emph{degree} of a vertex \(v\in V\) is the number of edges incident to that vertex, \(\delta(v)=|\{u \mid \{u,v\}\in E\}|\). The \emph{degree} of a graph \(G=(V,E)\) is the maximum degree of any of its vertices \(\delta(G)=\max_{v\in V} \delta(v)\).
A \emph{path} of length \(k\) in \(G\) is a tuple of vertices \((v_1, v_2, \ldots, v_{k-1})\) such that for each \(1\le i< k-1\) we have \(\{v_i,v_{i+1}\}\in E\). 
We only consider simple paths, i.e., we have \(v_i\neq v_j\) for all \(i\neq j\). A graph is called \emph{connected} if for every pair of vertices \(u,v\) there is a path connecting \(u\) and \(v\).
The \emph{distance} between two vertices is the length of the shortest path connecting these vertices and the \emph{diameter} of a graph is the maximum distance between a pair of vertices.  A \emph{circle} is a path \((v_1, v_2, \ldots, v_k)\) such that \(v_1 = v_k\) and \(v_i\neq v_j\) only for all other pairs of \(i\neq j\). 

A \emph{forest} is a graph without circles. A connected forest is called a \emph{tree}. There is exactly one path connecting every pair of vertices in a tree. A tree is \emph{rooted} by choosing any vertex \(r\in V\) as the root. Then, every vertex \(v\), except for the root, has a \emph{parent}, which is the next vertex on the path from \(v\) to \(r\). All vertices that have \(v\) as a parent are referred to as the \emph{children} of \(v\). A vertex without children is called a \emph{leaf}. Given a rooted tree \(T\), by \(T_v\) we denote the subtree induced by \(v\) and its descendants, i.e., the set of vertices such that there is a path starting in \(v\) and ending in that vertex without using the edge to \(v\)'s parent. Observe that each forest is a bipartite graph, for example by placing all vertices with even distance to the root of their respective tree on one side and the other vertices on the other side.

A finite set \(U\) can be \emph{colored} by a function \(c: U\rightarrow [k]\), for some \(k\in\N_{>0}\). If there are only two colors, i.e., \(k=2\), for convenience we call them \emph{red} and \emph{blue}, instead by numbers.

For a \emph{partition} \(\mcP = \{S_1, S_2, \ldots, S_k\}\) with \(S_i\cap S_j = \emptyset\) for \(i\neq j\) of some set \(U = S_1\cup S_2\cup \ldots \cup S_k\) and some \(u\in U\) we use \(\mcP[u]\) to refer to the set \(S_i\) for which \(u\in S_i\). Further, we define the term \emph{coloring} on sets and partitions. The \emph{coloring of a set} counts the number of occurrences of each color in the set.  

\begin{definition}[Coloring of Sets]
    Let \(S\) be a set colored by a function \(c\colon S\rightarrow[k]\). Then, the coloring of \(S\) is an array \(C_S\) such that \(C_S[i] = |{\{s\in S\mid c(s)=i}\}|\) for all \(i\in[k]\).
\end{definition}
The \emph{coloring of a partition} counts the number of occurrences of set colorings in the partition. 
\begin{definition}[Coloring of Partitions]\label{def:coloringPartition}
    Let \(U\) be a colored set and let \(\mcP\) be a partition of \(U\). Let \(\mathcal{C} = \{C_S \mid S\subseteq U\}\) denote the set of set colorings for which there is a subset of \(U\) with that coloring.
    By an arbitrarily fixed order, let \(C_1,C_2,\ldots,C_\ell\) denote the elements of \(\mathcal{C}\).
    Then, the coloring of \(\mcP\) is an array \(C_\mcP\) such that \(C_\mcP[i] = |\{S\in \mcP\mid C_S = C_i\}|\) for all \(i\in[\ell]\).
\end{definition}

\subsection{Problem Definitions}
\label{subsec:problemDef}

In order to define \fcc, we first give a formal definition of the unfair clustering objective.
\cc receives a pairwise similarity measure for a set of objects and aims at minimizing the number of similar objects placed in separate clusters and the number of dissimilar objects placed in the same cluster. For the sake of consistency, we reformulate the definition of Bonchi et al.\ \cite{Bonchi_GarciaSoriano_Gullo_2022} such that the pairwise similarity between objects is given by a graph rather than an explicit binary similarity function. 
Given a graph \(G=(V,E)\) and a partition \(\mcP\) of \(V\), the \cccost is 
\begin{gather*}
    \cost{G,\mcP} = |\nwspace \{\{u,v\}\in \tbinom{V}{2}\setminus E \mid \mcP[u] = \mcP[v] \} \nwspace |
    	 + |\nwspace  \{\{u,v\}\in E \mid \mcP[u] \neq \mcP[v]\} \nwspace |.
\end{gather*}
We refer to the first summand as the \emph{intra-cluster cost} \(\psi\) and the second summand as the \emph{inter-cluster cost} \(\chi\). Where \(G\) is clear from context, we abbreviate to \(\cost{\mcP}\).  Sometimes, we consider the cost of \(\mcP\) on an induced subgraph. To this end, we allow the same cost definition as above also if \(\mcP\) partitions some set \(V'\supseteq V\).
We define (unfair) \cc as follows.

\optPDef{\cc}
{Graph \(G=(V,E)\).}
{Find a partition \(\mcP\) of \(V\) that minimizes \(\cost{\mcP}\).}

We emphasize that this is the complete, unweighted, min-disagree form of \cc. It is complete as every pair of objects is either similar or dissimilar but none is indifferent regarding the clustering.
It is unweighted as the (dis)similarity between two vertices is binary. 
A pair of similar objects that are placed in separate clusters as well as a pair of dissimilar objects in the same cluster is called a \emph{disagreement}, hence the naming of the min-disagree form. An alternative formulation would be the max-agree form with the objective to maximize the number of pairs that do not form a disagreement. Note that both formulations induce the same ordering of clusterings though approximation factors may differ because of the different formulations of the cost function. 

Our definition of the \fcc problem loosely follows \cite{Ahmadian_Epasto_Kumar_Mahdian_2020}. The fairness aspect limits the solution space to \emph{fair} partitions. A partition is fair if each of its sets has the same color distribution as the universe that is partitioned.

\begin{definition}[Fair Subset]
    Let \(U\) be a finite set of elements colored by a function \(c : U\rightarrow [k]\) for some \(k\in \N_{>0}\). Let \(U_i = \{u\in U\mid c(u)=i\}\) be the set of elements of color \(i\) for all \(i\in[k]\). Then, some \(S\subseteq U\) is fair if and only if for all colors \(i\in [k]\) we have \(\frac{|S\cap U_i|}{|S|} = \frac{|U_i|}{|U|}\).
\end{definition}
\begin{definition}[Fair Partition]
    Let \(U\) be a finite set of elements colored by a function \(c : U\rightarrow [k]\) for some \(k\in \N_{>0}\). Then, a partition \(S_1\cup S_2 \cup \ldots \cup S_\ell = U\) is fair if and only if all sets \(S_1, S_2, \ldots, S_\ell\) are fair.
\end{definition}

We now define complete, unweighted, min-disagree variant of the \fcc problem.
When speaking of (\textsc{Fair}) \cc, we refer to this variant,
unless explicitly stated otherwise.

\optPDef{\fcc}
{Graph $G = (V, E)$, coloring $c\colon V\rightarrow [k]$.}
{Find a fair partition $\mathcal{P}$ of \(V\) that minimizes \(\cost{\mcP}\).}

\section{Structural Insights}
\label{sec:structure}
We prove here the structural results outlined in \autoref{subsec:contrib_structural}.
The most important insight is that in bipartite graphs, and in forests in particular, there is always a minimum-cost \fc such that all clusters are of some fixed size. This property is very useful, as it helps for building reductions in hardness proofs as well as algorithmic approaches that enumerate possible clusterings. Further, by the following lemma, this also implies that minimizing the inter-cluster cost suffices to minimize the \cccost, which simplifies the development of algorithms solving \fcc on such instances.

\costbycuts*

\begin{proof}
    Note that in each of the \(\frac{n}{d}\) clusters there are \(\frac{d(d-1)}{2}\) pairs of vertices, each incurring an intra-cost of 1 if not connected by an edge. Let the total intra-cost be \(\psi\).
    As there is a total of \(m\) edges, we have 
    \begin{gather*}
        \cost{\mcP} = \chi + \psi = \chi + \frac{n}{d}\cdot \frac{d(d-1)}{2} -(m - \chi) 
        	= \frac{(d-1)n}{2} -m +2\chi. \qedhere
    \end{gather*}
\end{proof}

In particular, if \(G\) is a tree, this yields \(\cost{\mcP} = \frac{(d-3)n}{2} + 2\chi +1\) as there \(m=n-1\).

\subsection{Forests}
\label{sec:structural_forests}

We find that in forests in every minimum-cost partition all sets in the partition are of the minimum size required to fulfill the fairness requirement.

\smallclustersforest*

\begin{proof}
    Let \(d = \sum_{i=1}^k c_i\).
    For any clustering \(\mcP\) of \(V\) to be fair, all clusters must be at least of size \(d\). We show that if there is a cluster \(S\) in the clustering with \(|S|>d\), then we decrease the cost by splitting \(S\).
    First note that in order to fulfill the fairness constraint, we have \(|S| = ad\) for some \(a\in \N_{\ge 2}\).
    Consider a new clustering \(\mcP'\) obtained by splitting \(S\) into \(S_1, S_2\), where \(S_1\subset S\) is an arbitrary fair subset of \(S\) of size \(d\) and \(S_2 = S\setminus S_1\).
    Note that the cost incurred by every edge and non-edge with at most one endpoint in \(S\) is the same in both clusterings. Let \(\psi\) be the intra-cluster cost of \(\mcP\) on \(F[S]\). Regarding the cost incurred by the edges and non-edges with both endpoints in \(S\), we know that
    \begin{gather*}
        \cost{F[S], \mcP} \ge \psi \ge \frac{ad(ad-1)}{2}-(ad-1)
         	= \frac{a^2d^2 - 3ad+2}{2}
    \end{gather*}
    since the cluster is of size \(ad\) and as it is part of a forest it contains at most \(ad-1\) edges.
    In the worst case, the \(\mcP'\) cuts all the \(ad-1\) edges. However, we profit from the smaller cluster sizes. We have
    \begin{align*}
        \cost{F[S], \mcP'} = \chi + \psi 
        	&\le ad - 1 + \frac{d(d-1)}{2} + \frac{(a-1)d \cdot ((a-1)d - 1)}{2}\\
        	&= \frac{2d^2+ a^2d^2-2ad^2+ad-2}{2}.
    \end{align*}
    Hence, \(\mcP'\) is cheaper by
    \begin{gather*}
        \cost{F[S], \mcP} - \cost{F[S], \mcP'} \ge \frac{2ad^2-2d^2 - 4ad +4}{2}
        	= ad(d-2)-d^2+2.
    \end{gather*}
    This term is increasing in \(a\). As \(a\ge 2\), by plugging in \(a=2\), we hence obtain a lower bound of $\cost{F[S], \mcP} - \cost{F[S], \mcP'} \ge d^2-4d+2$.
     For \(d\ge 2\), the bound is increasing in \(d\) and it is positive for \(d> 3\). This means, if \(d>3\) no clustering with a cluster of size more than \(d\) has minimal cost implying that all optimum clusterings only consist of clusters of size \(d\). 

    Last, we have to argue the case \(d=3\), i.e., we have a color ratio of \(1:2\) or \(1:1:1\). In this case \(d^2-4d+2\) evaluates to \(-1\).
    However, we obtain a positive change if we do not split arbitrarily but keep at least one edge uncut. Note that this means that one edge less is cut and one more edge is present, which means that our upper bound on \(\cost{T[S], \mcP'}\) decreases by 2, so \(\mcP\) is now cheaper. 
    Hence, assume there is an edge \(\{u,v\}\) such that \(c(u)\neq c(v)\). Then by splitting \(S\) into \(\{u,v,w\}\) and \(S\setminus \{u,v,w\}\) for some vertex \(w\in S\setminus\{u,v\}\) that makes the component \(\{u,v,w\}\) fair, we obtain a cheaper clustering.
    If there is no such edge \(\{u,v\}\), then \(T[S]\) is not connected. This implies there are at most \(3a-3\) edges if the color ratio is \(1:1:1\) since no edge connects vertices of different colors and there are \(a\) vertices of each color, each being connected by at most \(a-1\) edges due to the forest structure. By a similar argument, there are at most \(3a-2\) edges if the color ratio is \(1:2\). Hence, the lower bound on \(\cost{T[S],\mcP}\) increases by 1. At the same time, even if \(\mcP'\) cuts all edges it cuts at most \(3a-2\) times, so it is at least 1 cheaper than anticipated. Hence, in this case \(\cost{T[S], \mcP'} < \cost{T[S], \mcP}\) no matter how we cut.
\end{proof}

Note that \Cref{lem:smallClustersForest} makes no statement about the case of two colors in a ratio of \(1:1\).

\subsection{Bipartite Graphs}
\label{subsec:structural_bipartite}

We are able to partially generalize our findings for trees to bipartite graphs. We show that there is still always a minimum-cost \fc with cluster sizes fixed by the color ratio. However, in bipartite graphs there may also be minimum-cost clusterings with larger clusters. We start with the case of two colors in a ratio of \(1:1\) and then generalize to other ratios. 

\begin{restatable}{lemma}{smallclustersbipartiteoneone}
\label{lem:smallClustersBipartiteOneOne}
    Let \(G=(A\cup B, E)\) be a bipartite graph with two colors in a ratio of \(1:1\). Then, there is a minimum-cost \fc of \(G\) that has no clusters with more than 2 vertices.  
    Further, each minimum-cost \fc can be transformed into a minimum-cost \fc such that all clusters contain no more than 2 vertices in linear time.
    If \(G\) is a forest, then no cluster in a minimum-cost \fc is of size more than 4.
\end{restatable}

\begin{proof}
     Note that, due to the fairness constraint, each \fc consists only of evenly sized clusters.
     We prove both statements by showing that in each cluster of at least 4 vertices there are always two vertices such that by splitting them from the rest of the cluster the cost does not increase and fairness remains.

     Let \(\mcP\) be a clustering and \(S\in\mcP\) be a cluster with \(|S|\ge 4\). Let \(S_A = S\cap A\) and \(S_B = S\cap B\). 
     Assume there is \(a\in S_a\) and \(b\in S_b\) such that \(a\) and \(b\) have not the same color. Then, the clustering \(\mcP'\) obtained by splitting \(S\) into \(\{a,b\}\) and \(S\setminus \{a,b\}\) is fair. 
     We now analyze for each pair of vertices \(u,v, u\neq v\) how the incurred \cccost changes. The cost does not change for every pair of vertices of which at most one vertex of \(u\) and \(v\) is in \(S\). Further, it does not change if either \(\{u,v\}=\{a,b\}\) or \(\{u,v\}\subseteq S\setminus \{a,b\}\). 
     There are at most
    $|S_A|-1+|S_B|-1=|S|-2$
      edges with one endpoint in \(\{a,b\}\) and the other in \(S\setminus \{a,b\}\). Each of them is cut in \(\mcP'\) but not in \(\mcP\), so they incur an extra cost of at most \(|S|-2\). However, due to the bipartite structure, there are \(|S_A|-1\) vertices in \(S\setminus\{a,b\}\) that have no edge to \(a\) and \(|S_B|-1\) vertices in \(S\setminus\{a,b\}\) that have no edge to \(b\). These \(|S|-2\) vertices incur a total cost of \(|S|-2\) in \(\mcP\) but no cost in \(\mcP'\). This makes up for any cut edge in \(\mcP\), so splitting the clustering never increases the cost. 

     If there is no \(a\in S_a\) and \(b\in S_b\) such that \(a\) and \(b\) have not the same color, then either \(S_A = \emptyset\) or \(S_B = \emptyset\). In both cases, there are no edges inside \(S\), so splitting the clustering in an arbitrary fair way never increases the cost.

     By iteratively splitting large clusters in any fair clustering, we hence eventually obtain a minimum-cost \fc such that all clusters consist of exactly two vertices.  

     Now, assume \(G\) is a forest and there would be a minimum-cost clustering \(\mcP\) with some cluster \(S\in \mcP\) such that \(|S|>2a\) for some \(a\in \N_{>2}\). Consider a new clustering \(\mcP'\) obtained by splitting \(S\) into \(\{u,v\}\) and \(S\setminus\{u,v\}\), where \(u\) and \(v\) are two arbitrary vertices of different color that have at most 1 edge towards another vertex in \(S\). There are always two such vertices due to the forest structure and because there are \(\frac{S}{2}\) vertices of each color.
     Then, \(\mcP'\) is still a fair clustering. 
     Note that the cost incurred by each edge and non-edge with at most one endpoint in \(S\) is the same in both clusterings.
     Let \(\psi\) denote the intra-cluster cost of \(\mcP\) in \(G[S]\). Regarding the edges and non-edges with both endpoints in \(S\), we know that  
      \begin{gather*}
      	\cost{G[S], \mcP} \ge \psi \ge \frac{2a(2a-1)}{2}-(2a-1) = 2a^2-3a+1
      \end{gather*}
      as the cluster consists of \(2a\) vertices and has at most \(2a-1\) edges due to the forest structure.
      In the worst case, \(\mcP'\) cuts \(2\) edges. However, we profit from the smaller cluster sizes. We have
      \begin{gather*}
      	\cost{G[S], \mcP'} \le 2+ \psi \le 2+1+\frac{2(a-1)(2(a-1)-1)}{2} - (2a-1 - 2)
        	= 2a^2-5a+6.
      \end{gather*} 
      Hence, \(\mcP\) costs at least \(2a-5\) more than \(\mcP'\), which is positive as \(a>2\). Thus, in every minimum-cost \fc all clusters are of size 4 or 2.
\end{proof}

We employ an analogous strategy if there is a different color ratio than \(1:1\) in the graph. However, then we have to split more than 2 vertices from a cluster. To ensure that the clustering cost does not increase, we have to argue that we can take these vertices in some balanced way from both sides of the bipartite graph.

\smallclustersbipartite*

\begin{proof}
     Due to the fairness constraint, each \fc consists only of clusters that are of size \(ad\), where \(a\in\N_{>0}\).
     We prove the statements by showing that a cluster of size at least \(2d\) can be split such that the cost does not increase and fairness remains.

     Let \(\mcP\) be a clustering and \(S\in\mcP\) be a cluster with \(|S|= ad\) for some \(a\ge 2\). Let \(S_A = S\cap A\) as well as \(S_B = S\cap B\) and w.l.o.g.\ \(|S_A|\ge|S_B|\).
     Our proof has three steps.
     
     \begin{itemize}
      \item First, we show that there is a fair \(\widetilde{S}\subseteq S\) such that \(|\widetilde{S}| = d\) and  \(|\widetilde{S}\cap A|\ge |\widetilde{S}\cap B|\).\vspace*{.25em}
     \item Then, we construct a fair set \(\widehat{S}\subseteq S\) by replacing vertices in \(\widetilde{S}\) with vertices in \(S_B\setminus \widetilde{S}\) such that still
     \(|\widehat{S}|=d, |\widehat{S}_A|\ge |\widehat{S}_B|\), with \(\widehat{S}_A = \widehat{S}\cap A\) and \(\widehat{S}_B=\widehat{S}\cap B\), and additionally \(|\widehat{S}_A|-|\hat{S}_B| \le |S_A|-|S_B|\).\vspace*{.25em}
     \item Last, we prove that splitting \(S\) into \(\widehat{S}\) and \(S\setminus \widehat{S}\) does not increase the clustering cost. 
   \end{itemize}
   
   We then observe that the resulting clustering is fair, so the lemma's statements hold because any fair clustering with a cluster of more than \(d\) vertices is transformed into a fair clustering with at most the same cost, and only clusters of size \(d\) by repeatedly splitting larger clusters.
   
   For the first step, assume there would be no such \(\widetilde{S}\subseteq S\), i.e., that we only could take \(s < \frac{d}{2}\) vertices from \(S_A\) without taking more than \(c_i\) vertices of each color \(i\in[k]\).
     Let \(s_i\) be the number of vertices of color \(i\) among these \(s\) vertices for all \(i\in[k]\).
     Then, if \(s_i = 0\) there is no vertex of color \(i\) in \(S_A\) as we could take the respective vertex into \(\widetilde{S}\), otherwise.
     Analogously, if \(s_i < c_i\), then there are no more then \(s_i\) vertices of color \(i\) in \(S_A\).
    If we take \(s_i = c_i\) vertices, then up to all of the \(ac_i=as_i\) vertices of that color are possibly in \(S_A\). Hence, 
    $|S_A|\le \sum_{i=1}^k as_i = as <  \frac{ad}{2}$.
    This contradicts \(S_A\ge S_B\) because \(|A|+|B|=ad\). Thus, there is a fair set \(\widetilde{S}\) of size \(d\) such that \(|\widetilde{S}\cap S_A|\ge |\widetilde{S}\cap S_B|\).

    Now, for the second step, we transform \(\widetilde{S}\) into \(\widehat{S}\).
    Note that, if \(|S_A\setminus \widetilde{S}|\ge |S_B\setminus \widetilde{S}|\) it suffices to set \(\widehat{S} =\widetilde{S}\). Otherwise, we replace some vertices from \(\widetilde{S}\cap S_A\) by vertices of the respective color from \(S_B\setminus \widetilde{S}\). We have to show that after this we still take at least as many vertices from \(S_A\) as from \(S_B\) and \(|S_A|-|\widehat{S}_A|\ge |S_B|-|\widehat{S}_B|\). Let
     \begin{equation*}
      \delta= |S_B\setminus \widetilde{S}|- |S_A\setminus \widetilde{S}| > 0.
     \end{equation*}
      Recall that \(|S_A|\ge |S_B|\), so 
      \(\delta \le |\widetilde{S}\cap A|-|\widetilde{S}\cap B|\).
     Then, we build \(\widehat{S}\) from \(\widetilde{S}\) by replacing \(\frac{\delta}{2}\le\frac{d}{2}\) vertices from \(\widetilde{S}\cap S_A\) with vertices of the respective color from \(S_B\setminus \widetilde{S}\). If there are such \(\frac{\delta}{2}\) vertices, we have \(|S_A\setminus \widehat{S}_A|=|S_B\setminus \widehat{S}_B|\) and \(|\widehat{S}_A|\ge |\widehat{S}_B|\). Consequently, \(\widehat{S}\) fulfills the requirements.  

     Assume there would be no such \(\frac{\delta}{2}\) vertices but that we could only replace \(s<\frac{\delta}{2}\) vertices. Let \(s_i\) be the number of vertices of color \(i\) among these vertices for all \(i\in[k]\). By a similar argumentation as above and because there are only \((a-1)c_i\) vertices of each color \(i\) in \(S\setminus \widehat{S}\), we have
     \begin{gather*}
      |S_B\setminus \widehat{S}| \le \sum_{i=1}^k (a-1)s_i = (a-1)s
     	<  \frac{(a-1)d}{2}.
     \end{gather*}
      This contradicts \(|S_B\setminus \widetilde{S}|> |S_A\setminus \widetilde{S}|\) as \(|(S_A\cup S_B)\setminus \widetilde{S}| = (a-1)d\). Hence, there are always enough vertices to create \(\widehat{S}\).

    For the last step, we show that splitting \(S\) into \(\widehat{S}\) and \(S\setminus\widehat{S}\) does not increase the cost by analyzing the change for each pair of vertices \(\{u,v\}\in \binom{V}{2}\).
     If not \(u\in S\) and \(v\in S\), the pair is not affected. Further, it does not change if either \(\{u,v\}\subseteq \widehat{S}\) or \(\{u,v\}\subseteq (S\setminus \widehat{S})\).     
     For the remaining pairs of vertices, there are at most 
     \begin{align*}
        |\widehat{S}_A|\cdot|S_B\setminus \widehat{S}_B|+|\widehat{S}_B|\cdot |S_A\setminus \widehat{S}_A|
        =|\widehat{S}_A|\cdot |S_B| + |\widehat{S}_B|\cdot |S_A| - 2\left(|\widehat{S}_A|\cdot |\widehat{S}_B|\right)
     \end{align*}
     edges that are cut when splitting \(S\) into \(\widehat{S}\) and \(S\setminus \widehat{S}\). 
     At the same time, there are 
     \begin{align*}
        |\widehat{S}_A|\cdot|S_A\setminus \widehat{S}_A|+|\widehat{S}_B|\cdot |S_B\setminus \widehat{S}_B|
        =|\widehat{S}_A|\cdot |S_A| + |\widehat{S}_B|\cdot |S_B| - |\widehat{S}_A|^2-|\widehat{S}_B|^2
     \end{align*} 
     pairs of vertices that are not connected and placed in separate clusters in \(\mcP'\) but not in \(\mcP\). Hence, we have  \(\mcP\) is more expansive than \(\mcP'\) by at least
     \begin{align*}
      \cost{\mcP}-\cost{\mcP'}
       &\ge |\widehat{S}_A|\cdot |S_A| + |\widehat{S}_B|\cdot |S_B| - |\widehat{S}_A|\cdot |S_B| - |\widehat{S}_B|\cdot |S_A|\\
       &\hphantom{{}\ge} - \left(|\widehat{S}_A|^2  - 2\left(|\widehat{S}_A|\cdot |\widehat{S}_B|\right) + |\widehat{S}_B|^2\right)\\
       &\ge \left(|\widehat{S}_A|-|\widehat{S}_B|\right)\cdot \left(|S_A| -|S_B|\right) - \left(|\widehat{S}_A| - |\widehat{S}_B|\right)^2.
     \end{align*}
     This is non-negative as \(|\widehat{S}_A|\ge |\widehat{S}_B|\) and \(|\widehat{S}_A|-|\widehat{S}_B| \le |S_A|-|S_B|\). Hence, splitting a cluster like this never increases the cost.       
    \end{proof}

Unlike in forests, however, the color ratio yields no bound on the maximum cluster size in minimum-cost \fcs on bipartite graphs but just states there is a minimum-cost fair clustering with bounded cluster size. Let \(G=(R\cup B, \{\{r,b\}\mid r\in R\wedge b\in B\})\) be a complete bipartite graph with \(|R|=|B|\) such that all vertices in \(R\) are red and all vertices in \(B\) are blue. Then, all \fcs in \(G\) have the same cost, including the one with a single cluster \(S=R\cup B\). This holds because of a similar argument as employed in the last part of \Cref{lem:smallClustersBipartiteOneOne} since every edge that is cut by a clustering is compensated for with exactly one pair of non-adjacent vertices that is then no longer in the same cluster.

\section{Hardness Results}
\label{sec:hardness}
This section provides \NP-hardness proofs for \fcc under various restrictions.

\subsection{Forests and Trees}
\label{subsec:hardness_fortest}

With the knowledge of the fixed sizes of clusters in a minimum-cost clustering, we are able to show that the problem is surprisingly hard, even when limited to certain instances of forests and trees. 

To prove the hardness of \fcc under various assumptions, we reduce from the strongly \NP-complete \thrPart problem \cite{Garey_Johnson_1979}.
\optPDef{\thrPart}
{\(n=3p\) with \(p\in\N\), positive integers \(a_1,a_2,\ldots,a_n\), and \(B\in\N\) such that \(\frac{B}{4}<a_i<\frac{B}{2}\) as well as \(\sum_{i=1}^n a_i = pB\).}
{Decide if there is a partition of the numbers \(a_i\) into triples such that the sum of each triple is \(B\).}

Our first reduction yields hardness for many forms of forests.

\begin{theorem}
\label{thm:forest_hard}
    \emph{\fcc} on forests with two colors in a ratio of \(1:c\) is \emph{\NP}-hard. It remains \emph{\NP}-hard when arbitrarily restricting the shape of the trees in the forest as long as for every \(a\in\N\) it is possible to form a tree with \(a\) vertices.
\end{theorem}

\begin{proof}
    We reduce from \thrPart. For every \(a_i\), we construct an arbitrarily shaped tree of \(a_i\) red vertices. Further, we let there be \(p\) isolated blue vertices. Note that the ratio between blue and red vertices is \(1:B\). We now show that there is a \fc \(\mcP\) such that  
    \begin{equation*}
        \cost{\mcP} = p\cdot\frac{B(B+1)}{2}-p(B-3)
    \end{equation*}
     if and only if the given instance is a yes-instance for \thrPart. 

    If we have a yes-instance of \thrPart, then there is a partition of the set of trees into \(p\) clusters of size \(B\). By assigning the blue vertices arbitrarily to one unique cluster each, we hence obtain a fair partition. As there are no edges between the clusters and each cluster consists of \(B+1\) vertices and \(B-3\) edges, this partition has a cost of \(p\cdot\frac{B(B+1)}{2}-p(B-3)\).

    For the other direction, assume there is a \fc of cost \(\frac{B(B+1)}{2}-p(B-3)\).
    By \autoref{lem:smallClustersForest}, each of the clusters consists of exactly one blue and \(B\) red vertices. Each cluster requires \(\frac{B(B+1)}{2}\) edges, but the graph has only \(p(B-3)\) edges. The intra-cluster cost alone is hence at least \(p\cdot\frac{B(B+1)}{2}-p(B-3p)\). This means that the inter-cluster cost is 0, i.e., the partition does not cut any edges inside the trees. Since all trees are of size greater than \(\frac{B}{4}\) and less than \(\frac{B}{2}\), this implies that each cluster consists of exactly one blue vertex and exactly three uncut trees with a total of \(B\) vertices. This way, such a clustering gives a solution to \thrPart, so our instance is a yes-instance.  

    As the construction of the graph only takes polynomial time in the instance size, this implies our hardness result.  
\end{proof}

Note that the hardness holds in particular for forests of paths, i.e., for forests with maximum degree 2.

With the next theorem, we adjust the proof of \autoref{thm:forest_hard} to show that the hardness remains if the graph is connected.

\begin{theorem}
\label{thm:tree_hard}
    \emph{\fcc} on trees with diameter 4 and two colors in a ratio of \(1:c\) is \emph{\NP}-hard.
\end{theorem}

\begin{proof}
    We reduce from \thrPart. For every \(a_i\), we construct a star of \(a_i\) red vertices. Further, we let there be a star of \(p\) blue vertices. We obtain a tree of diameter 4 by connecting the center \(v\) of the blue star to all the centers of the red stars. The construction is depicted in \autoref{fig:treeOctopus}.
    \begin{figure}
        \centering
        \includegraphics[height=.3\textwidth]{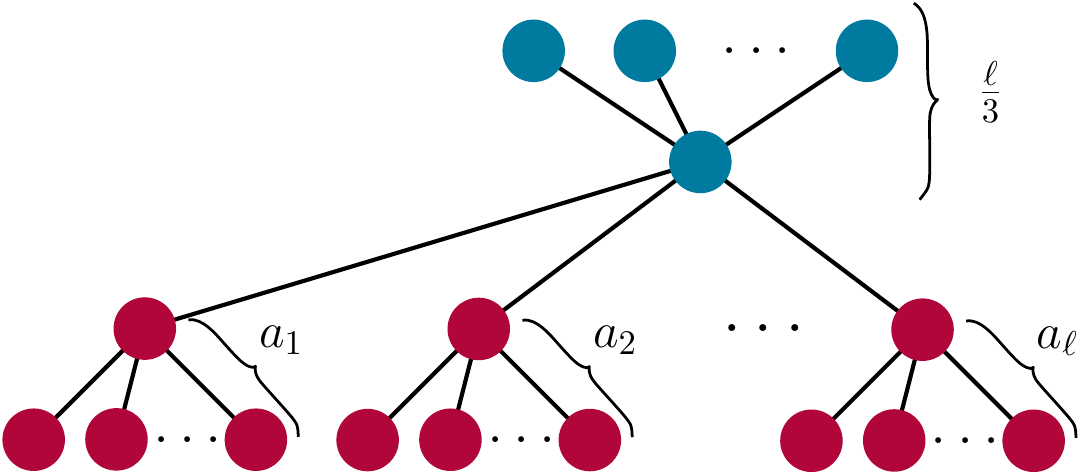}
        \caption{The tree with diameter 4 in the reduction from \thrPart to \fcc.
        The notation follows that of \autoref{thm:tree_hard}.}
    \label{fig:treeOctopus}
    \end{figure}
    Note that the ratio between blue and red vertices is \(1:B\). 
    We now show that there is a \fc \(\mcP\) such that 
    \begin{align*}
        \cost{\mcP} \le \frac{pB^2-pB}{2}+7p-7
    \end{align*}
     if and only if the given instance is a yes-instance for \thrPart. 

    If we have a yes-instance of \thrPart, then there is a partition of the set of stars into \(p\) clusters of size \(B\), each consisting of three stars. By assigning the blue vertices arbitrarily to one unique cluster each, we hence obtain a fair partition.
    We first compute the inter-cluster cost \(\chi\). We call an edge \emph{blue} or \emph{red} if it connects two blue or red vertices, respectively. We call an edge \emph{blue-red} if it connects a blue and a red vertex. All \(p-1\) blue edges are cut. Further, all edges between \(v\) (the center of the blue star) and red vertices are cut except for the three stars to which \(v\) is assigned. This causes \(3p-3\) more cuts, so the inter-cluster cost is $\chi = 4p-4$.
    Each cluster consists of \(B+1\) vertices and \(B-3\) edges, except for the one containing \(v\) which has \(B\) edges. The intra-cluster cost is hence 
    \begin{gather*}
        \psi = p \left(\frac{B(B+1)}{2}-B+3\right)-3 = \frac{pB^2-pB}{2}+3p-3.
    \end{gather*}
     Combining the intra- and inter-cluster costs yields the desired cost of 
    \begin{gather*}
        \cost{\mcP} = \chi + \psi = \frac{pB^2-pB}{2}+7p-7.
    \end{gather*}

    For the other direction, assume there is a \fc of cost at most \(\frac{pB^2-pB}{2}+7p-7\).
    As there are \(p(B+1)\) vertices, \autoref{lem:smallClustersForest} gives that there are exactly \(p\) clusters, each consisting of exactly one blue and \(B\) red vertices. Let \(a\) denote the number of red center vertices in the cluster of \(v\). We show that \(a=3\). To this end, let \(\chi_r\) denote the number of cut red edges. We additionally cut \(p-1\) blue and \(3p-a\) blue-red edges. The inter-cluster cost of the clustering hence is $\chi = \chi_r + 4p -a -1$.
    Regarding the intra-cluster cost, there are no missing blue edges and as \(v\) is the only blue vertex with blue-red edges, there are \((p-1)B+B-a=pB-a\)
     missing blue-red edges. Last, we require \(p\cdot \frac{B (B-1)}{2}\) red edges, but the graph has only \(pB-3p\) red edges and \(\chi_r\) of them are cut. Hence, there are at least 
    \(p\cdot \frac{B (B-1)}{2}  - pB + 3p + \chi_r\)
     missing red edges, resulting in a total intra-cluster cost of 
     $\psi \ge p\cdot \frac{B (B-1)}{2}+3p+\chi_r-a$.
    This results in a total cost of  
    \begin{gather*}
        \cost{\mcP} = \chi + \psi \ge \frac{pB^2-pB}{2}+7p+2\chi_r-2a-1.
    \end{gather*}
    As we assumed \(\cost{\mcP} \le \frac{pB^2-pB}{2}+7p-7\), we have
    $2\chi_r-2a+6 \le 0$,
    which implies \(a\ge 3\) since \(\chi_r\ge 0\). 
    Additionally, $\chi_r \ge \frac{aB}{4}-(B-a)$,
    because there are at least \(\frac{B}{4}\) red vertices connected to each of the \(a\) chosen red centers but only a total of \(B-a\) of them can be placed in their center's cluster. Thus, we have
	$\frac{aB}{2}-2B+6 = \frac{(a-4)B}{2}+6\le 0$,
    implying \(a < 4\) and proving our claim of \(a=3\). Further, as \(a=3\), we obtain \(\chi_r \le 0\), meaning that no red edges are cut, so each red star is completely contained in a cluster. Given that every red star is of size at least \(\frac{B}{4}\) and at most \(\frac{B}{2}\), this means each cluster consists of exactly three complete red stars with a total number of \(B\) red vertices each and hence yields a solution to the \thrPart instance.

    As the construction of the graph only takes polynomial time in the instance size and the constructed tree is of diameter 4, this implies our hardness result.  
\end{proof}

The proofs of \Cref{thm:forest_hard,thm:tree_hard} follow the same idea as the hardness proof of \cite[Theorem~2]{Feldmann_Foschini_2015}, which also reduces from \thrPart to prove a hardness result on the \kbalpart problem.
There, the task is to partition the vertices of an uncolored graph into \(k\) clusters of equal size \cite{Feldmann_Foschini_2015}. 

\optPDef{\kbalpart}
{Graph $G = (V, E)$, \(k\in[n]\).}
{Find a partition $\mcP$ of \(V\) that minimizes \(|\{\{u,v\}\in E\mid \mcP[u]\neq\mcP[v]\}|\) under the constraint that \(|\mcP|=k\) and \(|S|\le \lceil\frac{n}{k}\rceil\) for all \(S\in\mcP\) .}

\kbalpart is related to \fcc on forests in the sense that the clustering has to partition the forest into clusters of equal sizes by \Cref{lem:smallClustersForest,lem:smallClustersBipartiteOneOne}.
Hence, on forests we can regard \fcc as the fair variant of \kbalpart.
By \cite[Theorem~8]{Feldmann_Foschini_2015}, \kbalpart is \NP-hard on trees of degree 5. In their proof, Feldmann and Foschini~\cite{Feldmann_Foschini_2015} reduce from \thrPart. We slightly adapt their construction to transfer the result to \fcc.

\begin{theorem}
\label{thm:treeDeg5NPhard}
    \emph{\fcc} on trees of degree at most 5 with two colors in a ratio of \(1:c\) is \emph{\NP}-hard.
\end{theorem}

\begin{proof}
    We reduce from \thrPart, which remains strongly \NP-hard when limited to instances where \(B\) is a multiple of 4 since for every instance we can create an equivalent instance by multiplying all integers by 4. Hence, assume a \thrPart instance such that \(B\) is a multiple of \(4\). We construct a graph for \fcc by representing each \(a_i\) for \(i\in [n]\) by a gadget \(T_i\). Each gadget has a center vertex that is connected to the end of five paths: one path of length \(a_i\), three paths of length \(\frac{B}{4}\), and one path of length \(\frac{B}{4}-1\). Then, for \(i\in[n-1]\), we connect the dangling ends of the paths of length \(\frac{B}{4}-1\) in the gadgets \(T_i\) and \(T_{i+1}\) by an edge. So far, the construction is similar to the one by Feldmann and Foschini~\cite{Feldmann_Foschini_2015}. We color all vertices added so far in red. Then, we add a path of  \(\frac{4n}{3}\) blue vertices and connect it by an edge to an arbitrary vertex of degree 1. The resulting graph is depicted in \autoref{fig:octopusDeg5}.

    \begin{figure}
        \includegraphics[width=\textwidth]{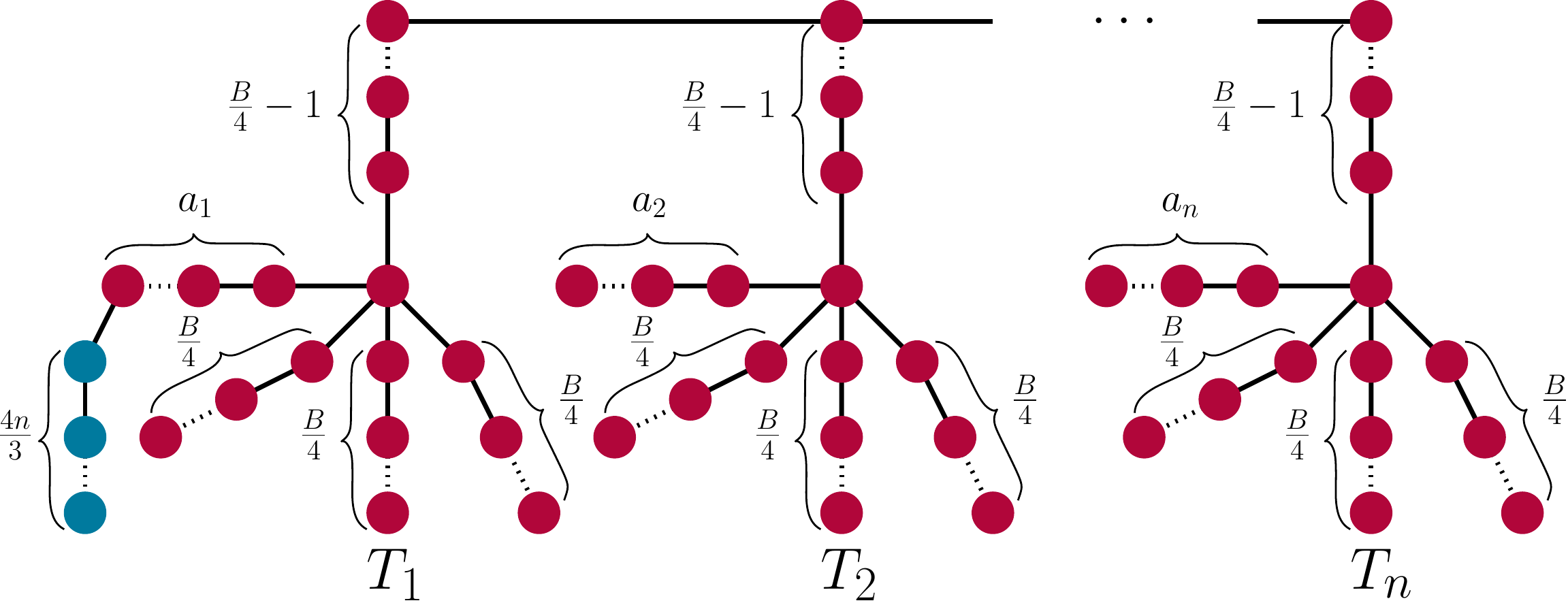}
        \caption{Tree with maximum degree 5 in the reduction from \thrPart to \fcc (\autoref{thm:treeDeg5NPhard}).}
    \label{fig:octopusDeg5}
    \end{figure}
    
    Note that the construction takes polynomial time and we obtain a graph of degree 5. We now prove that it has a \fc \(\mcP\) such that 
    \begin{equation*}
        \cost{\mcP}\le \frac{(B-2)n}{2}+\frac{20n}{3}-3
    \end{equation*}
    if and only if the given instance is a yes-instance for \thrPart.

    Assume we have a yes-instance for \thrPart. We cut the edges connecting the different gadgets as well as the edges connecting the \(a_i\)-paths to the center of the stars. Then, we have \(n\) components of size \(B\) and 1 component of size \(a_i\) for each \(i\in [n]\). The latter ones can be merged into \(p=\frac{n}{3}\) clusters of size \(B\) without further cuts. Next, we cut all edges between the blue vertices and assign one blue vertex to each cluster. Thereby, note that the blue vertex that is already connected to a red cluster should be assigned to this cluster. This way, we obtain a \fc with inter-cluster cost $\chi = n-1+n+\frac{4n}{3}-1 =\frac{10n}{3}-2$,
    which, by \autoref{lem:costByCuts}, gives
	$\cost{\mcP} = \frac{(B-2)n}{2}+\frac{20n}{3}-3$.
	
    For the other direction, let there be a minimum-cost \fc \(\mcP\) of cost at most \(\frac{(B-2)n}{2}+\frac{20n}{3}-3\).    
    As \(\sum_{i=1}^n a_i = \frac{nB}{3}\), the graph consists of \(\frac{4n}{3}\cdot B\) red and \(\frac{4n}{3}\) blue vertices. By \autoref{lem:smallClustersForest}, \(\mcP\) hence consists of \(\frac{4n}{3}\) clusters, each consisting of one blue vertex and \(B\) red vertices. Thus, \(\mcP\) has to cut the \(\frac{4n}{3}-1\) edges on the blue path. Also, \(\mcP\) has to partition the red vertices into sets of size \(B\). By \cite[Lemma~9]{Feldmann_Foschini_2015} this requires at least \(2n-1\) cuts. This bounds the inter-cluster cost by $\chi \ge 2n-1+\frac{4n}{3}-1=\frac{10n}{3}-2$,
	leading to a \cccost of \(\frac{(B-2)n}{2}+\frac{20n}{3}-3\) as seen above, so we know that no more edges are cut. Further, the unique minimum-sized set of edges that upon removal leaves no red components of size larger than \(B\) is the set of the \(n-1\) edges connecting the gadgets and the \(n\) edges connecting the \(a_i\) paths to the center vertices \cite[Lemma~9]{Feldmann_Foschini_2015}. Hence, \(\mcP\) has to cut exactly these edges. As no other edges are cut, the \(a_i\) paths can be combined to clusters of size \(B\) without further cuts, so the given instance has to be a yes-instance for \thrPart.
\end{proof}

\subsection{Paths}
\label{subsec:hardness_paths}

\autoref{thm:forest_hard} yields that \fcc is \NP-hard even in a forest of paths. The problem when limited to instances of a single connected path is closely related to the \textsc{Necklace Splitting} problem \cite{Alon_1987,Alon_West_1986}.

\optPDef{\textsc{Discrete Necklace Splitting}}
{Opened necklace \(N\), represented by a path of \(n\cdot k\) beads, each in one of \(t\) colors such that for each color \(i\) there are \(a_i\cdot k\) beads of that color for some \(a_i\in\N\).}
{Cut the necklace such that the resulting intervals can be partitioned into \(k\) collections, each containing the same number of beads of each color.}
The only difference to \fcc on paths, other than the naming, is that the number of clusters \(k\) is explicitly given. From \Cref{lem:smallClustersForest,lem:smallClustersBipartiteOneOne} we are implicitly given this value also for \fcc, though. 
However, Alon and West~\cite{Alon_West_1986} do not constructively minimize the number of cuts required for a fair partition but non-constructively prove that there is always a partition of at most \((k-1)\cdot t\) cuts, if there are \(t\) colors and the partition is required to consist of exactly \(k\) sets with the same amount of vertices of each color. Thus, it does not directly help us when solving the optimization problem.

Moreover, \fcc on paths is related to the 1-regular 2-colored variant of the \textsc{Paint Shop Problem for Words} (\textsc{PPW}). For \textsc{PPW}, a word is given as well as a set of colors, and for each symbol and color a requirement of how many such symbols should be colored accordingly. The task is to find a coloring that fulfills all requirements and minimizes the number of color changes between adjacent letters \cite{Epping_Hochstattler_Oertel_2004}.

\optPDef{\textsc{Paint Shop Problem for Words} (\textsc{PPW})}
{Word \(w=w_1,w_2,\ldots,w_n\in \Sigma^*\), number of colors \(k\in\N_{>0}\), and requirement function \(r: \Sigma\times [k]\rightarrow \N\) such that for each symbol \(s\) used in \(w\) with \(w[s]\) occurrences we have \(\sum_{i=1}^k r(s,i)= w[s]\).}
{Find an assignment function \(f\colon [n] \rightarrow [k]\) of colors to the letters in \(w\) such that for each symbol \(s\in\Sigma\) and color \(i\in[k]\) the coloring fulfills the requirement function, i.e., \(|\{j\in[n]\mid w_j=s \wedge f(j)=i\}| = r(s,i)\). The assignment \(f\) should minimze the number of color changes \(|\{j\in [n-1]\mid f(j)\neq f(j+1)\}|\).}
Let for example \(w=aabab\) and \(r(a, 1) = 2, r(a,2)=r(b, 1)=r(b,2)=1\). Then, the assignment \(f\) with \(f(1)=f(2)=f(3)=1\) and \(f(4)=f(5)=2\) fulfills the requirement and has 1 color change.

\textsc{PPW} instances with a word containing every symbol exactly twice and two \textsc{PPW}-colors, each requiring one of each symbol, are called \emph{1-regular 2-colored} and are shown to be \NP-hard and even \APX-hard \cite{Bonsma_Epping_Hochstattler_2006}. With this, we prove \NP-hardness of \fcc even on paths.

\begin{theorem}
\label{thm:pathsNPhard}
    \emph{\fcc} on paths is \emph{\NP-hard}, even when limited to instances with exactly 2 vertices of each color. 
 \end{theorem}
 
 \begin{proof}
    We reduce from 1-regular 2-colored \textsc{PPW}. Let \(w=s_1s_2,\ldots,s_\ell\). We represent the \(\frac{\ell}{2}\) different symbols by \(\frac{\ell}{2}\) colors and construct a path of length \(\ell\), where each type of symbol is represented by a unique color. 
    By \autoref{lem:smallClustersForest}, any optimum \fcc solution partitions the paths into two clusters, each containing every color exactly once, while minimizing the number of cuts (the inter-cluster cost) by \autoref{lem:costByCuts}. As this is exactly equivalent to assigning the letters in the word to one of two colors and minimizing the number of color changes, we obtain our hardness result.
 \end{proof}
 
 \APX-hardness however is not transferred since though there is a relationship between the number of cuts (the inter-cluster cost) and the \cccost, the two measures are not identical. In fact, as \fcc has a PTAS on forests by \autoref{thm:ptas}, \APX-hardness on paths would imply $\P = \NP$.

On a side note, observe that for every \fcc instance on paths we can construct an equivalent \textsc{PPW} instance (though not all of them are 1-regular 2-colored) by representing symbols by colors and \textsc{PPW}-colors by clusters.

We note that it may be possible to efficiently solve \fcc on paths if there are e.g. only two colors. There is an \NP-hardness result on \textsc{PPW} with just two letters in \cite{Epping_Hochstattler_Oertel_2004}, but a reduction from these instances is not as easy as above since its requirements imply an unfair clustering.

\subsection{Beyond Trees}
\label{subsec:hardness_general}

By \autoref{thm:tree_hard}, \fcc is \NP-hard even on trees with diameter 4. Here, we show that if we allow the graph to contain circles, the problem is already \NP-hard for diameter 2. Also, this nicely contrasts that \fcc is solved on trees of diameter 2 in linear time, as we will see in \autoref{subsec:easyInstances}.  

\begin{theorem}
\label{thm:generaGraphsDiam2Hardness}
    \emph{\fcc} on graphs of diameter 2 with two colors in a ratio of \(1:1\) is \emph{\NP}-hard.
 \end{theorem}
 \begin{proof}
    Cluster Editing, which is an alternative formulation of \cc, is \NP-hard on graphs of diameter 2 \cite{Bastos_Ochi_Protti_Subramanian_Martins_Pinheiro_2016}. 
    Further, Ahmadi et al.\ \cite{Ahmadi_Galhotra_Saha_Schwartz_2020} give a reduction from \cc to \fcc with a color ratio of \(1:1\). They show that one can solve \cc on a graph \(G=(V,E)\) by solving \fcc on the graph \(G'=(V\cup V', E\cup E'\cup \widetilde{E})\) that mirrors \(G\). The vertices in \(V\) are colored blue and the vertices in \(V'\) are colored red. Formally, \(V'=\{u'\mid u\in V\}\) and \(E'=\{\{u',v'\} \mid \{u,v\}\in E\}\). Further, \(\widetilde{E}\) connects every vertex with its mirrored vertex as well as the mirrors of adjacent vertices, i.e., \(\widetilde{E}= \{\{u,u'\}\mid u\in V\} \cup \{\{u,v'\} \mid u\in V\wedge v'\in V' \wedge \{u,v\}\in E\}\), see \autoref{fig:mirrorNPGeneral}.
    \begin{figure}
        \centering
        \includegraphics[height=.3\textwidth]{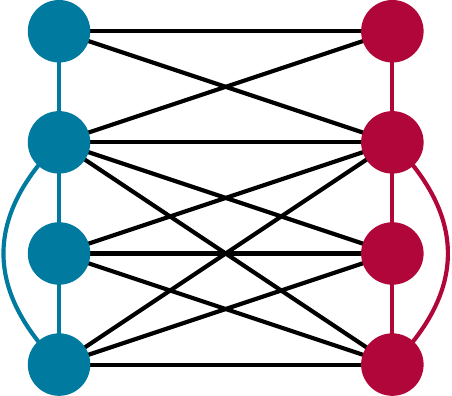}
    \caption{Graph as constructed by Ahmadi et al.\ \cite{Ahmadi_Galhotra_Saha_Schwartz_2020} for the reduction from \cc to \fcc. The blue vertices and edges correspond to the original graph \(G=(V,E)\), red vertices and edges to its mirror, i.e., \(V'\) and \(E'\), and black edges to \(\widetilde{E}\).}\label{fig:mirrorNPGeneral}
    \end{figure} 

    Observe that if \(G\) has diameter 2 then \(G'\) also has diameter 2 as follows. As every pair of vertices \(\{u,v\}\in\binom{V}{2}\) is of maximum distance 2 and the vertices as well as the edges of \(G\) are mirrored, every pair of vertices \(\{u',v'\}\in\binom{V'}{2}\) is of maximum distance 2. Further, every vertex and its mirrored vertex have a distance of 1. For every pair of vertices \(u\in V, v'\in V'\) we distinguish two cases. If \(\{u,v\}\in E\), then \(\{u,v'\}\in \widetilde{E}\), so the distance is 1. Otherwise, as the distance between \(u\) and \(v\) is at most 2 in \(G\), there is \(w\in V\) such that \(\{u,w\}\in E\) and \(\{v,w\}\in E\). Thus, \(\{u,w'\}\in \widetilde{E}\) and \(\{w',v'\}\in E'\), so the distance of \(u\) and \(v'\) is at most 2.
    
    As \cc on graphs with diameter 2 is \NP-hard and the reduction by Ahmadi et al.\ \cite{Ahmadi_Galhotra_Saha_Schwartz_2020} constructs a graph of diameter 2 if the input graph is of diameter 2, we have proven the statement.
 \end{proof}

Further, we show that on general graphs \fcc is \NP-hard, even if the colors of the vertices allow for no more than 2 clusters in any fair clustering. This contrasts our algorithm in \autoref{subsec:fewClustersAlgo} solving \fcc on forests in polynomial time if the maximum number of clusters is constant. To this end, we reduce from the \NP-hard \textsc{Bisection} problem \cite{Garey_Johnson_1979}, which is the \(k=2\) case of \kbalpart.

\optPDef{\textsc{Bisection}}
{Graph \(G=(V,E)\).}
{Find a partition \(\mcP=\{A,B\}\) of \(V\) that minimizes \(|\{\{u,v\}\in E\mid u\in A \wedge v\in B\}|\) under the constraint that \(|A|=|B|\).}  

 \begin{theorem}
 \label{thm:hardness_large1_C}
    \emph{\fcc} on graphs with two colors in a ratio of \(1:c\) is \emph{\NP}-hard, even if \(c=\frac{n}{2}-1\) and the graph is connected.
\end{theorem}
\begin{proof}
    We reduce from \textsc{Bisection}. Let \(G=(V,E)\) be a \textsc{Bisection} instance and assume it has an even number of vertices (otherwise it is a trivial no-instance). The idea is to color all of the vertices in \(V\) red and add two cliques, each consisting of one blue and \(|V|\) red vertices to enforce that a minimum-cost \fcc consists of exactly two clusters and thereby partitions the vertices of the original graph in a minimum-cost bisection. The color ratio is \(2:3|V|\) which equals \(1:\frac{|V'|}{2}-1\) with \(V'\) being the set of the newly constructed graph. 
    We have to rule out the possibility that a minimum-cost \fcc is just one cluster containing the whole graph. We do this by connecting the new blue vertices \(v_1,v_2\) to only one arbitrary red vertex \(v\in V\). We illustrate the scheme in \autoref{fig:bisectionRed}. 
    \begin{figure}
        \centering
        \includegraphics[height=.25\textwidth]{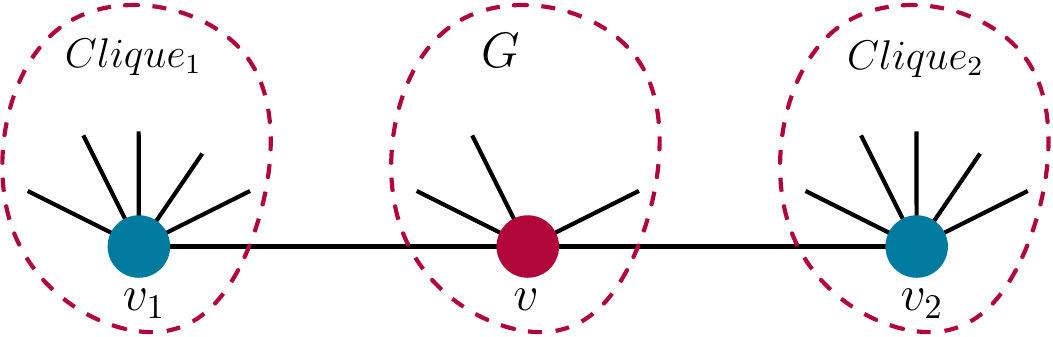}
    \caption{Graph constructed for the reduction from \textsc{Bisection} to a \fcc instance with just 2 large clusters. The middle part corresponds to the input graph \(G\) and is colored red. \(Clique_1\) and \(Clique_2\) are both cliques of \(|V|\) red vertices and one blue vertex each.}\label{fig:bisectionRed}
    \end{figure} 
    We first argue that every clustering with two clusters is cheaper than placing all vertices in the same cluster. Let \(n=|V|\) as well as \(m=|E|\). Let \(\mcP\) be a clustering that places all vertices in a single cluster. Then, 
    \begin{gather*}
        \cost{\mcP} = \frac{(3n+2)(3n+1)}{2} - \left(m+2+2\cdot\frac{n(n+1)}{2}\right)
        	= \frac{7n^2}{2} +\frac{7n}{2}-m -1,     
    \end{gather*}
    as the cluster is of size \(3n+2\), there is a total of \(m+2\) plus the edges of the cliques, and no edge is cut. 
    Now assume we have a clustering \(\mcP'\) with an inter-cluster cost of \(\chi'\) that puts each clique in a different cluster. Then, 
    \begin{align*}
        \cost{\mcP'} &= \chi'+ 2\cdot\frac{(\frac{3n}{2}+1)(\frac{3n}{2})}{2} - \left(m-\chi' + \frac{n(n+1)}{2}\right)\\
        &= \frac{7n^2}{4} +n-m + 2\chi'
         \le \frac{9n^2}{4} +n-m+2,     
    \end{align*}
    since there are at most \(\frac{n}{2}\cdot \frac{n}{2}\) inter-cluster edges between vertices of \(V\) and one inter-cluster edge from \(v\) to either \(v_1\) or \(v_2\), so \(\chi\le \frac{n^2}{4}+1\). 
    Placing all vertices in the same cluster is hence more expensive by 
    \begin{gather*}
        \cost{\mcP}-\cost{\mcP'} \ge \frac{7n^2}{2} 
        	+\frac{7n}{2}-m -1 - \left(\frac{9n^2}{4} +n-m+2\right)
        	= \frac{5n^2}{4}+\frac{5n}{2}-3
    \end{gather*}
    than any clustering with two clusters. This is positive for \(n\ge 2\). Thus, \fcc will always return at least two clusters. Also, due to the fairness constraint and there being only two blue vertices, it creates exactly two clusters. 

    Further, it does not cut vertices from one of the two cliques for the following reason. As the clusters are of fixed size, by \autoref{lem:costByCuts} we can focus on the inter-cluster cost to argue that a minimum-cost \fcc only cuts edges in \(E\). 
    First, note that it is never optimal to cut vertices from both cliques as just cutting the difference from one clique cuts fewer edges. This also implies that at most \(\frac{n}{2}\) red vertices are cut from the clique as otherwise, the other cluster would have more than the required \(\frac{3n}{2}\) red vertices. So, assume \(0<a\le \frac{n}{2}\) red vertices are cut from one clique. Any such solution has an inter-cluster cost of \(a\cdot(n+1-a)+ \chi_E\), where \(\chi_E\) is the number of edges in \(E\) that are cut to split \(V\) into two clusters of size \(\frac{n}{2}+a\) and \(\frac{n}{2}-a\) as required to make a fair partition. We note that by not cutting the cliques and instead cutting off \(a\) vertices from the cluster of size \(\frac{n}{2}+a\), we obtain at most \(a\cdot \frac{n}{2} + \chi_E\) cuts. As \(\frac{n}{2}<n+1-a\), this implies that no optimal solution cuts the cliques.
    Hence, each optimal solution partitions the \(V\) in a minimum-cost bisection. 
    
    Thus, by solving \fcc on the constructed graph we can solve \textsc{Bisection} in \(G\). As further, the constructed graph is of polynomial size in \(|V|\), we obtain our hardness result.
\end{proof}

\section{Algorithms}
\label{sec:algorithms}
The results from \autoref{sec:hardness} make it unlikely that there is a general polynomial time algorithm solving \fcc on trees and forests. However, we are able to give efficient algorithms for certain classes of instances.

\subsection{Simple Cases}
\label{subsec:easyInstances}

First, we observe that \fcc on bipartite graphs is equivalent to the problem of computing a maximum bipartite matching if there are just two colors that occur equally often. This is due to there being a minimum-cost \fc such that each cluster is of size 2.

\begin{theorem}
\label{thm:bipartiteOneOneBipMatching}
    Computing a minimum-cost \fc with two colors in a ratio of $1:1$ 
    is equivalent to the maximum bipartite matching problem under linear-time reductions,
    provided that the input graph has a minimum-cost \fc in which each cluster has cardinality at most $2$.
\end{theorem}

\begin{proof}
    Let the colors be red and blue.
    By assumption, there is an optimum clustering for which all clusters are of size at most 2. Due to the fairness constraint, each such cluster consists of exactly 1 red and 1 blue vertex. By \autoref{lem:costByCuts}, the lowest cost is achieved by the lowest inter-cluster cost, i.e., when the number of clusters where there is an edge between the two vertices is maximized. This is exactly the matching problem on the bipartite graph \(G'=(R\cup B, E')\), with \(R\) and \(B\) being the red and blue vertices, respectively, and \(E'=\{\{u,v\}\in E\mid u\in R\wedge v\in B\}\).
    After computing an optimum matching, each edge of the matching defines a cluster and unmatched vertices are packed into fair clusters arbitrarily.

    For the other direction, if we are given an instance \(G'=(R\cup B, E')\) for bipartite matching, we color all the vertices in \(R\) red and the vertices in \(B\) blue. Then, a minimum-cost \fc is a partition that maximizes the number of edges in each cluster as argued above. As each vertex is part of exactly one cluster and all clusters consist of one vertex in \(R\) and one vertex in \(B\), this corresponds to a maximum bipartite matching in \(G'\). 
\end{proof}

By \autoref{lem:smallClustersBipartiteOneOne}, the condition of \autoref{thm:bipartiteOneOneBipMatching}
is met by all bipartite graphs.
The recent maxflow breakthrough~\cite{Chen22MaxFlowAlmostLinear} also gives an $m^{1+o(1)}$-time algorithm to compute bipartite matchings,
this then transfers also to \fcc with color ratio \(1:1\).
For \fcc on forests, we can do better as the reduction in \autoref{thm:bipartiteOneOneBipMatching}
again results in a forest, for which bipartite matching can be solved in linear time by standard techniques.
We present the algorithm here for completeness.

\begin{theorem}
\label{thm:fcc_on_forests}
    \emph{\fcc} on forests with a color ratio $1:1$ can be solved in time \(\mathcal{O}(n)\).
\end{theorem}

\begin{proof}
	We apply \autoref{thm:bipartiteOneOneBipMatching} to receive a sub-forest of the input for which we have to compute a maximum matching.
    We do so independently for each of the trees by running the following dynamic program.
    We visit all vertices, but each one only after we have already visited all its children (for example by employing topological sorting).
    For each vertex \(v\), we compute the maximum matching \(M_v\) in the subtree rooted at \(v\) as well as the maximum matching \(M_v'\) in the subtree rooted at \(v\) assuming \(v\) is not matched.
    We directly get that \(M_v'\) is simply the union of the matchings \(M_u\) for each child \(u\) of \(v\).
    Further, either \(M_v = M_v'\) or in \(M_v\) there is an edge between \(v\) and some child \(u\).
    In the latter case, \(M_v\) is the union of \(\{u,v\}, M_u',\) and the union of all \(M_w\) for all children \(w\neq u\).
    Trying out all possible choices of \(u\) and comparing them among another and to \(M_v'\) yields \(M_v\).
    In the end, the maximum matching in the tree with root \(r\) is \(M_r\).
    
    Each vertex is visited once. 
    If the matchings are not naively merged during the process but only their respective sizes are tracked and the maximum matching is retrieved after the dynamic program by using a back-tracking approach, the time complexity per vertex is linear in the number of its children. 
    Thus, the dynamic program runs in time in \(\mathcal{O}(n)\).
\end{proof}

Next, recall that \autoref{thm:tree_hard} states that \fcc on trees with a diameter of at least 4 is \NP-hard. With the next theorem, we show that we can efficiently solve \fcc on trees with a diameter of at most 3, so our threshold of 4 is tight unless $\P = \NP$.

\begin{theorem}
\label{thm:treeDiam3Linear}
    \emph{\fcc} on trees with a diameter of at most 3 can be solved in time $O(n)$.
\end{theorem}

\begin{proof}
    Diameters of 0 or 1 are trivial and the case of two colors in a ratio of \(1:1\) is handled by \autoref{thm:bipartiteOneOneBipMatching}. So, assume \(d>2\) to be the minimum size of a fair cluster. A diameter of two implies that the tree is a star. In a star, the inter-cluster cost equals the number of vertices that are not placed in the same cluster as the center vertex. By \autoref{lem:smallClustersForest}, every clustering of minimum cost has minimum-sized clusters. As in a star, all these clusterings incur the same inter-cluster cost of \(n-d+1\) they all have the same \cccost by \autoref{lem:costByCuts}. Hence, outputting any \fc with minimum-sized clusters solves the problem. Such a clustering can be computed in time in $\bigO(n)$.

    If we have a tree of diameter 3, it consists of two adjacent vertices \(u,v\) such that every vertex \(w\in V\setminus\{u,v\}\) is connected to either \(u\) or \(v\) and no other vertex, see \autoref{fig:diameter3Tree}. 
    \begin{figure}
        \centering
        \includegraphics[height=.2\textwidth]{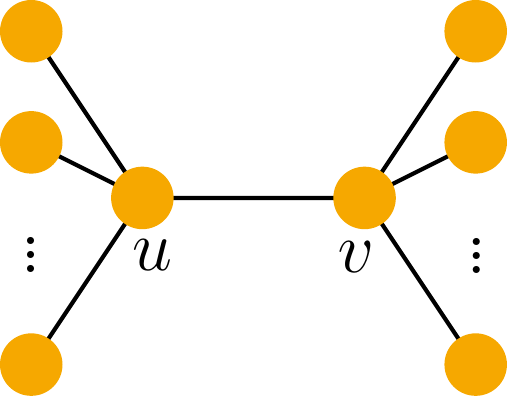}
    \caption{Shape of every tree with diameter 3.}
    \label{fig:diameter3Tree}
    \end{figure} 
    This is due to every graph of diameter 3 having a path of four vertices. Let the two in the middle be \(u\) and \(v\). The path has to be an induced path or the graph would not be a tree. We can attach other vertices to \(u\) and \(v\) without changing the diameter but as soon as we attach a vertex elsewhere, the diameter increases. Further, there are no edges between vertices in \(V\setminus\{u,v\}\) as the graph would not be circle-free.
    
    For the clustering, there are now two possibilities, which we try out separately. Either \(u\) and \(v\) are placed in the same cluster or not. In both cases, \autoref{lem:smallClustersForest} gives that all clusters are of minimal size \(d\). If \(u\) and \(v\) are in the same cluster, all clusterings of fair minimum sized clusters incur an inter-cluster cost of \(n-d+2\) as all but \(d-2\) vertices have to be cut from \(u\) and \(v\). In $\bigO(n)$, we greedily construct such a clustering \(\mcP_1\). If we place \(u\) and \(v\) in separate clusters, the minimum inter-cluster is achieved by placing as many of their respective neighbors in their respective clusters as possible. After that, all remaining vertices are isolated and are used to make these two clusters fair and if required form more fair clusters. Such a clustering \(\mcP_2\) is also computed in $\bigO(n)$.
    We then return the cheaper clustering. This is a \fc of minimum cost as either \(u\) and \(v\) are placed in the same cluster or not, and for both cases, \(\mcP_1\) and \(\mcP_2\) are of minimum cost, respectively.
\end{proof}

\subsection{Color Ratio 1\nwspace :\nwspace 2}
\label{sec:12forests}

We now give algorithms for \fcc on forests that do not require a certain diameter or degree.
As a first step to solve these less restricted instances, we develop an algorithm to solve \fcc on forests with a color ratio of \(1:2\). 

W.l.o.g., the vertices are colored blue and red with twice as many red vertices as blue ones. We call a connected component of size 1 a \emph{\bcomp} or \emph{\rcomp}, depending on whether the contained vertex is blue or red. Analogously, we apply the terms \emph{\brcomp, \rrcomp}, and \emph{\brrcomp} to components of size 2 and 3.

\subsubsection{Linear Time Attempt}

Because of \autoref{lem:smallClustersForest}, we know that in every minimum-cost \fc each cluster contains exactly 1 blue and 2 red vertices. 
Our high-level idea is to employ two phases.

In the first phase, we partition the vertices of the forest \(F\) in a way such that in every cluster there are at most 1 blue and 2 red vertices. We call such a partition a \emph{splitting} of \(F\). We like to employ a standard tree dynamic program that bottom-up collects vertices to be in the same connected component and cuts edges if otherwise there would be more than 1 blue or 2 red vertices in the component. We have to be smart about which edges to cut, but as only up to 3 vertices can be placed in the topmost component, we have only a limited number of possibilities we have to track to find the splitting that cuts the fewest edges.

After having found that splitting, we employ a second phase, which finds the best way to assemble a fair clustering from the splitting by merging components and cutting as few additional edges as possible. As, by \autoref{lem:costByCuts}, a fair partition with the smallest inter-cluster cost has a minimum \cccost, this would find a minimum-cost \fc. 

Unfortunately, the approach does not work that easily. We find that the number of cuts incurred by the second phase also depends on the number of \(br\)- and \rcomps.

\begin{lemma}\label{lem:assembling}
    Let \(F=(V,E)\) be an \(n\)-vertex forest with colored vertices in blue and red in a ratio of \(1:2\). 
    Suppose in each connected component (in the above sense) 
    there is at most 1 blue vertex and at most 2 red vertices.
    Let \(\#(br)\) and \(\#(r)\) be the number of \(br\)- and \rcomps, respectively.
    Then, after cutting \(\max (0, \frac{\#(br) -\#(r)}{2})\) edges, the remaining connected components can be merged such that all clusters consist of exactly 1 blue and 2 red vertices. Such a set of edges can be found in time in $\bigO(n)$.
    Further, when cutting less than \(\max(0, \frac{\#(br) -\#(r)}{2})\) edges, such merging is not possible.
\end{lemma}

\begin{proof}
    As long as possible, we arbitrarily merge \bcomps with \rrcomps as well as  \brcomps with \rcomps. For this, no edges have to be cut.
    Then, we split the remaining \rrcomps and merge the resulting \rcomps with one \brcomp each. This way, we incur \(\max(0,\frac{\#(br)-\#(r)}{2})\) more cuts and obtain a fair clustering as now each cluster contains two red and one blue vertex. This procedure is done in time in $\bigO(n)$.

    Further, there is no cheaper way.
    For each \brcomp to be merged without further cuts we require an \rcomp. There are \(\#(r)\) \rcomps and each cut creates either at most two \rcomps or one \rcomp while removing a \brcomp. Hence, \(\max(0,\frac{\#(br)-\#(r)}{2})\) cuts are required.
\end{proof}

For our approach to work, the first phase has to simultaneously minimize the number of cuts as well as the difference between \(br\)- and \rcomps. This is, however, not easily possible.
Consider the tree in \autoref{fig:tree_12_gadget}.

\begin{figure}
    \centering
    \includegraphics[height=.3\textwidth]{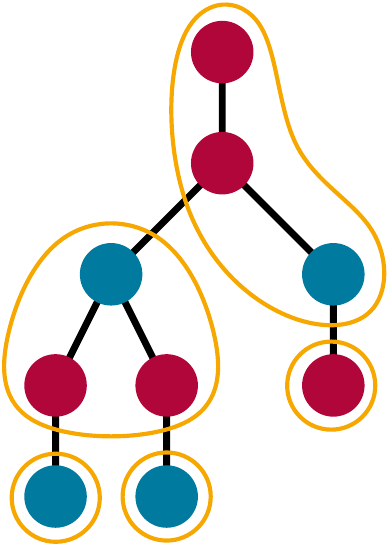}
    \hspace{.1\textwidth}
    \includegraphics[height=.3\textwidth]{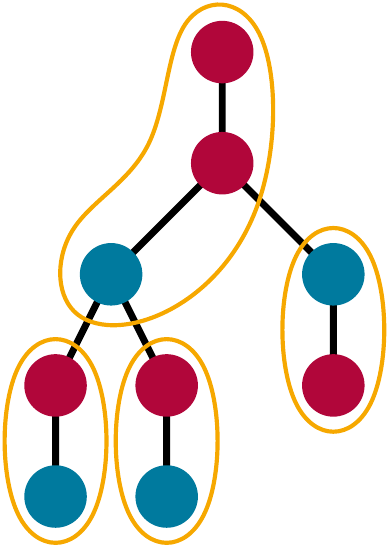}
    \caption{A tree for which the splitting with the minimum number of cuts (right) has 3 more \brcomps and 1 less \rcomp than a splitting with one more edge cut (left).}\label{fig:tree_12_gadget}
\end{figure}

There, with one additional cut edge we have three \brcomps less and one \rcomp more. Using a standard tree dynamic program, therefore, does not suffice as when encountering the tree as a subtree of some larger forest or tree, we would have to decide between optimizing for the number of cut edges or the difference between \(br\)- and \rcomps. There is no trivial answer here as the choice depends on how many \(br\)- and \rcomps are obtained in the rest of the graph.
For our approach to work, we hence have to track both possibilities until we have seen the complete graph, setting us back from achieving a linear running time.

\subsubsection{The Join Subroutine} 

In the first phase, we might encounter situations that require us to track multiple ways of splitting various subtrees. When we reach a parent vertex of the roots of these subtrees, we join these various ways of splitting. For this, we give a subroutine called \join. We first formalize the output by the following lemma, then give an intuition on the variables, and lastly prove the lemma by giving the algorithm. 

\begin{lemma}
\label{lem:join}
    Let \(R_1,R_2,\ldots,R_{\ell_1}\) for \(\ell_1\in\N_{>1}\) with \(R_i \in (\N\cup\{\infty\})^{\ell_2}\) for \(\ell_2\in\N, i\in[\ell_1]\) and \(f\) be a computable function \(f\colon [\ell_2]\times [\ell_2]\rightarrow 2^{[\ell_2]}\).
    For \(x\in [\ell_2]\), let
    \begin{align*}
        A_x &= \{M \in \left([\ell_2]\right)^{\ell_1}\mid x\in \widehat{f}(M[1], M[2],\ldots, M[\ell_2])\},
    \intertext{whereby for all \(x_1,x_2, \ldots \in [\ell_2]\)}
        \widehat{f}(x_1, x_2) &= f(x_1, x_2)
    \intertext{and for all \(2\le k \le \ell_2\)}
        \widehat{f}(x_1,x_2,\ldots, x_k) &= \bigcup_{x\in \widehat{f}(x_1,x_2,\ldots, x_{k-1})} f(x, x_k).
    \end{align*}
    Then, an array \(R\in (\N\cup\{\infty\})^{\ell_2}\) such that \(R[x] = \min_{M\in A_x} \sum_{i=1}^{\ell_1} R_i[M[i]]\) for all \(x\in[\ell_2]\) can be computed in time in $\bigO(\ell_1\cdot \ell_2^2\cdot T_f)$, where \(T_f\) is the time required to compute \(f\).
\end{lemma}

As we later reuse the routine, it is formulated more generally than required for this section. 
Here, for the \(1:2\) case, assume we want to join the splittings of the children \(u_1,u_2,\ldots, u_{\ell_1}\) of some vertex \(v\). 
For example, assume \(v\) has three children as depicted in \autoref{fig:joinExample}.

\begin{figure}
    \centering
    \includegraphics[height=.3\textwidth]{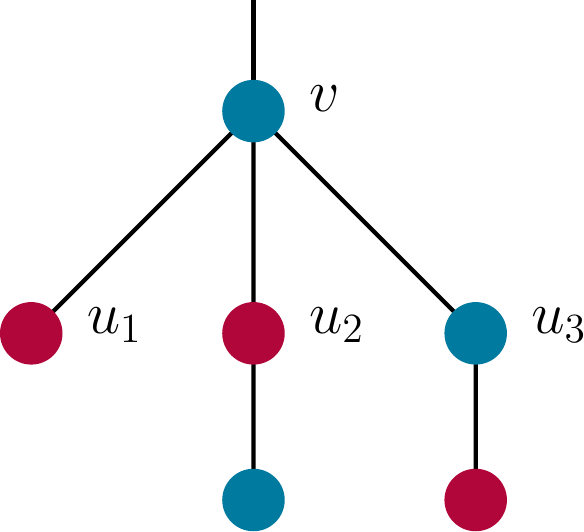}
    \caption{Exemplary graph for a \joinroutine.}
    \label{fig:joinExample}
\end{figure}

Then, for each child \(u_i\), let there be an array \(R_i\) such that \(R_i[x]\) is the minimum number of cuts required to obtain a splitting of the subtree \(T_{u_i}\) that has exactly \(x\) more \brcomps than \rcomps.
For our example, assume all edges between \(v\) and its children have to be cut.
We see, that \(R_1[-1]=1\) and \(R_1[x]=\infty\) for \(x\neq -1\), as the only possible splitting for the subtree of \(u_1\) cuts only the edge to \(v\) and has one more \rcomp than \brcomps.
Further, we have \(R_2[1]=1\) (by only cutting \(\{v,u_2\}\)), \(R_2[-1]=2\) (by cutting both edges of \(u_2\)), and \(R_2[x]=\infty\) for \(x\notin \{-1,1\}\). Last, note that \(R_3=R_2\). 

The function \(f\) returns the set of indices that should be updated when merging two possibilities. When a splitting of one child's subtree has \(x_1\) more \brcomps and a splitting of another child's subtree has \(x_2\) more \brcomps, then the combination of these splittings has \(x_1+x_2\) more \brcomps than \rcomps. Hence, the only index to update is \(f(x_1,x_2) = \{x_1 + x_2\}\). Later, we will require to update more than a single index, so \(f\) is defined to return a set instead of a single index.
Note that by the definition of \(f\) and \(\widehat{f}\), each value placed in \(R[x]\) by the routine corresponds to choosing exactly one splitting from each array \(R_i\) such that the total difference between \brcomps and \rcomps sums up to exactly \(x\). 

In our example, assume any splitting is chosen for each of the three subtrees. Let \(x_i\) denote the difference of \(br\)- and \rcomps of the chosen splitting for the subtree rooted at \(u_i\) for \(1\le i\le 3\). Then, \join sets \(R[x]\) for \(x=x_1+x_2+x_3\). 
If there are multiple ways to achieve an index \(x\), the one with the minimum number of cuts is stored in \(R[x]\).
In the example, we have 4 possibilities, as \(x_1 = -1\) and \(x_2,x_3\in \{-1,1\}\). Note that \(x_1=-1, x_2=-1, x_3=1\) and \(x_1=-1, x_2=1, x_3=-1\) both evaluate to \(x=-1\). Hence, only one of the two combinations is stored (the one with fewer cuts, here an arbitrary one as both variants imply 4 cuts).
For the resulting array \(R\), we have \(R[-3]=5, R[-1]=4, R[1]=3\), and \(R[x]=\infty\) for \(x\notin\{-3,-1,1\}\). Observe that the numbers of cuts in \(R\) correspond to the sums of the numbers of cuts in the subtrees for the respective choice of \(x_i\).

We now describe how the \joinroutine is computed.

\begin{proof}[Proof of \autoref{lem:join}]
    The algorithm works in an iterative manner. Assume it has found the minimum value for all indices using the first \(i-1\) arrays and they are stored in \(R^{i-1}\). It then \emph{joins} the \(i\)-th array by trying every index \(x_1\) in \(R^{i-1}\) with every index \(x_2\) in \(R_i\). Each time, for all indices \(x\in f(x_1,x_2)\), it sets \(R^i[x]\) to \(R^{i-1}[x_1]+R_i[x_2]\) if it is smaller than the current element there. Thereby, it tries all possible ways of combining the interim solution with \(R_i\) and for each index tracks the minimum that can be achieved. 
    Formally, we give the algorithm in \autoref{algo:join}. 

    \begin{algorithm}
    \setstretch{1.25}
        \KwIn{\(R_1,R_2,\ldots,R_{\ell_1}\) for \(\ell_1 \ge 2\) with \(R_i \in (\N\cup\{\infty\})^{\ell_2}\) for \(0\le i< \ell_1\), and a computable function \(f\colon [\ell_2]\times[\ell_2]\rightarrow 2^{[\ell_2]}\).}
        \KwOut{\(R\in (\N\cup\{\infty\})^{\ell_2}\) such that, for all \(x\in [\ell_2]\), \(R[x] = \min_{M\in A_x} \sum_{i=1}^{\ell_1} R_i[M[i]]\)
        with \(A_x = \{M \in \left([\ell_2]\right)^{\ell_1}\mid x\in \widehat{f}(M[1], M[2],\ldots, M[\ell_2])\}\),
        \(\widehat{f}(x_1,x_2,\ldots, x_k) = \bigcup_{x\in \widehat{f}(x_1,x_2,\ldots, x_{k-1})} f(x, x_k)\), and \(\widehat{f}(x_1, x_2) = f(x_1, x_2)\). }
        \(R\gets R_1\)\\
        \For{$i\gets 2$ \KwTo $\ell_1$}{
            \(R'\gets R\)\\
            \ForEach{$(x_1,x_2)\in \left([\ell_2]\right)^2$}{\label{algo:line_conv1}
                \ForEach{$x\in f(x_1,x_2)$}{\label{algo:line_conv2}
                    \(R'[x]\gets \min\left(R'[x], R[x_1]+R_i[x_2]\right)\)\label{algo:line_conv3}
                }
            }
            \(R\gets R'\)
        }
        \caption{The \join subroutine.}
    \label{algo:join}
    \end{algorithm}

    The algorithm terminates after $\bigO(k\cdot \ell^2\cdot T_f)$ iterations due to the nested loops. 
    We prove by induction that \(R\) is a solution of \join over the arrays \(R_1,\ldots,R_i\) after each iteration \(i\). 
    The first one simply tries all allowed combinations of the arrays \(R_1,R_2\) and tracks the minimum value for each index, matching our definition of \join.
    Now assume the statement holds for some \(i\).
    Observe that we only update a value \(R[x]\) if there is a respective \(M\in A_x\), so none of the values is too small. 
    To show that no value is too large, take any \(x\in[\ell_2]\) and let \(a\) be the actual minimum value that can be obtained for \(R[x]\) in this iteration. 
    Let \(j_1,j_2,\ldots, j_{i+1}\) with \(x\in\widehat{f}(j_1,j_2,\ldots,j_{i+1})\) be the indices that obtain \(a\). 
    Then, there is \(y\in [\ell_2]\) such that after joining the first \(i\) arrays the value at index \(y\) is \(a-R_{i+1}[j_{i+1}]\) and \(y\in\widehat{f}(j_1,j_2,\ldots,j_i)\). 
    This implies \(R[y]\le a-R_{i+1}\) by our induction hypothesis.
    Further, as both \(x\in\widehat{f}(j_1,j_2,\ldots,j_{i+1})\) and \(y\in\widehat{f}(j_1,j_2,\ldots,j_i)\), we have \(x\in f(y, j_{i+1})\). Thus, in this iteration, \(R[x]\) is set to at most \(R[y]+R_{i+1}[j_{i+1}]\le a\). With this, all values are set correctly.    
\end{proof}

Observe that in the case of \(f(x_1,x_2)= \{x_1 + x_2\}\), which is relevant to this section, the loop in lines~4-6
computes the \((\min,+)\)-convolution of the arrays \(R\) and \(R_i\). Simply trying all possible combinations as done in the algorithm has a quadratic running time. This cannot be improved without breaking the \textsc{MinConv Conjecture}, which states there is no algorithm computing the \((\min,+)\)-convolution of two arrays of length \(n\) in time in $\bigO(n^{2-\varepsilon})$ for any constant \(\varepsilon>0\)~\cite{Cygan_2019}.

\subsubsection{The Tracking Algorithm}

With the \joinroutine at hand, we are able to build a dynamic program solving \fcc on forests with two colors in a ratio of \(1:2\). We first describe how to apply the algorithm to trees and then generalize it to work on forests.
  
In the first phase, for each possible difference between the number of \brcomps and \rcomps, we compute the minimum number of cuts to obtain a splitting with that difference. 
In the second phase, we find the splitting for which the sum of edges cut in the first phase and the number of edges required to turn this splitting into a fair partition is minimal. This sum is the inter-cluster cost of that partition, so by \autoref{lem:costByCuts} this finds a fair partition with the smallest \cccost. 

\paragraph*{Splitting the tree.}
In the first phase, our aim is to compute an array \(D\), such that, for all integers \(-n \le x \le \frac{n}{3}\), \(D[x]\subseteq E\) is a minimum-sized set of edges such that \(x = br(T-D[x]) - r(T-D[x])\), where \(br(T-D[x])\) and \(r(T-D[x])\) are the number of \(br\)- and \rcomps in \(T-D[x]\), respectively. To mark the case if no such set exists, we expect \(D[x]=\N\) to have an infinitely large entry.
We fill the array in a dynamic programming way, by computing an array \(D_v^h\) for each vertex \(v\), and every possible \emph{head} \(h\in \heads\). 
Here, \(D_v^h[x]\), is a minimum-sized set of edges such that in the subtree \(T_v\) rooted at \(v\) upon removal we have exactly \(x\) more \brcomps than \rcomps. The head \(h\) refers to the colors in the topmost component, which is of particular interest as it might later contain vertices from outside \(T_v\) as well. Head \(h=r\) refers to a component with a red vertex, \(h=br\) with a blue and a red vertex so on. This component is empty (\(h=\emptyset\)) if the edge above \(v\) is cut. The head is not counted as an \brcomp or \rcomp for the computation of \(x\). \autoref{fig:tree12DP} gives examples of how a head is composed from the splittings of the children.

\begin{figure}
    \centering
    \begin{subfigure}{0.25\textwidth}
        \includegraphics[width=\textwidth]{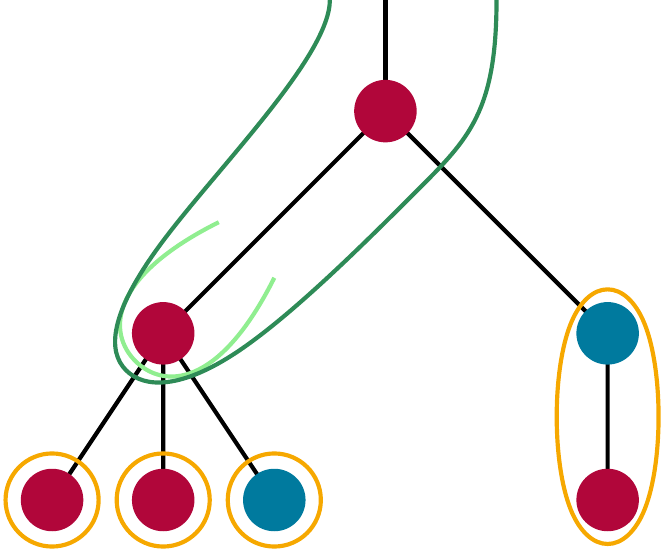}
        \caption{}\label{fig:tree12DPa}
    \end{subfigure}
    \hspace{.15\textwidth}
    \begin{subfigure}{0.25\textwidth}
        \includegraphics[width=\textwidth]{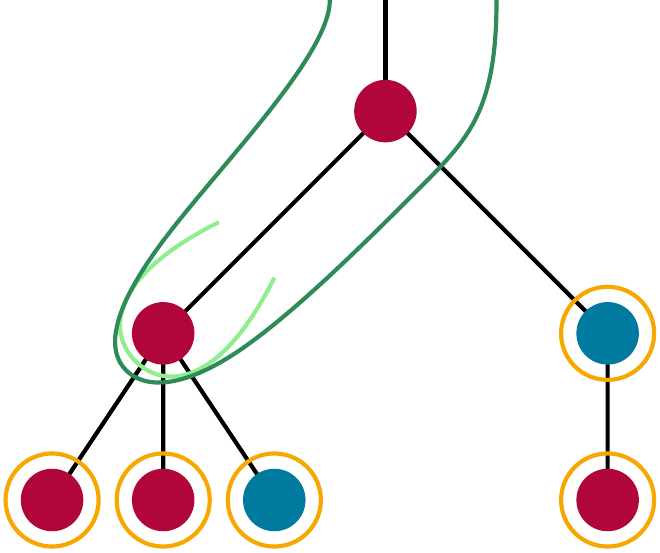}
        \caption{}\label{fig:tree12DPb}
    \end{subfigure}\\
    \bigskip
    \begin{subfigure}{0.25\textwidth}
        \includegraphics[width=\textwidth]{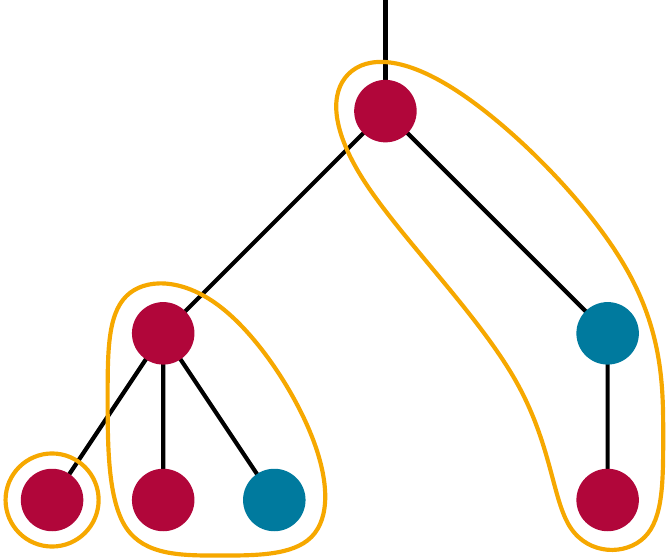}
        \caption{}\label{fig:tree12DPc}
    \end{subfigure}
    \hfill    
    \begin{subfigure}{0.25\textwidth}
        \includegraphics[width=\textwidth]{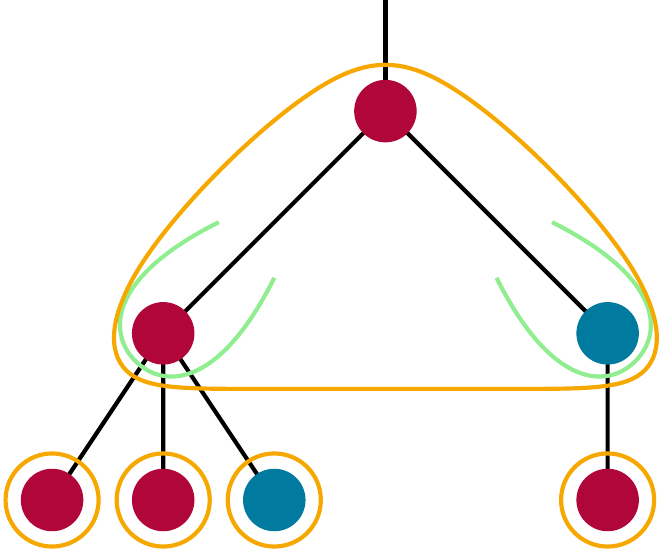}
        \caption{}\label{fig:tree12DPd}
    \end{subfigure}
    \hfill
    \begin{subfigure}{0.25\textwidth}
        \includegraphics[width=\textwidth]{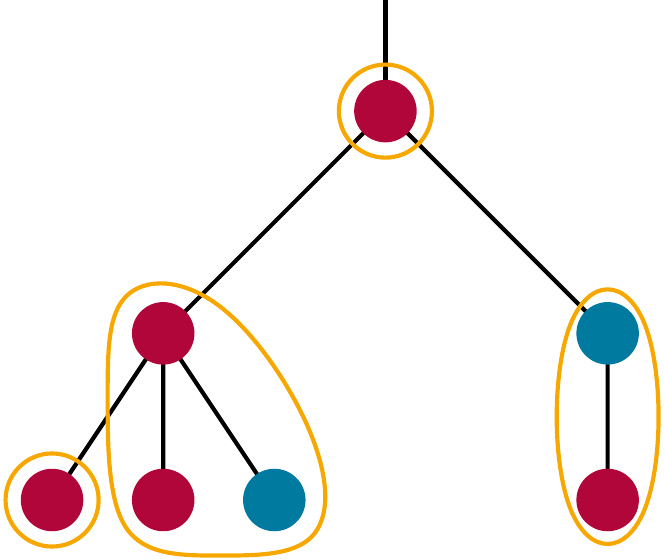}
        \caption{}\label{fig:tree12DPe}
    \end{subfigure}
    \caption{Exemplary subtree with various possibilities to obtain a head. \Cref{fig:tree12DPa,fig:tree12DPb} show splittings with an \(rr\)-head (dark green). The choice for the heads of the children (light green) is unambiguous as the only way to obtain an \(rr\)-head is to choose the \(r\)-head for the left child and an \(\emptyset\)-head for the right one. Both the left and the right variants have to be considered as they differ in the number of \brcomps minus the number of \rcomps.
    The splittings in \Cref{fig:tree12DPc,fig:tree12DPd,fig:tree12DPe} create an \(\emptyset\)-head, as they cut the edge above the root of the subtree, so no vertices of the subtree can be part of a component with vertices outside the subtree. Out of these 3 splittings, however, only \Cref{fig:tree12DPc,fig:tree12DPd} will be further considered as \autoref{fig:tree12DPe} obtains the same difference between \(br\)- and \rcomps as \autoref{fig:tree12DPc} but cuts one more edge. We note that other splittings obtain an \(\emptyset\)-head as well that are not listed here.}\label{fig:tree12DP}
\end{figure}

In the following, we only show how to compute \(\Delta_v^h[x] = |D_v^h[x]|\), the size of the set of edges to obtain a respective splitting. The set \(D_v^h[x]\) is, however, obtained by a simple backtracking approach in the same asymptotic running time. If \(D_v^h[x] = \N\), we have \(\Delta_v^h[x]=\infty\).
We initialize all values with \(\Delta_v^h[x]=\infty\), meaning we know of no set of edges which upon removal give that head and that difference between \(br\)- and \rcomps.
Then, for every red leaf \(v\) we set \(\Delta_v^r[0] = 0\) and \(\Delta_v^\emptyset[-1] = 1\).
For every blue leaf \(v\) we set \(\Delta_v^b[0] = 0\) and \(\Delta_v^\emptyset[0] = 1\). This concludes the computations for the leaves, as the only possibilities are to cut the edge above the leaf or not. 
Now suppose we have finished the computation for all children \(u_1, u_2, \ldots, u_k\) of some vertex \(v\). Observe that at most two children of \(v\) are placed in a head with \(v\). For every head \(h\in\heads\) that is formable at vertex \(v\), we try all possibilities to obtain that head. 

If \(h\in \{r,b\}\) and \(c(v)\) corresponds to \(h\), this is done by choosing \(\emptyset\) heads for all children. There is no unique splitting of the subtrees however, as for each subtree rooted at some child vertex \(u_i\) there is a whole array \(D_{u_i}^\emptyset\) of possible splittings with different numbers of \(br\)- and \rcomps. To find the best choices for all child vertices, we employ the \joinroutine that, when called with \(f(x_1,x_2) = \{x_1+x_2\}\) and a list of arrays, returns an array \(R\) such that, for all indices \(x\) \(R[x]\) is the minimum value obtained by summing up exactly one value from each of the input arrays such that the indices of the chosen values sum up to \(i\). We hence set \(\Delta_v^h = \join(\Delta_{u_1}^\emptyset,\ldots,\Delta_{u_{k}}^\emptyset)\). Here and in the following, we only call the \joinroutine with at least two arrays. If we would only input a single array, we go on as if the \joinroutine returned that array.
We note that here our indexing ranges from \(-n\) to \(\frac{n}{3}\) while the \joinroutine assumes positive indices. We hence implicitly assume that an index of \(x\) here maps to an index \(x+n+1\) in the subroutine.

If \(h=br\) or both \(h=rr\) and \(c(v)\) corresponds to \(r\), then the heads for all children should be \(\emptyset\) except for one child that we place in the same component as \(v\). It then has a head \(h'\in \{r,b\}\), depending on \(h\) and \(c(v)\). We have \(h'=r\) if \(h=rr\) and \(c(v)\) corresponds to \(R\) or \(h=rb\) and \(c(v)\) corresponds to \(b\). Otherwise, \(h'=b\).
For all \(i\in[k]\), we compute an array \(\Delta'_{u_i} = \join(\Delta_{u_1}^\emptyset, \ldots, \Delta_{u_{i-1}}^\emptyset, \Delta_{u_i}^{h'}, \Delta_{u_{i+1}}^\emptyset,\ldots, \Delta_{u_k}^\emptyset)\), referring to \(u_i\) having the non-empty head.
Lastly, for all \(-n\le x \le \frac{n}{3}\), we set \(\Delta_v^h[x]=\min_{i\in [k]} \Delta'_{u_i}[x]\). 

If \(h=\emptyset\), then we have to try out all different possibilities for the component \(v\) is in and, in each case, cut the edge above \(v\). 
First assume we want to place \(v\) in a \brrcomp. Then it has to be merged with to vertices, either by taking a head \(h'\in \{br,rr\}\) at one child or by taking heads \(h_1,h_2 \in \{r,b\}\) at two children. The exact choices for \(h',h_1,h_2\) of course depend on \(c(v)\).
We compute an array \(\Delta_{h'} = \join(\Delta_{u_1}^\emptyset, \ldots, \Delta_{u_{i-1}}^\emptyset, \Delta_{u_i}^{h'}, \Delta_{u_{i+1}}^\emptyset,\ldots, \Delta_{u_k}^\emptyset)\) for the first option.
For the second option, we compute the arrays\\ \(\Delta_{i,j}=\join(\Delta_{u_1}^\emptyset, \ldots, \Delta_{u_{i-1}}^\emptyset, \Delta_{u_i}^{h_1}, \Delta_{u_{i+1}}^\emptyset,\ldots, \Delta_{u_{j-1}}^\emptyset, \Delta_{u_j}^{h_2}, \Delta_{u_{j+1}}^\emptyset,\ldots, \Delta_{u_k}^\emptyset)\) for all pairs of children \(u_i,u_j\) of \(v\) such that \(i<j\) and \(\{v,u_i,u_j\}\) is a \brrcomp. 
We now have stored the minimum number of cuts for all ways to form a \brrcomp with \(v\) and for all possibilities for \(x\) in the arrays \(\Delta_{h'}\) and \(\Delta_{i,j}\) for all possibilities of \(i,j\). 
However, \(v\) may also be in an \(r\)-, \(b\)-, \(rr\)-, or \brcomp. Hence, when computing \(\Delta_v^\emptyset[x]\) we take the minimum value at position \(x\) not only among the arrays \(\Delta_{h'}\) and \(\Delta_{i,j}\) but also of the arrays \(\Delta_v^{r}, \Delta_v^{br}, \Delta_v^{rr}\), and \(\Delta_v^{br}\). 
Note that here we have to shift all values in \(\Delta_v^{r}\) to the left by one since by isolating \(v\) we create another \rcomp. An entry we have written into \(\Delta_v^{r}[x]\) hence should actually be placed in \(\Delta_v^{r}[x-1]\). Similarly, we have to shift \(\Delta_v^{br}\) to the right, since here we create a new \brcomp at the top of the subtree. 
Lastly, as long as \(v\) is not the root of \(T\), we have to increase all values in \(\Delta_v^\emptyset\) by one, reflecting the extra cut we have to make above \(v\). 

After all computations are completed by the correctness of the \joinroutine and an inductive argument, \(\Delta_v^h\) is correctly computed for all vertices \(v\) and heads \(h\). 
Note that in the \joinroutine, as \(f(x_1,x_2)\) returns the correct index for merging two subtrees, \(\widehat{f}(x_1,x_2,\ldots,x_k)\) gives the correct index of merging \(k\) subtrees.
In particular, \(\Delta_r^\emptyset\) is the array containing for each \(-n\le x\le \frac{n}{3}\) the minimum number of edges to cut such that the there are exactly \(x\) more \brcomps than \rcomps, where \(r\) is the root of \(T\). By adjusting the \joinroutine to track the exact combination that leads to the minimum value at each position, we also obtain an array \(D\) that contains not only the numbers of edges but the sets of edges one has to cut or is marked with \(\N\) if no such set exists. 

At each node, computing the arrays takes time $\bigO(n^5)$, which is dominated by computing $\bigO(n^2)$ arrays \(D_{u,w}\) in time $\bigO(n^3)$ each by \autoref{lem:join} since \(\ell_1,\ell_2 \in \bigO(n)\). This phase hence takes time in $\bigO(n^6)$.

\paragraph*{Assembling a fair clustering.} 
Let \(D\) be the set computed in the first phase. Note that each set of edges \(D[x]\) directly gives a splitting, namely the partition induced by the connected components in \(T-D[x]\).

By \autoref{lem:assembling}, the cheapest way to turn the splitting given by \(D[x]\) into a clustering of sets of 1 blue and 2 red vertices is found in linear time and incurs \(\frac{\max(0,x)}{2}\) more cuts. Hence, we find the \(-n \le x \le \frac{n}{3}\) for which \(|D[x]| + \max(0,\frac{x}{2})\) is minimal. We return the corresponding clustering as it has the minimum inter-cluster cost.

This phase takes only constant time per splitting if we tracked the number of components of each type in the first phase and is therefore dominated by the first phase. 

\paragraph*{Forests.}
Our algorithm is easily generalized to also solve \fcc on unconnected forests with two colors in a ratio of \(1:2\) by slightly adapting the first phase. We run the dynamic program as described above for each individual tree. This still takes overall time in $\bigO(n^6)$. For each tree \(T_i\) in the forest and every \(h\in\heads\), let then \(\Delta_{T_i}^\emptyset\) denote the array \(\Delta_r^\emptyset\) with \(r\) being the root of tree \(T_i\). To find a splitting of the whole forest and not just on the individual trees, we perform an additional run of the \joinroutine using these arrays \(\Delta_{T_i}\) and the function \(f(x_1,x_2)=\{x_1+x_2\}\). This gives us an array \(R\) such that \(R[x]\) is the minimum number of cuts required to obtain a splitting with exactly \(x\) more \brcomps than \rcomps for the whole tree rather than for the individual trees. Note that we choose the \(\emptyset\)-head at each tree as the trees are not connected to each other, so in order to find a splitting we do not yet have to consider how components of different trees are merged, this is done in the second phase. The first phase then outputs an array \(D\) that contains the set of edges corresponding to \(R\), which is obtained by a backtracking approach. As the additional subroutine call takes time in $\bigO(n^3)$, the asymptotic run time of the algorithm does not change. This gives the following result.

\begin{theorem}\label{thm:forestOneTwoAlgo}
    \emph{\fcc} on forests with two colors in a ratio of \(1:2\) can be solved in time in $\bigO(n^6)$.
\end{theorem}

\subsection{Small Clusters}
\label{subsec:forestTracing}

To obtain an algorithm that handles more colors and different color ratios, we generalize our approach for the \(1:2\) color ratio case from the previous section.
We obtain the following.

\begin{theorem}
\label{thm:forestByColorsAlgo}
    Let \(F\) be a forest of \(n\) vertices, each colored in one of \(k\ge 2\) colors. Let the colors be distributed in a ratio of \(c_1:c_2:\ldots:c_k\) with \(c_i\in\N_{>0}\) for all \(i\in [k]\) and \(\gcd(c_1,c_2,\ldots,c_k)=1\).
    Then \emph{\fcc} on \(F\) can be solved in time in 
    $\bigO(n^{2 \nwspace \svars+\smax+2}\cdot \svars^\smax)$, where \(\svars=\prod_{i=1}^k (c_i+1)\) and \(\smax = \sum_{i=1}^k c_i\).
\end{theorem}

Once more, the algorithm runs in two phases. First, it creates a list of possible splittings, i.e., partitions in which, for every color, every component has at most as many vertices of that color as a minimum-sized fair component has. In the second phase, it checks for these splittings whether they can be merged into a fair clustering.
Among these, it returns the one of minimum cost. We first give the algorithm solving the problem on trees and then generalize it to also capture forests.
\paragraph*{Splitting the forest.}
For the first phase in the 1:2 approach, we employed a dynamic program that kept track of the minimum number to obtain a splitting for each possible cost incurred by the reassembling in the second phase.
Unfortunately, if we are given a graph with \(k\ge 2\) colors in a ratio of \(c_1:c_2:\ldots:c_k\), then the number of cuts that are required in the second phase is not always as easily bounded by the difference of the number of two component types like \(r\)- and \brcomps in the \(1:2\) case.
However, we find that it suffices to track the minimum number of cuts required to obtain any possible coloring of a splitting. 

We first bound the number of possible colorings of a splitting. As during the dynamic program we consider splittings of a subgraph of \(G\) most of the time, we also have to count all possible colorings of splittings of less than \(n\) vertices. 

\begin{lemma}
\label{lem:numberColoringsOfPartitions}
    Let \(U\) be a set of \(n\) elements, colored in \(k\in\N_{>1}\) colors, and let \(d_1, d_2,\ldots,d_k\in\N\).
    Let \(\mcS\) be the set of all possible partitions of subsets of \(U\) such that for every color \(i\) there are at most \(d_i\) vertices of that color in each cluster. 
    Let \(\mathcal{C}\) be the set of all colorings of partitions in \(\mcS\).
    Then, \(|\mathcal{C}| \le (n+1)^{\svars-1}\), where \(\svars = \prod_{i=1}^k (d_i+1)\).
\end{lemma}
\begin{proof}
    The number of sets with different colorings is at most $\svars$
    as there are \(0\) to \(d_i\) many vertices of color \(i\) in each component. 
    Thus, a coloring of a partition \(\mcP\) using only these sets is characterized by an array of size \(\svars\) with values in \([n]\cup\{0\}\) as no component occurs more than \(n\) times. There are 
        \((n+1)^{\svars}\) 
    ways to fill such an array. However, as the set colorings together have to form a partition, the last entry is determined by the first \(\svars - 1\) entries, giving only 
    \((n+1)^{\svars -1}\) possibilities.
\end{proof}

With this, we employ a dynamic program similar to the one presented in \autoref{sec:12forests} but track the minimum cut cost for all colorings of splittings. It is given by the following lemma.

\begin{lemma}
\label{lem:splittingForests}
Let \(F=(V,E)\) be a forest with vertices in \(k\) colors. Further,
let \(d_1,d_2,\ldots,d_k\in\N\) and
\(\mcS\) be the set of all possible partitions of \(V\) such that there are at most \(d_i\) vertices of color \(i\) in each cluster for \(i\in [k]\). 
Let \(\mathcal{C}\) be the set of all colorings of partitions in \(\mcS\).
Then, in time in $\bigO(n^{2\nwspace\svars+\smax+2}\cdot \svars^\smax)$ with \(\svars=\prod_{i=1}^k (d_i+1)\) and \(\smax = \sum_{i=1}^k d_i\), for all \(C\in\mathcal{C}\), we find a minimum-sized set \(D_C\subseteq E\) such that the connected components in \(F-D_C\) form a partition of the vertices with coloring \(C\) or certify that there is no such set. 
\end{lemma}
\begin{proof}
    We first describe how to solve the problem on a tree \(T\) and then generalize the approach to forests. We call a partition of the vertices such that for every color \(i\) there are at most \(d_i\) vertices of that color in each cluster a \emph{splitting}. 

    We employ a dynamic program that computes the set \(D_C\) for the colorings of all possible splittings and all subtrees rooted at each vertex in \(T\). We do so iteratively, by starting to compute all possible splittings at the leaves and augmenting them towards the root. Thereby, the connected component that is connected to the parent of the current subtree's root is of particular importance as it is the only connected component that can be augmented by vertices outside the subtree. We call this component the \emph{head}. Note that the head is empty if the edge between the root and its parent is cut. We do not count the head in the coloring of the splitting and only give it explicitly.
    Formally, for every \(v\in V\), every possible coloring of a splitting \(C\), and every possible coloring \(h\) of the head we compute \(D_v^h[C]\subseteq E\), the minimum-sized set of edges such that the connected components of \(T_v-D_v^h[C]\) form a splitting with coloring \(C\) and head \(h\). We set \(D_v^h[C]=\N\), an infinitely large set, if no such set exists.

    Let all \(D_v^h[C]\) be initialized with \(\N\). 
    Then, for every leaf \(v\) with parent \(w\), we set \(D_v^{h_{c(v)}}[C_\emptyset] = \emptyset\), where \(h_{c(v)}\) is the coloring of the component \(\{v\}\) and \(C_\emptyset\) the coloring of the partition over the empty set. Also, we set \(D_v^{h_\emptyset}[C_{c(v)}] = \{\{v,w\}\}\), where the vertex \(v\) is not placed in the head as the edge to its parent is cut. As to cut or not to cut the edge above are the only options for leaves, this part of the array is now completed.
    
    Next, suppose we have finished the computation for all children of some vertex \(v\). 
    For every possible coloring \(h\) of the head that is formable at vertex \(v\), we try all possibilities to obtain that coloring. 
    
    To this end, first assume \(h\) to be non-empty. 
    Therefore, \(v\) has to be placed in the head.
    Let \(h_{-c(v)}\) denote the coloring obtained by decreasing \(h\) by one at color \(c(v)\). 
    To obtain head \(h\), we hence have to choose colorings of splittings of the subtrees rooted at the children \(u_1,u_2,\ldots, u_{\ell}\) of \(v\) such that their respective heads \(h_{u_1}, h_{u_2}, \ldots, h_{u_{\ell}}\) combine to \(h_{-c(v)}\). 
    A \emph{combination} of colorings \(C_1,C_2,\ldots, C_\ell\) refers to the coloring of the union of partitions \(M_1,M_2,\ldots, M_\ell\) that have the respective colorings and is defined as the element-wise sum over the arrays \(C_1,C_2,\ldots, C_\ell\). 
    Often, there are multiple ways to choose heads for the child vertices that fulfill this requirement. 
    As every head is of size at most $\smax$, \(h_{-c(v)}\) and contains \(v\), it is composed of less than $\smax$ non-empty heads.
    As there are at most \(\svars\) possible heads and we have to choose less than \(\smax\) children, there are at most 
    \(\binom{n}{\smax-1}\cdot \svars^{\smax - 1} < n^{\smax -1}\cdot \svars^{\smax - 1}\) possible ways to form \(h_{-c(v)}\) with the children of \(v\). Let each way be described by a function \(H\) assigning each child of \(v\) a certain, possibly empty, head. Then, even for a fixed \(H\), there are multiple splittings possible. This stems from the fact that even if the head \(H(u)\) for a child \(u\) is fixed, there might be multiple splittings of the subtree of \(u\) with different colorings resulting in that head. For each possible \(H\), we hence employ the \joinroutine with the arrays \(D_u^{H(u)}\) for all children \(u\) using the cardinality of the sets as input for the subroutine. 
    For the sake of readability, we index the arrays here by some vector \(C\) instead of a single numerical index as used in the algorithmic description of the \joinroutine. We implicitly assume that each possible coloring is represented by a positive integer. By letting these indices enumerate the vectors in a structured way, converting between the two formats only costs an additional time factor in $\bigO(n)$.
    
    For \(f(x_1,x_2)\) we give the function returning a set containing only the index of the coloring obtained by combining the colorings indexed by \(x_1\) and \(x_2\), which is computable in time in $\bigO(n)$. Combining the colorings means for each set coloring summing the occurrences in both partition colorings. Thereby, \(\widehat{f}(x_1,x_2,\ldots,x_k)\) as defined in the \joinroutine returns the index of the combination of the colorings indexed by \(x_1,x_2,\ldots,x_k\).
    Note that there are at most \(n\) arrays and each is of length less than \((n+1)^{\svars -1}\) as there are so many different colorings by \autoref{lem:numberColoringsOfPartitions}.
    After executing the \joinroutine, by \autoref{lem:join}, we obtain an array \(D_{H}\) that contains the minimum cut cost required for all possible colorings that can be achieved by splitting according to \(H\). By modifying the \joinroutine slightly to use a simple backtracking approach, we also obtain the set \(D\subseteq E\) that achieves this cut cost. We conclude our computation of \(D_v^h\) by element-wisely taking the minimum-sized set over all computed arrays \(D_{H}\) for the possible assignments \(H\). 

    If \(h\) is the empty head, i.e., the edge above \(v\) is cut, then \(v\) is placed in a component that is either of size \(\smax\) or has a coloring corresponding to some head \(h'\). In the first case, we compute an array \(D_{\text{full}}\) in the same manner as described above by trying all suitable assignments \(H\) and employing the \joinroutine.
    In the second case, we simply take the already filled array \(D_v^{h'}\). 
    Note that in both cases we have to increment all values in the array by one to reflect cutting the edge above \(v\), except if \(v\) is the root vertex. 
    Also, we have to move the values in the arrays around, in order to reflect that the component containing \(v\) is no longer a head but with the edge above \(v\) cut should also be counted in the coloring of the splitting. Hence, the entry \(D_{\text{full}}[C]\) is actually stored at \(D_{\text{full}}[C_{-\text{full}}]\) with \(C_{-\text{full}}\) being the coloring \(C\) minus the coloring of a minimum-sized fair cluster. If no such entry \(D_{\text{full}}[C_{-\text{full}}]\) exists, we assume it to be \(\infty\). The same goes for accessing the arrays \(D_v^{h'}\) where we have to subtract the coloring \(h'\) from the index. Taking the element-wise minimum-sized element over the such modified arrays \(D_{\text{full}}\) and \(D_v^{h'}\) for all possibilities for \(h'\) yields \(D_v^\emptyset\). 

    By the correctness of the \joinroutine and as we try out all possibilities to build the specified heads and colorings at every vertex, we thus know that after completing the computation at the root \(r\) of \(T\), the array \(D_r^\emptyset\) contains for every possible coloring of a splitting of the tree the minimum cut cost to achieve that coloring.

    For each of the \(n\) vertices and the \(\svars\) possible heads, we call the \joinroutine at most \(n^{\smax-1}\cdot \svars^{\smax -1}\) many times.
    Each time, we call it with at most \(n\) arrays and, as by \autoref{lem:numberColoringsOfPartitions} there are 
    $\bigO(n^{\smax})$ possible colorings, all these arrays have that many elements. Hence, each subroutine call takes time in 
     \(\bigO(n\cdot \left(n^{\svars}\right)^2) = \bigO(n^{2\nwspace\svars+1})\),
     so the algorithm takes time in 
     $\bigO(n^{2\nwspace\svars+\smax+2}\cdot \svars^\smax)$, including an additional factor in $\bigO(n)$ to account for converting the indices for the \joinroutine.

    When the input graph is not a tree but a forest \(F\), we apply the dynamic program on every tree in the forest. Then, we additionally run the \joinroutine with the arrays for the \(\emptyset\)-head at the roots of all trees in the forest. The resulting array contains all minimum-cost solutions from all possible combinations from colorings of splittings from the individual trees and is returned as output.
    The one additional subroutine does not change the asymptotic running time.
\end{proof}

Because of \Cref{lem:smallClustersForest,lem:smallClustersBipartiteOneOne} it suffices to consider partitions as possible solutions that have at most \(c_i\) vertices of color \(i\) in each cluster, for all \(i\in [k]\). We hence apply \autoref{lem:splittingForests} on the forest \(F\) and set \(d_i = c_i\) for all \(i\in[k]\). This way, for every possible coloring of a splitting we find the minimum set of edges to obtain a splitting with that coloring.

\paragraph*{Assembling a fair clustering.} Let \(D\) be the array produced in the first phase, i.e., for every coloring \(C\) of a splitting, \(D[C]\) is a minimum-sized set of edges such that the connected components in \(F-D[C]\) induce a partition with coloring \(C\). In the second phase, we have to find the splitting that gives the minimum \cccost. 
 We do so by deciding for each splitting whether it is \emph{assemblable}, i.e., whether its clusters can be merged such that it becomes a fair solution with all clusters being no larger than \(\smax\).
Among these, we return the one with the minimum inter-cluster cost computed in the first phase.

This suffices because of the following reasons. 
First, note that deciding assemblability only depends on the coloring of the splitting so it does not hurt that in the first phase we tracked only all possible colorings of splittings and not all possible splittings themselves. 
Second, we do not have to consider further edge cuts in this phase: Assume we have a splitting \(S\) with coloring \(C_S\) and we would obtain a better cost by further cutting \(a\) edges in \(S\), obtaining another splitting \(S'\) of coloring \(C_{S'}\). However, as we filled the array \(D\) correctly, there is an entry \(D[C_{S'}]\) and \(|D[C_{S'}]|\le |D[C_S]|+a\). As we will consider this value in finding the minimum anyway, there is no need to think about cutting the splittings any further.
Third, the minimum inter-cluster cost yields the minimum \cccost by \autoref{lem:costByCuts}. When merging clusters, the inter-cluster cost computed in the first phase may decrease but not increase. If it decreases, we overestimate the cost. However, this case implies that there is an edge between the two clusters and as they are still of size at most \(\smax\) when merged, in the first phase we will also have found another splitting considering this case.  

We employ a dynamic program to decide the assemblability for all possible $\bigO(n^{\svars})$ colorings of splittings. Define the \emph{size} of a partition coloring to be the number of set colorings in that partition coloring (not necessarily the number of different set colorings). We decide assemblability for all possible colorings of splittings from smallest to largest.
Note that each such coloring is of size at least \(\frac{n}{\smax}\). If it is of size exactly \(\frac{n}{\smax}\), then all contained set colorings are of size \(\smax\), so this partition coloring is assemblable if and only if all set colorings are fair.
Now assume we have found all assemblable colorings of splittings of size exactly \(j\ge \frac{n}{\smax}\).
Assume a partition coloring \(C\) of size \(j+1\) is assemblable.
Then, at least two set colorings \(C_1,C_2\) from \(C\) are merged together. 
Hence, let \(C'\) be the partition coloring obtained by removing the set colorings \(C_1,C_2\) from \(C\) and adding the set coloring of the combined coloring of \(C_1\) and \(C_2\). Now, \(C'\) is of size \(j\) and is assemblable. 
Thus, every assemblable splitting with \(j+1\) components has an assemblable splitting with \(j\) components. The other way round, if we split a set coloring of an assemblable partition coloring of size \(j\) we obtain an assemblable partition coloring of size \(j+1\). 
Hence, we find all assemblable colorings of splittings of size \(j+1\) by for each assemblable partition coloring of size \(j\) (less than \(n^{\svars}\) many) trying each possible way to split one of its set colorings (less than \(i\cdot 2^\smax\) as there are \(j\) set colorings each of size at most \(\smax\)). Thus, to compute all assemblable colorings of splittings of size \(j+1\), we need time in $\bigO(n^{\svars}\cdot j\cdot 2^\smax)$, which implies a total time for the \(n-\frac{n}{\smax}\) iterations in the second phase in 
$\bigO(n^{\svars+2}\cdot 2^\smax)$. This is dominated by the running time of the first phase.
The complete algorithm hence runs in time in 
$\bigO(n^{2\svars+\smax+2}\cdot \svars^\smax)$, which implies \autoref{thm:forestByColorsAlgo}.

This gives an algorithm that solves \fcc on arbitrary forests. The running time however may be exponential in the number of vertices depending on the color ratio in the forest.

\subsection{Few Clusters}
\label{subsec:fewClustersAlgo}

The algorithm presented in the previous section runs in polynomial time if the colors in the graph are distributed in a way such that each cluster in a minimum-cost solution is of constant size. The worst running time is obtained when there are very large but few clusters. For this case, we offer another algorithm, which runs in polynomial time if the number of clusters is constant. However, it is limited to instances where the forest is colored in two colors in a ratio of \(1:c\) for some \(c\in\N\).

The algorithm uses a subroutine that computes the minimum number of cuts that are required to slice off clusters of specific sizes from the tree. It is given by \autoref{lem:treeCutOffCosts}.

\begin{lemma}
\label{lem:treeCutOffCosts}
    Let \(T=(V,E)\) be a tree rooted at \(r\in V\) and \(k\in\N\). Then, we can compute an array \(R\) such that, for each \(a_0\in [n]\) and \(a={a_1,a_2,\ldots,a_k}\in \left([n-1]\cup \{0\}\right)^k\) with \(a_i \ge a_{i+1}\) for \(i \in [k-1]\) and \(\sum_{i=0}^k a_i = n\), we have that \(R[a_0,a]\) is the partition \(\mcP = \{S_0,S_1,\ldots,S_{k}\}\) of \(V\) with minimum inter-cluster cost that satisfies \(r\in S_0\) and \(|S_i|=a_i\) for \(i\in [k]\). The computation time is in $\bigO((k+3)!\cdot n^{2k+3})$.
\end{lemma}

\begin{proof}
    We give a construction such that \(R[a_0,a]\) stores not the partition itself but the incurred inter-cluster cost. By a simple backtracking approach, the partitions are obtained as well.

    We employ a dynamic program that involves using the \joinroutine. 
    For the sake of readability, we index the arrays here by some vector \(a\in [n]^k\) and \(a_0\in [n]\) instead of a single numerical index as used in the algorithmic description of the \joinroutine. We implicitly assume that each possible \(a_0,a\) is represented by some index in \([n^{k+1}]\). By letting these indices enumerate the vectors in a structured way, converting between the two formats only costs an additional time factor in $\bigO(k)$.

    Starting at the leaves and continuing at the vertices for which all children have finished their computation, we compute an array \(R_v\) with the properties described for \(R\) but for the subtree \(T_{v}\) for each vertex \(v\in V\).
    In particular, for every vertex \(v\) we do the following. Let \(R_v^0\) be an array with \(\infty\)-values at all indices except for \(R_v^0[1, (0,0,\ldots,0)]=0\), as this is the only possible entry for the tree \(T[\{v\}]\).    

    If \(v\) has no children, then \(R=R_v^0\). 
    Otherwise, let the children of \(v\) be \(u_1,u_2,\ldots, u_\ell\).
    Then we call the \joinroutine with the arrays \(R_v^0, R_{u_1}, R_{u_2}, \ldots, R_{u_\ell}\).
    We have to define \(f\) such that it gives all possibilities to combine the children's subtrees partitions and \(v\). 
    For all possible values of \(a_0, a\) and \(a_0', a'\) recall that \(f((a_0,a),(a_0',a'))\) should return a set of indices of the form \((a_0'',a'')\).
    Each such index describes a combination of all possibilities for \(v\) and the already considered children (\(a_0,a\)) and the possibilities for the next child (\(a_0',a'\)).
    First, we consider the possibility to cut the edge between \(v\) and the child \(u\) that is represented by \((a_0', a'')\). Then, we add all possible ways of merging the two sets with their \(k+1\) clusters each. As we cut the edge \(\{u,v\}\), there are \(k\) possible ways to place the cluster containing \(u\) (all but the cluster containing \(v\)) and then there are \(k!\) ways to assign the remaining clusters. All these are put into the set \(f((a_0,a),(a_0',a'))\).
    Second, we assume the edge \(\{u,v\}\) is not cut. Then, the clusters containing \(v\) and \(u\) have to be merged, so there are only \(k!\) possible ways to assign the other clusters. In particular, for all indices \((a_0'',a'')\) put into \(f((a_0,a),(a_0',a'))\) this way, we have \(a_0''=a_0+a_0'\). Note that \(f\) can be computed in $\bigO(k\cdot k!)$. Note that \(\widehat{f}(x_1,x_2,\ldots,x_\ell)\) as defined in the \joinroutine lists all possibilities to cut the combined tree as it iteratively combines all possibilities for the first child and the vertex \(v\) and for the resulting tree lists all possible combinations with the next child and so on.
    The \joinroutine takes time in $\bigO((k+1)\cdot \left(n^{k+1}\right)^2 \cdot (k\cdot k!)\cdot k)$, which is in $\bigO((k+3)!\cdot n^{2k+2})$. All \(\bigO(n)\) calls of the subroutine hence take time in $\bigO((k+3)!\cdot n^{2k+3})$.
\end{proof}

With this, we are able to give an algorithm for graphs with two colors in a ratio of \(1:c\), which runs in polynomial time if there is only a constant number of clusters, i.e., if \(c\in \Theta(n)\).

\begin{theorem}
\label{thm:forestLarge1_CAlgo}
    Let \(F\) be an \(n\)-vertex forest with two colors in a ratio of \(1:c\) with \(c\in \N_{>0}\) and let \(p = \frac{n}{c+1}\). Then, \emph{\fcc} on \(F\) can be solved in $\bigO(n^{p^3+p^2+p})$.
\end{theorem}

\begin{proof}
    Note that, if there are \(c\) red vertices per 1 blue vertex, \(p = \frac{n}{c+1}\) is the number of blue vertices.
    By \autoref{lem:smallClustersForest}, any minimum-cost clustering consists of \(p\) clusters, each containing exactly one blue vertex, and from \autoref{lem:costByCuts} we know that it suffices to minimize the number of edges cut by any such clustering. 
    All blue vertices are to be placed in separate clusters. They are separated by cutting at most \(p-1\) edges, so we try all of the $\bigO((p-1)\cdot\binom{n-1}{p-1})$ subsets of edges of size at most \(p-1\). Having cut these edges, we have \(\ell\) trees \(T_1, T_2, \ldots, T_\ell\), with \(p\) of them containing exactly one blue vertex and the others no blue vertices. We root the trees at the blue vertex if they have one or at an arbitrary vertex otherwise. For each tree \(T_i\), let \(r_i\) be the number of red vertices. If we have exactly \(p\) trees and \(r_i = c\) for all \(i\in [p]\), we have found a minimum-cost \clust, where the \(i\)-th cluster is simply the set of vertices of \(T_i\) for all \(i\in[p]\). Otherwise, we must cut off parts of the trees and assign them to other clusters in order to make the partition fair. To this end, for each tree \(T_i\) we compute an array \(R_i\) that states the cost of cutting up to \(p-1\) parts of certain sizes off. More precisely, \(R_i[(a_1, a_2, \ldots, a_{p-1})]\) is the number of cuts required to cut off \(p-1\) clusters of size \(a_1, a_2, \ldots, a_{p-1}\), respectively, and \(\infty\) if there is no such way as \(\sum_{i=1}^{p-1}> r_i\). It suffices to compute \(R_i[(a_1, a_2, \ldots, a_{p-1})]\) with \(0\le a_i\le a_{i+1} \le n\) for \(i\in[p-2]\).
    
    We compute these arrays employing \autoref{lem:treeCutOffCosts}. Note that here we omitted the \(a_0\) used in the lemma, which here refers to the number of vertices \emph{not} cut from the tree. However, \(a_0\) is still unambiguously defined over \(a\) as all the values sum up to the number of vertices in this tree. Further, by connecting all trees without blue vertices to some newly added auxiliary vertex \(z\) and using this tree rooted at \(z\) as input to \autoref{lem:treeCutOffCosts}, we reduce the number of subroutine calls to \(p+1\). Then, the only entries from the array obtained for the all-red tree we consider are the ones with \(a_0=1\) as we do not want to merge \(z\) in a cluster but every vertex except \(z\) from this tree has to be merged into another cluster. We call the array obtained from this tree \(R_0\) and the arrays obtained for the other trees \(R_1, R_2,\ldots,R_p\), respectively.

    Note that every fair clustering is characterized by choosing one entry from each array \(R_i\) and assigning the cut-off parts to other clusters. 
    As each array has less than \(\frac{n^p}{p!}\) entries and there are at most \((p!)^{p}\) ways to assign the cut-off parts to clusters, there are at most \(n^{p^2}\) possibilities in total. For each of these, we compute in linear time whether they result in a fair clustering. Among these fair clusterings, we return the one with the minimum inter-cluster cost, computed by taking the sum over the chosen entries from the arrays \(R_i\). By \autoref{lem:costByCuts}, this clustering has the minimum \cccost.
    We obtain a total running time of 
	\begin{equation*}
	    \bigO((p-1)\cdot\binom{n-1}{p-1}\cdot \left((p+1)\cdot \left(n^{p+3}+n^{p^2+p-2}\right)+n^{p^2+1}\right)) \subseteq 
    \bigO(n^{p^3+p^2+p}). \qedhere
	\end{equation*}    
\end{proof}

Combining the results of \Cref{thm:forestByColorsAlgo,thm:forestLarge1_CAlgo}, we see that for the case of a forest with two colors in a ratio of \(1:c\) for some \(c\in\N_{>0}\), there are polynomial-time algorithms when the clusters are either of constant size or have sizes in \(\Theta(n)\). As \autoref{thm:forest_hard} states that \fcc on forests is \NP-hard, we hence know that this hardness evolves somewhere between the two extremes.

\section{Relaxed Fairness}
\label{sec:relaxed}
It might look like the hardness results for \fcc are due to the very strict definition of fairness, which enforces clusters of a specific size on forests. However, in this section, we prove that even when relaxing the fairness requirements our results essentially still hold.

\subsection{Definitions}
\label{subsec:relaxed_def}

We use the relaxed fairness constraint as proposed by Bera et al.~\cite{Bera_Chakrabarty_Flores_Negahbani_2019} and employed for \fcc by Ahmadi et al.~\cite{Ahmadi_Galhotra_Saha_Schwartz_2020}. For the following definitions, given a set \(U\) colored by a function \(c : U\rightarrow k\), by \(U_i=\{u\in U\mid c(u)=i\}\) we denote the set of vertices of color \(i\) for all \(i\in[k]\).

\begin{definition}[Relaxed Fair Set]
    Let \(U\) be a finite set of elements colored by a function \(c : U\rightarrow [k]\) for some \(k\in \N_{>0}\) and let \(p_i,q_i\in \Q\) with \(0<p_i\le \frac{|U_i|}{|U|} \le q_i < 1\) for all \(i\in [k]\). 
    Then, some \(S\subseteq U\) is relaxed fair with regard to the \(q_i\) and \(p_i\) if and only if for all colors \(i\in [k]\) we have \(p_i \le \frac{|S\cap U_i|}{|S|} \le q_i\).
\end{definition}
Note that we require \(p_i\) and \(q_i\) to be such that an exact fair solution is also relaxed fair. Further, we exclude setting \(p_i\) or \(q_i\) to 0 as this would allow clusters that do not include every color, which we do not consider fair. 

\begin{definition}[Relaxed Fair Partition]
    Let \(U\) be a finite set of elements colored by a function \(c : U\rightarrow [k]\) for some \(k\in \N_{>0}\) and let \(p_i,q_i\in \Q\) with \(0<p_i\le \frac{|U_i|}{|U|} \le q_i < 1\) for all \(i\in [k]\). Then, a partition \(S_1\cup S_2 \cup \ldots \cup S_\ell = U\) is relaxed fair with regard to the \(q_i\) and \(p_i\) if and only if all sets \(S_1, S_2, \ldots, S_\ell\) are relaxed fair with regard to the \(q_i\) and \(p_i\).
\end{definition}

\optPDef{\rfcc}
{Graph $G = (V, E)$, coloring $c\colon V\rightarrow [k]$, \(p_i,q_i\in \Q\) with \(0<p_i\le \frac{|U_i|}{|U|} \le q_i < 1\) for all \(i\in [k]\).}
{Find a relaxed fair partition $\mathcal{P}$ of \(V\) with regard to the \(p_i\) and \(q_i\) that minimizes \(\cost{\mcP}\).}

While we use the above definition for our hardness results, we restrict the possibilities for the \(p_i\) and \(q_i\) for our algorithms.

\begin{definition}[\(\alpha\)-relaxed Fair Set]
    Let \(U\) be a finite set of elements colored by a function \(c : U\rightarrow [k]\) for some \(k\in \N_{>0}\) and let \(0<\alpha<1\).
    Then, some \(S\subseteq U\) is \(\alpha\)-relaxed fair if and only if it is relaxed fair with regard to \(p_i = \frac{\alpha |U_i|}{|U|}\) and \(q_i = \frac{|U_i|}{\alpha|U|}\) for all \(i\in[k]\).
\end{definition}

\begin{definition}[\(\alpha\)-relaxed Fair Partition]
    Let \(U\) be a finite set of elements colored by a function \(c : U\rightarrow [k]\) for some \(k\in \N_{>0}\) and let \(0<\alpha<1\). 
    Then, a partition \(S_1\cup S_2 \cup \ldots \cup S_\ell = U\) is \(\alpha\)-relaxed fair if and only if all sets \(S_1, S_2, \ldots, S_\ell\) are \(\alpha\)-relaxed fair.
\end{definition}

\optPDef{\arfcc}
{Graph $G = (V, E)$, coloring $c\colon V\rightarrow [k]$, \(0<\alpha<1\).}
{Find a \(\alpha\)-relaxed fair partition $\mathcal{P}$ of \(V\) that minimizes \(\cost{\mcP}\).}

\subsection{Hardness for Relaxed Fairness}
\label{subsec:relaxed_hardness}

The hardness result for exact fairness on paths, see \autoref{thm:pathsNPhard}, directly carries over to the relaxed fairness setting. This is due to it only considering instances in which there are exactly two vertices of each color. As any relaxed fair clustering still requires at least one vertex of every color in each cluster, this means that every relaxed clustering either consists of a single cluster or two clusters, each with one vertex of every color. Thereby, relaxing fairness makes no difference in these instances.

\begin{corollary}
\label{cor:pathNPhard_relaxed}
    \emph{\rfcc} on paths is \emph{\NP}-hard, even when limited to instances with exactly 2 vertices of each color. 
\end{corollary}

Our other hardness proofs for relaxed fairness are based on the notion that we can use similar constructions as for exact fairness and additionally prove that in these instances the minimum-cost solution has to be exactly fair and not just relaxed fair. To this end, we require a lemma giving a lower bound on the intra-cluster cost of clusterings. 

\begin{lemma}
\label{lem:clustersCauchyLowerBound}
    Let \(G=(V,E)\) be an \(n\)-vertex \(m\)-edge graph and \(\mcP\) a partition of \(V\) with an inter-cluster cost of \(\chi\). Then, the intra-cluster cost of \(\mcP\) is at least \(\frac{n^2}{2|\mcP|}-\frac{n}{2}-m+\chi\). If \(|S|=\frac{n}{|\mcP|}\) for all clusters \(S\in \mcP\), then the intra-cluster cost of \(\mcP\) is exactly \(\psi = \frac{n^2}{2|\mcP|}-\frac{n}{2}-m+\chi\).
\end{lemma}

\begin{proof}
    We first prove the lower bound.
    We employ the Cauchy-Schwarz inequality,
    stating that for every \(\ell\in\N\), \(x_1,x_2,\ldots,x_\ell\), and \(y_1,y_2,\ldots,y_\ell\), we have
    $\left(\sum_{i=1}^\ell x_iy_i\right)^2\le \left(\sum_{i=1}^\ell x_i^2\right)\cdot \left(\sum_{i=1}^\ell y_i^2\right)$.
     In particular, it holds that
    $\left(\sum_{i=1}^\ell x_i\right)^2\le \ell\cdot \sum_{i=1}^\ell x_i^2$.
    Observe that we can write the intra-cluster cost \(\psi\) of \(\mcP\) as 
    \begin{align*}
        \psi &= \left(\sum_{S\in\mcP} \frac{|S|\cdot(|S|-1)}{2}\right)-(m-\chi)
        	= \frac{1}{2}\left(\sum_{S\in\mcP} |S|^2\right)-\left(\sum_{S\in\mcP}\frac{|S|}{2}\right)-m+\chi\\
        	&= \frac{1}{2}\left(\sum_{S\in\mcP} |S|^2\right)-\frac{n}{2}-m+\chi.
    \end{align*}
   By Cauchy-Schwarz, we have
    $\sum_{S\in\mcP} |S|^2 \ge \frac{1}{|\mcP|}\cdot \left(\sum_{S\in\mcP}|S|\right)^2 =  \frac{n^2}{|\mcP|}$.
   This bounds the intra-cluster cost from below by
   $\psi \ge \frac{n^2}{2|\mcP|}-\frac{n}{2}-m+\chi$.
   
   For the second statement, assume all clusters of \(\mcP\) to be of size \(\frac{n}{|\mcP|}\). Then, there are \(\frac{1}{2}\cdot \frac{n}{|\mcP|}\cdot \left(\frac{n}{|\mcP|}-1\right)\) pairs of vertices in each cluster. Thereby, we have
   \begin{equation*}
    \psi = |\mcP|\cdot \frac{1}{2}\cdot \frac{n}{|\mcP|}\cdot \left(\frac{n}{|\mcP|}-1\right)-(m-\chi)
    	= \frac{n^2}{2|\mcP|}-\frac{n}{2}-m+\chi. \qedhere
  \end{equation*}
\end{proof}

We further show that no clustering with clusters of unequal size achieves the lower bound given by \autoref{lem:clustersCauchyLowerBound}.

\begin{lemma}
\label{lem:clustersCauchyUnequalIsWorse}
    Let \(G=(V,E)\) be an \(n\)-vertex \(m\)-edge graph and \(\mcP\) a partition of \(V\) with an inter-cluster cost of \(\chi\) such that there is a cluster \(S\in \mcP\) with \(|S|=\frac{n}{|\mcP|}+a\) for some \(a\ge 0\).
    Then, the intra-cluster cost of \(\mcP\) is \(\psi \ge \frac{a^2|\mcP|}{2|\mcP|-2}+ \frac{n^2}{2|\mcP|}-\frac{n}{2}-m+\chi\).   
\end{lemma}
\begin{proof}
    If \(a=0\), the statement is implied by \autoref{lem:clustersCauchyLowerBound}. So, assume \(a>0\).
    We write the intra-cluster cost as    
   \begin{align*}
    \psi = \frac{1}{2}\cdot\left(\frac{n}{|\mcP|}+a\right)\cdot\left(\frac{n}{|\mcP|}+a-1\right)+ \psi_{\text{rest}}
\end{align*}
with \(\psi_{\text{rest}}\) being the intra-cluster cost incurred by \(\mcP \setminus \{S\}\). By applying \autoref{lem:clustersCauchyLowerBound} on \(\mcP \setminus \{S\}\), we have
\begin{align*}
    \psi &\ge \frac{1}{2}\cdot\left(\frac{n}{|\mcP|}+a\right)\cdot\left(\frac{n}{|\mcP|}+a-1\right)+ \frac{\left(n-(\frac{n}{|\mcP|}+a)\right)^2}{2(|\mcP|-1)}-\frac{n-(\frac{n}{|\mcP|}+a)}{2}-m+\chi\\
    &=\frac{n^2}{2|\mcP|^2}+\frac{an}{|\mcP|}+\frac{a^2}{2}-\frac{n}{2|\mcP|}-\frac{a}{2}+ \frac{n^2-2n^2/|\mcP|-2an+n^2/|\mcP|^2+2a\frac{n}{|\mcP|}+a^2}{2|\mcP|-2}\\
    &\hphantom{=}-\frac{n}{2}+\frac{n}{2|\mcP|}+\frac{a}{2}-m+\chi.
\end{align*}
Bringing the first summands to a common denominator of \(2|\mcP|-2\) yields 
\begin{align*}
    \psi &\ge \left(\frac{n^2(|\mcP|-1)}{|\mcP|^2} + \frac{an(2|\mcP|-2)}{|\mcP|} + a^2(|\mcP|-1)+n^2-\frac{2n^2}{|\mcP|}-2an+\frac{n^2}{|\mcP|^2}+\frac{2an}{|\mcP|}+a^2\right)\\
    &\hphantom{\ge}\hphantom{ } \big/ (2|\mcP|-2)-\frac{n}{2}-m+\chi\\
    &=\left(\frac{n^2|\mcP|}{|\mcP|^2} + \frac{2an|\mcP|}{|\mcP|} + a^2|\mcP|+n^2-\frac{2n^2}{|\mcP|}-2an\right) \big/ (2|\mcP|-2)-\frac{n}{2}-m+\chi\\
    &=\left(-\frac{n^2}{|\mcP|} + a^2|\mcP|+n^2\right) \big/ (2|\mcP|-2)-\frac{n}{2}-m+\chi.
\end{align*}
We then add \(0=-\frac{n^2}{2|\mcP|}\cdot \frac{2|P|-2}{2|P|-2}+\frac{n^2}{2|\mcP|}\) and obtain
\begin{align*}
    \psi&\ge \left(-\frac{n^2}{|\mcP|} + a^2|\mcP|+n^2-\frac{n^2(|\mcP|-1)}{|\mcP|}\right) \big/ (2|\mcP|-2)+\frac{n^2}{2|\mcP|}-\frac{n}{2}-m+\chi\\
    &=\frac{a^2|\mcP|}{2|\mcP|-2}+\frac{n^2}{2|\mcP|}-\frac{n}{2}-m+\chi. \qedhere
\end{align*}
\end{proof}

Observe that as \(|\mcP|>1\) and \(a\neq 0\) this means that such a clustering never achieves the lower bound given by \autoref{lem:clustersCauchyLowerBound}. 
In particular, this means that for fixed inter-cluster costs in minimum-cost \fcs in forests all clusters are of equal size. This way, we are able to transfer some hardness results obtained for exact fairness to relaxed fairness.

\begin{theorem}
\label{thm:forestHard_relaxed}
    For every choice of \(0<p_1\le \frac{1}{c+1}\le  q_1<1\) and \(0<p_2\le \frac{c}{c+1} \le q_2<1\), \emph{\rfcc} on forests with two colors in a ratio of \(1:c\) is \emph{\NP}-hard. It remains \emph{\NP}-hard when arbitrarily restricting the shape of the trees in the forest as long as for every \(a\in\N\) it is possible to form a tree with \(a\) vertices.
\end{theorem}

\begin{proof}
    We reduce from \thrPart. Recall that there are \(3p\) values \(a_1,a_2,\ldots,a_{3p}\) and the task is to partition them in triplets that each sum to \(B\). 
    We construct a forest \(F\) as follows. For every \(a_i\) we construct an arbitrary tree of \(a_i\) red vertices. Further, we let there be \(p\) isolated blue vertices. Note that the ratio between blue and red vertices is \(1:B\). We now show that there is a \rfc \(\mcP\) such that 
    \begin{equation*}
        \cost{\mcP} \le p\cdot\frac{B(B+1)}{2}-p(B-3)
    \end{equation*}
     if and only if the given instance is a yes-instance for \thrPart.

    If we have a yes-instance of \thrPart, then there is a partition of the set of trees into \(p\) clusters of size \(B\). By assigning the blue vertices arbitrarily to one unique cluster each, we hence obtain an exactly fair partition, which is thus also relaxed fair. As there are no edges between the clusters and each cluster consists of \(B+1\) vertices and \(B-3\) edges, this partition has a cost of \(p\cdot\frac{B(B+1)}{2}-p(B-3)\).

    For the other direction, assume there is a \rfc \(\mcP\) such that  \(\cost{\mcP} \le p\cdot\frac{B(B+1)}{2}-p(B-3)\). We prove that this clustering has to be not just relaxed fair but exactly fair.
    Note that \(|V|=p(B+1)\) and \(|E|=p(B-3)\). As the inter-cluster cost \(\chi\) is non-negative, by \autoref{lem:clustersCauchyLowerBound} the intra-cluster cost has a lower bound of 
    \begin{align*}
        \psi \ge \frac{(p(B+1))^2}{2|\mcP|}-\frac{p(B+1)}{2}-p(B-3).
    \end{align*}
    As there are exactly \(p\) blue vertices and the relaxed fairness constraint requires putting at least one blue vertex in each cluster, we have \(|\mcP|\le p\). Hence, 
    \begin{equation*}
          \psi \ge \frac{p(B+1)^2}{2}-\frac{p(B+1)}{2}-p(B-3)
         	= p\cdot\frac{B(B+1)}{2}-p(B-3)
         	\ge \cost{\mcP}.
\end{equation*}
This implies that the inter-cluster cost of \(\mcP\) is 0 and \(|\mcP|=p\).  \autoref{lem:clustersCauchyUnequalIsWorse} then gives that all clusters in \(\mcP\) consist of exactly \(B+1\) vertices. As each of the \(p\) clusters has at least 1 blue vertex and there are \(p\) blue vertices in total, we know that each cluster consists of 1 blue and \(B\) red vertices.  
Since all trees are of size greater than \(\frac{B}{4}\) and less than \(\frac{B}{2}\), this implies each cluster consists of exactly one blue vertex and exactly three uncut trees with a total of \(B\) vertices. This way, such a clustering gives a solution to \thrPart, so our instance is a yes-instance.  

As the construction of the graph only takes polynomial time in the instance size, this implies our hardness result.
\end{proof}

Indeed, we note that we obtain our hardness result for any fairness constraint that allows the exactly fair solution and enforces at least 1 vertex of each color in every cluster. The same holds when transferring our hardness proof for trees of diameter 4.

\begin{theorem}
\label{thm:tree_relaxed_hard}
    For every choice of \(0<p_1\le \frac{1}{c+1}\le  q_1<1\) and \(0<p_2\le \frac{c}{c+1} \le q_2<1\), \emph{\rfcc} on trees with diameter 4 and two colors in a ratio of \(1:c\) is \NP-hard.
\end{theorem}
\begin{proof}
    We reduce from \thrPart. We assume \(B^2>16p\). We can do so as we obtain an equivalent instance of \thrPart when multiplying all \(a_i\) and \(B\) by the same factor, here some value in $\bigO(p)$.
    For every \(a_i\) we construct a star of \(a_i\) red vertices. Further, we let there be a star of \(p\) blue vertices. We obtain a tree of diameter 4 by connecting the center \(v\) of the blue star to all the centers of the red stars. Note that the ratio between blue and red vertices is \(1:B\). 
    We now show that there is a \rfc \(\mcP\) such that 
    \begin{align*}
        \cost{\mcP} \le \frac{pB^2-pB}{2}+7p-7
    \end{align*}
     if and only if the given instance is a yes-instance for \thrPart. 

    If we have a yes-instance of \thrPart, then there is a partition of the set of stars into \(p\) clusters of size \(B\), each consisting of three stars. By assigning the blue vertices arbitrarily to one unique cluster each, we hence obtain an exact fair partition, which is thus also relaxed fair.
    We first compute the inter-cluster cost. We call an edge \emph{blue} or \emph{red} if it connects two blue or red vertices, respectively. We call an edge \emph{blue-red} if it connects a blue and a red vertex. All \(p-1\) blue edges are cut. Further, all edges between \(v\) (the center of the blue star) and red vertices are cut except for the three stars to which \(v\) is assigned. This causes \(3p-3\) more cuts, so the inter-cluster cost is $\chi = 4p-4$.
    Each cluster consists of \(B+1\) vertices and \(B-3\) edges, except for the one containing \(v\) which has \(B\) edges. The intra-cluster cost is
    \begin{equation*}
        \psi = p \left(\frac{B(B+1)}{2}-B+3\right)-3 = \frac{pB^2-pB}{2}+3p-3.
    \end{equation*}
     Combining the intra- and inter-cluster costs yields the desired cost of 
    \begin{equation*}
        \cost{\mcP} = \chi + \psi
        	=\frac{pB^2-pB}{2}+7p-7.
    \end{equation*}

    For the other direction, assume there is a \rfc \(\mcP\) such that \(\cost{\mcP} \le \frac{pB^2-pB}{2}+7p-7\).
    We prove that this clustering is not just relaxed fair but exactly fair. 
    
    To this end, we first show \(|\mcP|= p\). 
    Because each cluster requires one of the \(p\) blue vertices, we have \(|\mcP| \le p\).
    Now, let \(\chi\) denote the inter-cluster cost of \(\mcP\). Note that \(|V|=p(B+1)\) and \(|E|=p(B-3)+3p+p-1=p(B+1)-1\). Then, by \autoref{lem:clustersCauchyLowerBound}, we have
    \begin{align}
         \psi &\ge\frac{\left(p(B+1)\right)^2}{2|\mcP|}-\frac{p(B+1)}{2}-\left(p(B+1)-1\right)+\chi\nonumber\\
        &= \frac{p^2B^2+2p^2B+p^2}{2|\mcP|}-\frac{3p(B+1)}{2}+1+\chi.\label{eq:relTreeHardness_lowerBound}
    \end{align}
    Note that the lower bound is decreasing in \(|\mcP|\). If we had \(|\mcP|\le p-1\), then 
    \begin{align*}
       \psi \ge \frac{p^2B^2+2p^2B+p^2}{2(p-1)}-\frac{3p(B+1)}{2}+1+\chi.
    \end{align*}
    As the inter-cluster cost \(\chi\) is non-negative, we would thereby get 
    \begin{align*}
        \cost{\mcP} &\ge \frac{p^2B^2+2p^2B+p^2}{2(p-1)}-\frac{3p(B+1)}{2}+1+\chi\\
        &\ge \frac{p^2B^2+2p^2B+p^2}{2(p-1)}-\frac{3p^2B-3pB+3p^2-3p}{2(p-1)}+\frac{2p-2}{2(p-1)}\\
        &\ge \frac{p^2B^2-p^2B-2p^2+3pB+5p-2}{2(p-1)}.
    \end{align*}
    However, we know
    \begin{align*}
    \cost{\mcP} &\le \frac{pB^2-pB}{2}+7p-7\\
        &= \frac{p^2B^2-pB^2-p^2B+pB+14p^2-14p-14p+14}{2(p-1)}\\
        &= \frac{p^2B^2-pB^2-p^2B+pB+14p^2-28p+14}{2(p-1)}.
    \end{align*}
    
    Hence, \(|\mcP|\le p-1\) holds only if 
    $-2p^2+3pB+5p-2 \le -pB^2 +pB + 14p^2 -28p +14$
     which is equivalent to 
    $pB^2-16p^2+2pB+33p-16 \le 0$.
    As we assume \(B^2 > 16p\), this is always false, so \(|\mcP|=p\).
    Plugging this into \autoref{eq:relTreeHardness_lowerBound} yields
    \begin{equation*}
        \psi \ge \frac{pB^2+2pB+p}{2}-\frac{3p(B+1)}{2}+1+\chi
     		= \frac{pB^2-pB}{2}-p+1+\chi.
   \end{equation*}
    As \(\cost{\mcP} = \chi + \psi\), we have 
    \begin{align}
        \frac{pB^2-pB}{2}-p+1+2\chi \le \cost{\mcP}\le \frac{pB^2-pB}{2}+7p-7,\label{eq:relTreeHardness_lowerBoundMatch}
    \end{align}
    which yields $\chi \le 4p-4$.
    
     As no two blue vertices are placed in the same cluster, the cuts between blue vertices incur an inter-cluster cost of exactly \(p-1\).
    To estimate the number of cut blue-red edges, let \(a\) denote the number of red center vertices placed in the cluster of the blue center vertex \(v\). Then, there are \(3p-a\) of the \(3p\) red edges cut.
    Let \(\chi_r\) denote the number of cut red edges. Note that
    $\chi = p-1 + 3p-a + \chi_r = 4p -a -1 + \chi_r$.
   
    We prove that \(a=3\). As \(\chi \le 4p-4\) we have \(\chi_r - a \le -3\), whence \(a\ge 3\). 
    Next, we bound \(\chi_r\) by \(a\). Let \(\delta\in\Z\) be such that \(B+\delta\) is the number of red vertices in the cluster containing the blue center vertex \(v\). Then,
    \begin{equation*}
        \chi_r \ge \frac{aB}{4}-(B+\delta-a)
        	= \frac{(a-4)B}{4}-\delta+a
    \end{equation*}
    as each red center vertex is connected to at least \(\frac{B}{4}\) red leaves but in the cluster of \(v\) there is only space for \(B+\delta-a\) of them. First, assume \(\delta\le 0\). This implies 
    $\chi_r - a \ge \frac{(a-4)B}{4}$.
    As we required \(\chi_r -a \le -3\), this gives \(a < 4\), as desired.
    
    The case \(\delta \ge 1\) is a bit more involved.
    From \autoref{lem:clustersCauchyUnequalIsWorse},
    \(p=|\mcP|\), and \(m=n-1=p(B+1)-1\), we get
    \begin{equation*}
        \psi \ge \frac{\delta^2|\mcP|}{2|\mcP|-2}+\frac{\left(p(B+1)\right)^2}{2|\mcP|}-\frac{p(B+1)}{2}-m+\chi
        =\frac{\delta^2p}{2p-2}+\frac{pB^2+2pB+p}{2}-\frac{3p(B+1)}{2}+ \chi + 1.
    \end{equation*}
    This yields 
    \begin{align*}
        \frac{\delta^2p}{2p-2}+\frac{pB^2-pB}{2}-p+2\chi +1 \le \cost{\mcP}\le \frac{pB^2-pB}{2}+7p-7.
    \end{align*}
    We derive from this inequality that \(\chi \le 4p-4 - \frac{\delta^2p}{4p-4}\) and 
    $\chi_r - a \le -3 -\frac{\delta^2p}{4p-4}$
    implying
    \begin{equation*}
        \frac{(a-4)B}{4}-\delta \le -3 -\frac{\delta^2p}{4p-4}
    \end{equation*}
    The right-hand side is decreasing in \(\delta\),
    and by plugging in the minimum value for the case \(\delta \ge 1\), we finally get
    $\frac{(a-4)B}{4} \le -2 -\frac{p}{4p-4}$.
    This shows that \(a < 4\) must hold here as well.
    
    Thus, we have proven \(a=3\), which also gives \(\chi_r = 0\) and \(\chi = 4p-4\). 
    So, not only do we have that \(\cost{\mcP}\le \frac{pB^2-pB}{2}+7p-7\) but \(\cost{\mcP}= \frac{pB^2-pB}{2}+7p-7\). 
    In \autoref{eq:relTreeHardness_lowerBoundMatch} we see that for \(\chi = 4p-4\) this hits exactly the lower bound established by \autoref{lem:clustersCauchyLowerBound}. Hence, by \autoref{lem:clustersCauchyUnequalIsWorse}, this implies that all clusters consist of exactly 1 blue and \(B\) red vertices and the clustering is exactly fair.
    
    As \(\chi_r = 0\), all red stars are complete. Given that every red star is of size at least \(\frac{B}{4}\) and at most \(\frac{B}{2}\), this means each cluster consists of exactly three complete red stars with a total number of \(B\) red vertices each and hence yields a solution to the \thrPart instance.
    As the construction of the graph only takes polynomial time in the instance size and the constructed tree is of diameter 4, this implies our hardness result.  
\end{proof}

In the hardness proofs in this section, we argued that for the constructed instances clusterings that are relaxed fair, but not exactly fair would have a higher cost than exactly fair ones. However, this is not generally true. It does not even hold when limited to paths and two colors in a \(1:1\) ratio, as illustrated in \autoref{fig:advantageRelaxedFairness}.

\begin{figure}
    \begin{center}
        \includegraphics[width=.5\textwidth]{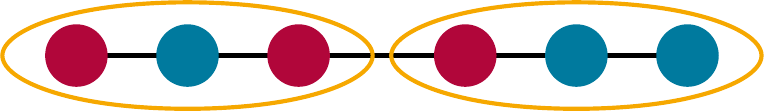}
    \end{center}
    \caption{Exemplary path with a color ratio of \(1:1\) where there is a \(\frac{2}{3}\)-\rfc of cost 3 (marked by the orange lines) and the cheapest exactly \fc costs 4.}
    \label{fig:advantageRelaxedFairness}
\end{figure}

Because of this, we have little hope to provide a general scheme that transforms all our hardness proofs from \autoref{sec:hardness} to the relaxed fairness setting at once. Thus, we have to individually prove the hardness results in this setting as done for \Cref{thm:forestHard_relaxed,thm:tree_relaxed_hard}. We are optimistic that the other hardness results still hold in this setting, especially as the construction for \autoref{thm:treeDeg5NPhard} is similar to the ones employed in this section. We leave the task of transferring these results to future work.

\subsection{Algorithms for Relaxed Fairness}
\label{sec:relaxed_algorithms}
 
We are also able to transfer the algorithmic result of \autoref{thm:forestByColorsAlgo} to a specific \(\alpha\)-relaxed fairness setting. We exploit that the algorithm does not really depend on exact fairness but on the fact that there is an upper bound on the cluster size, which allows us to compute respective splittings. In the following, we show that such upper bounds also exist for \(\alpha\)-relaxed fairness with two colors in a ratio of \(1:1\) and adapt the algorithm accordingly.
To compute the upper bound, we first prove \autoref{lem:alwaysGoodSplit}, which analogously to \autoref{lem:smallClustersForest} bounds the size of clusters but in uncolored forests. Using this lemma, with \autoref{lem:smallClustersRelaxedOneOne}, we then prove an upper bound on the cluster size in minimum-cost \arfcs for forests with two colors in ratio \(1:1\).

\begin{lemma}
\label{lem:alwaysGoodSplit}
    Let \(F=(V,E)\) be an \(n\)-vertex \(m\)-edge forest and let \(\mcP_1= \{V\}\). Further, let \(S\subset V\) with \(4 < |S| \le n-3\) and let \(\mcP_2 = \{S, V\setminus S\}\). Then, \(\cost{\mcP_1} > \cost{\mcP_2}\). 
\end{lemma}

\begin{proof}
    We have \(\cost{\mcP_1}=\frac{n(n-1)}{2} - m\) as there are \(\frac{n(n-1)}{2}\) pairs of vertices and \(m\) edges, none of which is cut by \(\mcP_1\).
    In the worst case, \(\mcP_2\) cuts all of the at most \(n-1\) edges in the forest. It has one cluster of size \(|S|\) and one of size \(n-|S|\), so 
    \begin{align*}
        \cost{\mcP_2}&\le n-1 + \frac{(n-|S|)(n-|S|-1)}{2} + \frac{|S|(|S|-1)}{2} - (m - n-1)\\   
        &= \frac{n(n-1)}{2} + \frac{-2n|S|+|S|^2+|S|}{2} + \frac{|S|^2-|S|}{2} - m +2n -2\\
        &= \frac{n(n-1)}{2} - n|S|+|S|^2 - m +2n -2.
    \end{align*}
    Then, we have 
    \begin{align*}
        \cost{\mcP_1}-\cost{\mcP_2} &\ge n|S|-|S|^2 -2n +2 \ge (|S|-2)n-|S|^2+2.
    \end{align*}
    Note that the bound is increasing in \(n\). As we have, \(n \ge |S|+3\) and \(|S|>4\), this gives 
    \begin{equation*}
        \cost{\mcP_1}-\cost{\mcP_2} \ge (|S|-2)(|S|+3)-|S|^2+2 = |S|-4 > 0. \qedhere
    \end{equation*}
\end{proof}

With the knowledge of when it is cheaper to split a cluster, we now prove that also for \arfcc there is an upper bound on the cluster size in minimum-cost solutions in forests. The idea is to assume a cluster of a certain size and then argue that we can split it in a way that reduces the cost and keeps \(\alpha\)-relaxed fairness.

\begin{lemma}
\label{lem:smallClustersRelaxedOneOne}
    Let \(F\) be a forest with two colors in a ratio of \(1:1\). Let \(0 <\alpha < 1\) and let \(\hat{\alpha}\in\N\) be minimal such that \(\frac{2\hat{\alpha}}{\alpha}\in\N\) and \(\frac{2\hat{\alpha}}{\alpha} > 4\). Then, if \(\mcP\) is a minimum-cost \arfc on \(F\), we have \(|S| < 4\frac{\hat{\alpha}}{\alpha^2}\) for all \(S\in \mcP\).   
\end{lemma}

\begin{proof}
    Assume otherwise, i.e., there is a cluster \(S\) with \(|S| \ge 4\frac{\hat{\alpha}}{\alpha^2}\). 
    Let \(b\) and \(r\) denote the number of blue and red vertices in \(S\), respectively, and assume w.l.o.g. that \(b \le r\). 
    Because \(|S|\ge 4\frac{\hat{\alpha}}{\alpha^2}\) we have \(\frac{\alpha}{2}\ge \frac{2\hat{\alpha}}{\alpha |S|}\).
    Due to the \(\alpha\)-relaxed fairness constraint, this yields \(\frac{b}{|S|}\ge \frac{2\hat{\alpha}}{\alpha |S|}\) and thereby \(r \ge b\ge \frac{2\hat{\alpha}}{\alpha}\).

    Then, consider the clustering obtained by splitting off \(\hat{\alpha}\) blue and \(\frac{2\hat{\alpha}}{\alpha}-\hat{\alpha}\) red vertices of from \(S\) into a new cluster \(S_1\) and let \(S_2 = S\setminus S_1\). Note that we choose \(\hat{\alpha}\) in a way that this is possible, i.e., that both sizes are natural numbers.
    As the cost induced by all edges with at most one endpoint in \(S\) remains the same and the cost induced by the edges with both endpoints in \(S\) decreases, as shown in \autoref{lem:alwaysGoodSplit}, the new clustering is cheaper than \(\mcP\). 
    As we now prove that the new clustering is also \(\alpha\)-relaxed Fair, this contradicts the optimality of \(\mcP\).

    We first prove the \(\alpha\)-relaxed fairness of \(S_1\).
    Regarding the blue vertices, we have a portion of \(\frac{\hat{\alpha}}{\hat{\alpha}+\frac{2\hat{\alpha}}{\alpha}-\hat{\alpha}} = \frac{\alpha}{2}\) in \(S_1\), which fits the \(\alpha\)-relaxed fairness constraint.
    Regarding the red vertices, we have \(\frac{\frac{2\hat{\alpha}}{\alpha}-\hat{\alpha}}{\hat{\alpha}+\frac{2\hat{\alpha}}{\alpha}-\hat{\alpha}} = 1-\frac{\alpha}{2}\), which fits the \(\alpha\)-relaxed fairness constraint as \(0<\alpha<1\), so \(1-\frac{\alpha}{2} \ge \frac{\alpha}{2}\) and \(1-\frac{\alpha}{2}=\frac{2\alpha-\alpha^2}{2\alpha}\le \frac{1}{2\alpha}\).

    Now we prove the \(\alpha\)-relaxed fairness of \(S_2\). 
    The portion of blue vertices in \(S_2\) is \(\frac{b-\hat{\alpha}}{r+b-\frac{2\hat{\alpha}}{\alpha}}\), so we have to show that this value lays between \(\frac{\alpha}{2}\) and \(\frac{1}{2\alpha}\). 
    We start with showing the value is at least \(\frac{\alpha}{2}\) by proving \(\frac{\alpha}{2}\cdot \left(r+b-\frac{2\hat{\alpha}}{\alpha}\right) \le b-\hat{\alpha}\).
    As \(S\) is \(\alpha\)-relaxed fair, we have \(r\le \frac{2b}{\alpha}-b\) because otherwise \(\frac{b}{b+r} < \frac{b}{b+\frac{2b}{\alpha}-b} = \frac{\alpha}{2}\).
    Hence, we have 
    \begin{align*}
        \frac{\alpha}{2}\cdot \left(r+b-\frac{2\hat{\alpha}}{\alpha}\right)
        \le \frac{\alpha}{2}\cdot \left(\frac{2b}{\alpha}-b+b-\frac{2\hat{\alpha}}{\alpha}\right)
        = b-\hat{\alpha}.
    \end{align*}
    Similarly, we show the ratio is at most \(\frac{1}{2\alpha}\) by proving the equivalent statement of \(2\alpha (b-\hat{\alpha}) \le r+b-\frac{2\hat{\alpha}}{\alpha}\). As we assume \(r\ge b\), we have
    \begin{align*}
        r+b-\frac{2\hat{\alpha}}{\alpha}
        \ge 2b - \frac{2\hat{\alpha}}{\alpha} 
        \ge 2\left(b - \frac{\hat{\alpha}}{\alpha} - \left((1-\alpha)b + (\alpha^2-1)\frac{\hat{\alpha}}{\alpha}\right)\right)
        = 2\alpha\left(b - \hat{\alpha}\right).
    \end{align*}
    The second step holds because we assumed \(b\ge \frac{2\hat{\alpha}}{\alpha} \ge \frac{\alpha\hat{\alpha}+\hat{\alpha}}{\alpha} = \frac{\frac{\hat{\alpha}}{\alpha}-\alpha\hat{\alpha}}{1-\alpha}\), so we have
    \((1-\alpha)b + (\alpha^2-1)\frac{\hat{\alpha}}{\alpha} \ge 0\). 
    Now, we regard the portion of red vertices in \(S_2\), which is \(\frac{r-\left(\frac{2\hat{\alpha}}{\alpha}-\hat{\alpha}\right)}{r+b-\frac{2\hat{\alpha}}{\alpha}}\). 
    We know that  $r \ge \frac{2\hat{\alpha}}{\alpha}$, that is, 
    $(1-\alpha)r \ge \frac{2\hat{\alpha}}{\alpha}-2\hat{\alpha}$ or, in other words,
    $r-\left(\frac{2\hat{\alpha}}{\alpha}+\hat{\alpha}\right) \ge \alpha r - \hat{\alpha}$.
   	As \(r\ge b\), this implies
   	\begin{equation*}
   		r-\left(\frac{2\hat{\alpha}}{\alpha}+\hat{\alpha}\right) \ge \frac{\alpha}{2} \cdot \left(r+b - \frac{2\hat{\alpha}}{\alpha}\right)
   	\end{equation*}
    and therefore
    $\frac{r-\left(\frac{2\hat{\alpha}}{\alpha}-\hat{\alpha}\right)}{r+b-\frac{2\hat{\alpha}}{\alpha}}  \ge  \frac{\alpha}{2}$.
    
    It remains to prove that this ratio is also at most \(\frac{1}{2\alpha}\).
    We have  $r \ge  \frac{2\hat{\alpha}}{\alpha}-\hat{\alpha}$, which is equivalent to
    \begin{equation*}
        \left(2\alpha -1 - \frac{\alpha}{2-\alpha}\right)r \le 4\hat{\alpha} - 2\alpha\hat{\alpha}-\frac{2\hat{\alpha}}{\alpha}.
    \end{equation*}
     Note that \(2\alpha-1-\frac{\alpha}{2-\alpha} = - \frac{2\alpha^2-4\alpha+2}{2-\alpha} = - \frac{2(\alpha-1)^2}{2-\alpha}< 0\) and that \(r\le \frac{2b}{\alpha}-b\) gives \(b \ge \frac{r}{\frac{2}{\alpha}-1}=\frac{\alpha r}{2-\alpha}\). With this, the above inequality implies
    \begin{equation*}
        (2\alpha -1)r - b \le 4\hat{\alpha} - 2\alpha\hat{\alpha}-\frac{2\hat{\alpha}}{\alpha}
    \end{equation*}
    From this, we finally arrive at
        $2\alpha\cdot \left(r-\left(\frac{2\hat{\alpha}}{\alpha}-\hat{\alpha}\right)\right) \le  r+b-\frac{2\hat{\alpha}}{\alpha}$,
    that is,
    $\frac{r-\left(\frac{2\hat{\alpha}}{\alpha}-\hat{\alpha}\right)}{r+b-\frac{2\hat{\alpha}}{\alpha}} \le \frac{1}{2\alpha}$.
    
   This proves that both \(S_1\) and \(S_2\) are \(\alpha\)-relaxed fair. As splitting \(S\) into \(S_1\) and \(S_2\) remains \(\alpha\)-relaxed fair and is cheaper, this contradicts \(S\) being in a minimum-cost \(\alpha\)-relaxed fair clustering.
\end{proof}

We are now able to adapt the algorithm presented in \autoref{subsec:forestTracing} to solve \rfcc on forests with two colors in a ratio of \(1:1\). While the original algorithm exploited that any optimum solution has fair clusters of minimum size, with \autoref{lem:smallClustersRelaxedOneOne} we are able to bound the clusters also in the \(\alpha\)-relaxed setting. 

Like the original algorithm, we first create a list of possible splittings. However, these splittings can contain not only components with one or two vertices, as we know would suffice for the exact fairness with two colors in a \(1:1\) ratio, but each component may contain up to \(4\frac{\hat{\alpha}}{\alpha^2}\) vertices with \(\hat{\alpha}\) being the smallest natural number such that \(\frac{2\hat{\alpha}}{\alpha}\in\N\) and \(\frac{2\hat{\alpha}}{\alpha} > 4\) as defined in \autoref{lem:smallClustersRelaxedOneOne}. 
In the following, we set \(d=4\frac{\hat{\alpha}}{\alpha^2}\) to refer to this maximum size of a cluster.
In the second phase, it checks which of these splitting can be merged into an \(\alpha\)-relaxed fair clustering and among these returns the one of minimum cost. 

\paragraph*{Splitting the forest.}
To get the optimal way to obtain a splitting of each possible coloring, we simply apply \autoref{lem:splittingForests} and set \(d_1 = d_2 = d\) as we know the optimum solution has to be among clusters with no more than \(d\) vertices of either color. This phase takes time in 
\(\bigO(n^{2(d+1)^2+2d+2}\cdot \left((d+1)^2\right)^{2d}) 
= \bigO(n^{2d^2+6d+4}\cdot (d+1)^{4d})\).  

\paragraph*{Assembling a fair clustering.} In the second phase, we have to find a splitting in \(D_r^\emptyset\) that can be transformed into an \(\alpha\)-relaxed fair clustering and yields the minimum \cccost. As we tracked the minimum inter-cluster cost for each possible partition coloring of splittings in the first phase, we do not have to consider cutting more edges in this phase, because for the resulting splittings coloring we already have tracked a minimum inter-cluster cost. Hence, the only questions are whether a splitting is \emph{assemblable}, i.e., whether its components can be merged such that it becomes an \(\alpha\)-relaxed fair clustering, and, if so, what the cheapest way to do so is. 

Regarding the first question, observe that the \emph{assemblability} only depends on the partition coloring of the splitting. Hence, it does not hurt that in the first phase we tracked only all possible partition colorings of splittings and not all possible splittings themselves. 
First, note that the coloring of a splitting may itself yield an \(\alpha\)-relaxed fair clustering. We mark all such partition colorings as assemblable, taking time in $\bigO(n^{d^2+1})$. For the remaining partition colorings, we employ the following dynamic program. 

Recall that the size of a partition coloring refers to the number of set colorings it contains (not necessarily the number of different set colorings). We decide assemblability for all possible partition colorings from smallest to largest.  
Note that each partition coloring is of size at least \(\lceil \frac{n}{d} \rceil\). If it is of size exactly \(\lceil\frac{n}{d}\rceil\), then there are no two set colorings that can be merged and still be of size at most \(d\), as all other set colorings are of size at most \(d\). Hence, in this case, a splitting is assemblable if and only if it is already an \(\alpha\)-relaxed fair clustering so we have already marked the partition colorings correctly.
Now, assume that we decided assemblability for all partition colorings of size \(i \ge \lceil \frac{n}{d}\rceil\). We take an arbitrary partition coloring \(C\) of size \(i+1\), which is not yet marked as assemblable. Then, it is assemblable if and only if at least two of its set colorings are merged together to form an \(\alpha\)-relaxed fair clustering. In particular, it is assemblable if and only if there are two set colorings \(C_1,C_2\) in \(C\) such that the coloring \(C'\) obtained by removing the set colorings \(C_1,C_2\) from \(C\) and adding the set coloring of the combined coloring of \(C_1\) and \(C_2\) is assemblable. Note that \(C'\) is of size \(i\). Given all assemblable partition colorings of size \(i\), we therefore find all assemblable partition colorings of size \(i+1\) by for each partition coloring of size \(i\) trying each possible way to split one of its set colorings into two. As there are at most \(i^{d^2}\) partition colorings of size \(i\), this takes time in \(\bigO(i^{d^2}\cdot i \cdot 2^d)\). The whole dynamic program then takes time in \(\bigO(n^{d^2+1}\cdot 2^d) \subseteq \bigO(n^{d^2+d+1})\).

It remains to answer how we choose the assembling yielding the minimum cost. In the algorithm for exact fairness, we do not have to worry about that as there we could assume that the \cccost only depends on the inter-cluster cost. Here, this is not the case as the \(\alpha\)-relaxed fairness allows clusters of varying size, so \autoref{lem:costByCuts} does not apply. However, recall that we can write the \cccost of some partition \(\mcP\) of the vertices as \(\sum_{S\in\mcP} \frac{|S|(|S-1|)}{2}+2\chi\), where \(\chi\) is the inter-cluster cost. The cost hence only depends on the inter-cluster cost and the sizes of the clusters, which in turn depends on the partition coloring. To compute the cost of a splitting, we take the inter-cluster cost computed in the first phase for \(\chi\).
Once more, we neglect decreasing inter-cluster cost due to the merging of clusters as the resulting splitting is also considered in the array produced in the first phase. By an argument based on the Cauchy-Schwarz Inequality, we see that merging clusters only increases the value of \(\sum_{S\in\mcP} \frac{|S|(|S-1|)}{2}\) as we have fewer but larger squares. Hence, the cheapest cost obtainable from a splitting which is itself \(\alpha\)-relaxed fair is just this very clustering. If a splitting is assemblable but not \(\alpha\)-relaxed fair itself, the sum is the minimum among all the values of the sums of \(\alpha\)-relaxed fair splittings it can be merged into. This value is easily computed by not only passing down assemblability but also the value of this sum in the dynamic program described above and taking the minimum if there are multiple options for a splitting. This does not change the running time asymptotically and the running time of the second phase is dominated by the one of the first phase.

The complete algorithm hence runs in time in 
$\bigO(n^{2d^2+6d+4}\cdot (d+1)^{4d})$.

\begin{theorem}
\label{thm:forestOneOneAlgoRelaxed}
    Let \(F\) be an \(n\)-vertex forest in which the vertices are colored with two colors in a ratio of \(1:1\).
    Then \emph{\arfcc} on \(F\) can be solved in time in $\bigO(n^{2d^2+6d+4}\cdot (d+1)^{4d})$, where \(d=4\frac{\hat{\alpha}}{\alpha^2}\) and \(\hat{\alpha}\in\N\) is minimal such that \(\frac{2\hat{\alpha}}{\alpha}\in\N\) and \(\frac{2\hat{\alpha}}{\alpha} > 4\).
\end{theorem}

We are confident that \autoref{lem:smallClustersRelaxedOneOne} can be generalized such that for an arbitrary number of colors in arbitrary ratios the maximum cluster size is bounded by some function in \(\alpha\) and the color ratio. Given the complexity of this lemma for the \(1:1\) case, we leave this task open to future work. If such a bound is proven, then the algorithmic approach employed in \autoref{thm:forestOneOneAlgoRelaxed} is applicable to arbitrarily colored forests. 
Similarly, bounds on the cluster size in the more general \rfcs can be proven. As an intermediate solution, we note that for \rfcc we can employ the approach used for \arfcc by setting \(\alpha\) large enough to contain all allowed solutions and filtering out solutions that do not match the relaxed fairness constraint in the assembling phase.
We do not give this procedure explicitly here as we suspect for these cases it is more promising to calculate the precise upper bound on the maximum cluster size and perform the algorithm accordingly instead of reducing to the \(\alpha\)-relaxed variant.

\section{Approximations}
\label{sec:approx}
So far, we have concentrated on finding an optimal solution to \fcc in various instances. Approximation algorithms that do not necessarily find an optimum but near-optimum solutions efficiently are often used as a remedy for hard problems, for example, the 2.06-approximation to (unfair) \cc \cite{Chawla_Makarychev_Schramm_Yaroslavtsev_2015}.
In this section, we find that just taking any fair clustering is a quite close approximation and the approximation becomes even closer to the optimum if the minimum size of any fair cluster, as given by the color ratio, increases.

Formally, a problem is an optimization problem if for every instance \(I\) there is a set of permissible solutions \(S(I)\) and an objective function \(m\colon S(I)\rightarrow \R_{>0}\) assigning a score to each solution. Then, some \(S\in S(I)\) is an optimal solution if it has the highest or lowest score among all permissible solutions, depending on the problem definition. We call the score of this solution \(m^*\!(I)\).
For example, for \fcc, the instance is given by a graph with colored vertices, every fair clustering of the vertices is a permissible solution, the score is the \cccost, and the objective is to minimize this cost.\footnote{We note that the clustering cost could be 0, which contradicts the definition \(m\colon S(I)\rightarrow \R_{>0}\). However, every 0-cost clustering simply consists of the connected components of the graph. We do not consider those trivial instances.}
An \(\alpha\)-approximation an optimization problem is an algorithm that, for each instance \(I\), outputs a permissible solution \(S\in S(I)\) such that \(\frac{1}{\alpha}\le \frac{m(S)}{m^*\!(I)} \le \alpha\). For \fcc in particular, this means the algorithm outputs a fair clustering with a cost of at most \(\alpha\) times the minimum clustering cost. 

\APX is the class of problems that admit an \(\alpha\)-approximation with \(\alpha \in \bigO(1)\). 
A polynomial-time approximation scheme (PTAS), is an algorithm that for each optimization problem instance as well as parameter \(\varepsilon > 0\) computes a \((1+\varepsilon)\)-approximation for a minimization problem or a \((1-\varepsilon)\)-approximation for a maximization problem in time in $\bigO(n^{f(\varepsilon)})$, for some computable function \(f\) depending only on \(\varepsilon\).
 We use \PTAS to refer to the class of optimization problems admitting a PTAS. 
 An optimization problem \(L\) is called \APX-hard if every problem in \APX has a \PTAS-reduction to \(L\), i.e., a PTAS for \(L\) implies there is a PTAS for every problem in \APX. If \(L\) is additionally in \APX itself, \(L\) is called \APX-complete. By definition, we have \(\PTAS \subseteq \APX\). Further, \(\PTAS \neq \APX\) unless $\P = \NP$.

We find that taking \emph{any} fair clustering of a forest yields a good approximation. 

 \begin{theorem}\label{thm:trivialApprox}
    Let \(F\) be an \(n\)-vertex \(m\)-edge forest with \(k\ge 2\) colors in a ratio of \(c_1:c_2:\ldots : c_k\) and \(d=\sum_{i=1}^k c_i \ge 4\). Then, there is a \(\frac{\left(d^2-d\right)n+2dm}{\left(d^2-5d+4\right)n+2dm}\)-approximation for \emph{\fcc} on \(F\) computable in time in $\bigO(n)$. 
\end{theorem}

\begin{proof}
    By first sorting the vertices by color and then iteratively adding the next \(c_i\) vertices of each color \(i\) to the next cluster, we obtain a \fc \(\mcP\) with clusters of size \(d\) in linear time. 
    In the worst-case, \(\mcP\) cuts all \(m\) edges.
    Hence, by \autoref{lem:costByCuts}, we have 
    \begin{equation*}
        \cost{\mcP} \le \frac{(d-1)n}{2}-m + 2m = \frac{(d-1)n}{2}+m.
    \end{equation*} 
    
    We compare this cost to the one of a minimum-cost \fc \(\mcP^*\). By \autoref{lem:smallClustersForest}, \(\mcP^*\) to consist of clusters of size \(d\). Each of the \(\frac{n}{d}\) clusters contains at most \(d-1\) edges due to the forest structure.
    Hence, at most \(\frac{n}{d}\cdot (d-1)\) edges are placed inside a cluster. Then, for the inter-cluster cost, we have
    $\chi \ge m - \frac{n}{d}\cdot (d-1) = \frac{n}{d}-n+m$.
   	\autoref{lem:costByCuts} gives 
    \begin{equation*}
        \cost{\mcP^*} \ge \frac{(d-1)n}{2} -m + 2\left(\frac{n}{d}-n+m\right)
        	= \frac{(d-5)n}{2}+\frac{2n}{d} + m .
    \end{equation*}
    Thereby, \(\mcP\) yields an \(\alpha\)-approximation to \fcc, where 
    \begin{align*}
        \alpha &= \left(\frac{(d-1)n}{2}+m\right) \big/ \left(\frac{(d-5)n}{2}+\frac{2n}{d} + m\right)\\
        	&= \left(\frac{\left(d^2-d\right)n+2dm}{2d}\right) \big/ \left(\frac{\left(d^2-5d+4\right)n+2dm}{2d}\right)
        	 = \frac{\left(d^2-d\right)n+2dm}{\left(d^2-5d+4\right)n+2dm}. \qedhere
    \end{align*}
\end{proof}

Observe that \(\alpha\) is decreasing in \(d\) for \(d\ge 4\) and converges to 1 as \(d\rightarrow \infty\). Further, for \(d=5\) we obtain \(\alpha = \frac{20n+10m}{4n+10m}<5\). Thus, for \(d\ge 5\) we have a 5-approximation to \fcc on forests. For \(d=4\), \(\alpha\) becomes linear in \(\frac{m}{n}\) and for smaller \(d\) it is not necessarily positive or not even defined if \(\left(d^2-5d+4\right)n+2dm=0\).
This is because if there are very small clusters, then in forests there are solutions of almost no cost. If \(d=2\), i.e., there are two colors in a \(1:1\) ratio, there are even forests with a cost of 0, namely the ones where all vertices have degree 1 and each edge connects 2 vertices of different colors. A solution cutting every edge is then much worse than an optimum solution. If the factor becomes negative or not defined, this is due to us bounding the inter-cluster cost of the optimum clustering by \(\frac{n}{d}-n+m\), which is possibly negative, while the inter-cluster cost is guaranteed to be non-negative.

On trees, however, if the clusters are small even an optimum solution has to cut some edges as now there always are edges between the clusters. Hence, in this case, we obtain a good approximation for all possible \(d\). 
Note that the proof of \autoref{thm:trivialApprox} does not really require \(d\ge 4\) but for \(d<4\) the approximation factor is just not helpful or defined. This changes, if we assume the forest to be a tree and plug in \(m=n-1\). 

\begin{corollary}\label{thm:trivialApproxTree}
    Let \(T\) be an \(n\)-vertex tree with \(k\ge 2\) colors in a ratio of \(c_1:c_2:\ldots : c_k\) and \(d=\sum_{i=1}^k c_i\). Then, there is a \(\frac{\left(d^2+d\right)n-2d}{\left(d^2-3d+4\right)n-2d}\)-approximation to \emph{\fcc} on \(T\) that is computed in time in $\bigO(n)$. 
\end{corollary}

Now, the approximation factor is still decreasing in \(d\) and converges to 1 as \(d\rightarrow \infty\). However, it is positive and defined for all \(d\ge 2\). For \(d=2\) we obtain \(\frac{6n-4}{2n-4}<3\). Therefore, we have a 3-approximation to \fcc on trees.

Nevertheless, our results for forest suffice to place \fcc in \APX and even in \PTAS.
First, for \(d\ge 5\) we have a 5-approximation to \fcc on forests. If \(d\le 4\), a minimum-cost \fc is found on the forest in polynomial time by \autoref{thm:forestByColorsAlgo}. Hence, \fcc on forests is in \APX. 
Next, recall that the larger the minimum fair cluster size \(d\), the better the approximation becomes. Recall that our dynamic program for \autoref{thm:forestByColorsAlgo} has better running time the smaller the value \(d\). By combining these results, we obtain a PTAS for \fcc on forests. This contrasts \fcc on general graphs, as even unfair \cc is \APX-hard there \cite{Charikar_Guruswami_Wirth_2005} and therefore does not admit a PTAS unless $\P = \NP$.

\begin{theorem}\label{thm:ptas}
    There is a PTAS for \emph{\fcc} on forests.
    Moreover, an $(1{+}\varepsilon)$-approximate fair clustering can be computed 
    in time $\bigO(n^{\emph{\poly}(1/\varepsilon)})$.
\end{theorem}

\begin{proof}
    If \(d\le 4\), we find a minimum-cost \fc in polynomial time by \autoref{thm:forestByColorsAlgo}.
    Else, if \(\frac{\left(d^2-d\right)n+2dm}{\left(d^2-5d+4\right)n+2dm}\le 1+\varepsilon\), it suffices to return any fair clustering by \autoref{thm:trivialApprox}.
    Otherwise, we have \(d\ge 5\) and 
    \begin{equation*}
        1+\varepsilon < \frac{\left(d^2-d\right)n+2dm}{\left(d^2-5d+4\right)n+2dm}
                      < \frac{\left(d^2-d\right)n}{\left(d^2-5d\right)n}
                      = \frac{d-1}{d-5}.
    \end{equation*}
    It follows that, $d-5+d\varepsilon-5\varepsilon < d-1$, which simplifies to
    $d < \frac{4}{\varepsilon}+5$.
	Hence, by \autoref{thm:forestByColorsAlgo}, we find a minimum-cost \fc 
	in time in $\bigO(n^{f(\varepsilon)})$ for some computable function \(f\) independent from \(n\).
	In all cases, we find a \fc with a cost of at most \(1+\varepsilon\) times the minimum \cccost and
	take time in $\bigO(n^{f(\varepsilon)})$, giving a PTAS. 
	
	To show that $f$ is in fact bounded by a polynomial in $\sfrac{1}{\varepsilon}$,
	we only need to look at the third case (otherwise $f$ is constant).
	The bound $d < \frac{4}{\varepsilon} + 5$ and $d = \sum_{i=1}^k c_i$ together
	imply the the number of colors $k$ is constant w.r.t.\ $n$.
	Under this condition, the exponent of the running time in \autoref{thm:forestByColorsAlgo}
	is a polynomial in $d$ and thus in $\sfrac{1}{\varepsilon}$.
\end{proof}

\bibliographystyle{plainurl} 
\bibliography{fair_clustering}

\end{document}